\documentclass{article}

     \PassOptionsToPackage{numbers, compress}{natbib}

\usepackage[final]{neurips_2022}

\usepackage[utf8]{inputenc} 
\usepackage[T1]{fontenc}    
\usepackage{hyperref}       
\usepackage{url}            
\usepackage{booktabs}       
\usepackage{amsfonts}       
\usepackage{nicefrac}       
\usepackage{microtype}      
\usepackage{xcolor}         

\usepackage{enumitem}

\usepackage{amsmath,amsfonts,bm}
\usepackage{amssymb,amsthm}
\usepackage{mathtools}
\usepackage{algorithm,algorithmic}
\usepackage{url}
\usepackage{multirow}
\usepackage{array}
\usepackage{cancel}
\usepackage{etoc}
\usepackage{pifont}
\usepackage[capitalize]{cleveref}
\usepackage{wrapfig}
\usepackage{subcaption}

\crefname{proposition}{Proposition}{Propositions}
\crefname{theorem}{Theorem}{Theorems}
\crefname{lemma}{Lemma}{Lemmas}
\crefname{update_rule}{Update}{Updates}
\crefname{algorithm}{Algorithm}{Algorithms}
\crefname{figure}{Figure}{Figures}
\crefname{claim}{Claim}{Claims}

\def\eqref#1{equation~\ref{#1}}

\def\1{\bm{1}}

\DeclareMathAlphabet{\mathsfit}{\encodingdefault}{\sfdefault}{m}{sl}
\SetMathAlphabet{\mathsfit}{bold}{\encodingdefault}{\sfdefault}{bx}{n}

\def\gA{{\mathcal{A}}}

\def\gE{{\mathcal{E}}}
\def\gF{{\mathcal{F}}}

\def\gM{{\mathcal{M}}}

\def\gP{{\mathcal{P}}}

\def\gS{{\mathcal{S}}}

\def\gX{{\mathcal{X}}}

\def\sI{{\mathbb{I}}}

\def\sR{{\mathbb{R}}}

\newcommand{\E}{\mathbb{E}}

\newcommand{\R}{\mathbb{R}}

\newcommand{\softmax}{\mathrm{softmax}}

\DeclareMathOperator*{\argmax}{arg\,max}

\newtheorem{theorem}{Theorem}
\newtheorem{lemma}{Lemma}
\newtheorem{definition}{Definition}
\newtheorem{proposition}{Proposition}
\newtheorem{remark}{Remark}

\newtheorem{assumption}{Assumption}
\newtheorem{corollary}{Corollary}
\newtheorem{update_rule}{Update}
\newtheorem{claim}{Claim}

\DeclareMathOperator*{\expectation}{\mathbb{E}}

\def\rvone{{\mathbf{1}}}

\def\diagonalmatrix{\text{diag}}

\DeclareMathOperator*{\probability}{Pr}

\newcommand{\cA}{\mathcal{A}}
\newcommand{\cF}{\mathcal{F}}
\newcommand{\chE}{\mathbb{E}}
\newcommand{\EE}[1]{\mathbb{E}[#1]}
\newcommand{\PP}{\mathbb{P}}
\newcommand{\EEt}[1]{\mathbb{E}_t[#1]}

\newlength\tocrulewidth
\title{The Role of Baselines in Policy Gradient Optimization}

\author{%
  Jincheng Mei$^{\, 1}$\thanks{Correspondence to: Jincheng Mei and Csaba Szepesv{\'a}ri}
  \hspace{7mm} 
  Wesley Chung$^{\, 2}$
  \hspace{7mm}
  Valentin Thomas$^{\, 3}$
  \hspace{7mm}
  Bo Dai$^{\, 1}$ \\
  \textbf{Csaba Szepesv{\'a}ri}$^{\, 4, \, 5, \, *}$
  \hspace{7mm}
  \textbf{Dale Schuurmans}$^{\, 1, \, 5}$ \\
  \\
  \hspace{-7mm} $^1$\normalfont{Google Research, Brain Team} \hspace{2mm} $^2$Mila, McGill University \hspace{2mm} $^3$Mila, University of Montreal \\
  \hspace{2mm} $^4$DeepMind \hspace{2mm} $^5$Amii, University of Alberta \\
  \\
  \hspace{-2mm}
  \texttt{\{jcmei,bodai,szepi,schuurmans\}@google.com} \hspace{2mm}
  \texttt{\{wesley.chung2,vltn.thomas\}@gmail.com}
}

\begin{document}
\etocdepthtag.toc{mtmainpaper}

\maketitle

\begin{abstract}
We study the effect of baselines in on-policy stochastic policy gradient optimization, and close the gap between the theory and practice of policy optimization methods.
Our first contribution is to show that the \emph{state value} baseline allows on-policy stochastic \emph{natural} policy gradient (NPG) to converge to a globally optimal policy at an $O(1/t)$ rate, which was not previously known.
The analysis relies on two novel findings: the expected progress of the NPG update satisfies a stochastic version of the non-uniform \L{}ojasiewicz (N\L{}) inequality, and with probability 1 the state value baseline prevents the optimal action's probability from vanishing, thus ensuring sufficient exploration. 
Importantly, these results provide a new understanding of the role of baselines in stochastic policy gradient: by showing that the variance of natural policy gradient estimates remains unbounded with or without a baseline, we find that variance reduction \emph{cannot} explain their utility in this setting.
Instead, the analysis reveals that the primary effect of the value baseline is to \textbf{reduce the aggressiveness of the updates} rather than their variance.
That is, we demonstrate that a finite variance is \emph{not necessary} for almost sure convergence of stochastic NPG, while controlling update aggressiveness is both necessary and sufficient.
Additional experimental results verify these theoretical findings.
\end{abstract}

\section{Introduction}

The policy gradient (PG) \citep{sutton2000policy} is a key concept in reinforcement learning (RL), lying at the foundation of policy-based and actor-critic methods, and responsible for some of the most prominent practical achievements in RL \citep{schulman2015trust,schulman2017proximal,haarnoja2018soft}. 
However, progress in the theoretical understanding of PG methods is recent, and a number of the techniques used in practice still lack rigorous support, particularly in the online stochastic regime where an action is sampled from the current policy at each iteration. 
We study stochastic policy optimization in more detail to close this gap between theory and practice. 

In stochastic policy optimization, the two most common techniques for improving the basic algorithm are to include on-policy importance sampling (IS) and subtract a baseline.
Including on-policy IS provides unbiased gradient estimates, but introduces high variance when an action's sampling probability is close to $0$.
Meanwhile, subtracting a baseline remains a heuristic \citep{recht2018updates} that has strong empirical but limited theoretical support.
One possible benefit of a baseline is that it provides variance reduction \citep{greensmith2004variance}, which has motivated work on designing alternative baselines that further reduce variance \citep{tucker2018mirage,bhatnagar2007incremental,mao2018variance,wu2018variance}.
However, other work \citep{chung2020beyond} has shown that variance reduction is not necessarily aligned with policy learning quality. 
To date, it has remained unclear how a baseline impacts the quality of the ultimate solution found by policy gradient optimization.
We resolve this question in this work.

Recent progress in the theory of deterministic PG has shown that, given exact gradients, softmax policy gradient is able to converge to a globally optimal policy at a $O(1/t)$ rate \citep{mei2020global}.
Unfortunately, despite this guarantee, the constants in this rate can be extremely large \citep{li2021softmax} due to initialization sensitivity and poor performance at escaping sub-optimal plateaus \citep{mei2020escaping}.
Therefore, in the exact gradient setting, several techniques have been considered for mitigating 
the weaknesses of softmax PG, leading to better constants \citep{agarwal2021theory} or even exponentially faster rates of $O(e^{- c \cdot t})$ for $c > 0$.
Such improvements include adding entropy regularization \citep{mei2020global,cen2021fast}, normalizing the gradients \citep{mei2021leveraging}, or applying natural policy gradient (NPG) \citep{cen2021fast,khodadadian2021linear,mei2021understanding}.

However, in the on-policy \textit{stochastic} optimization case, recent studies \citep{mei2021understanding,chung2020beyond} show that naively applying the above techniques, such as normalization or NPG, leads to unexpectedly \emph{worse} performance than stochastic PG.
That is, techniques that accelerate convergence in the exact policy gradient setting become \emph{unsound} in the stochastic gradient setting, by inducing a non-zero probability of failure (i.e., failing to converge to a globally optimal solution) \citep{mei2021understanding}.
Such failures occur even when stochastic PG can still converge to a global optimum in probability.
Previous work has indicated that one key reason behind the failure of these acceleration strategies arises from their ``over-committal behaviour'' in the stochastic setting, 
which occurs independently
of the variance of the gradient estimates \citep{chung2020beyond}. 
That is, baseline techniques with higher variance can still better avoid over-committal behaviour (i.e., premature convergence) and ultimately achieve better policy optimization \citep{chung2020beyond}.

To resolve this issue, we develop a deeper understanding of the role of baselines in stochastic policy optimization based on the following contributions. 
\textcolor{blue}{First}, we establish a new result that combining on-policy IS with a value function baseline and natural policy gradient (NPG) can achieve almost sure convergence to a globally optimal policy at a $O(1/t)$ rate. 
This result is based on two novel findings: 
\textbf{(i)} At any iteration $t$, the conditional expected progress of the algorithm's next iterate obeys a stochastic non-uniform \L{}ojasiewicz (N\L{}) inequality. 
\textbf{(ii)} The use of the state value baseline (with appropriate learning rate control) almost surely prevents the probability of the optimal action from vanishing.
These findings show that a key role of the value baseline is to automatically ensure ``sufficient exploration'' during on-policy stochastic optimization. 
\textcolor{blue}{Next}, we provide a detailed understanding of how baselines modulate the circular interaction between stochastic action sampling and updating.
Although a baseline has no effect on exact gradients, it can play a major role in stochastic gradients.
In this respect, we first show that the PG estimator variance is unbounded with or without a baseline, hence variance reduction cannot be the primary effect.
Instead, our analysis reveals that the key role the baseline plays in ensuring global convergence is to reduce the aggressiveness of updates.
That is, finite variance of the gradient estimates is not necessary for ensuring global convergence, while properly controlling update aggressiveness is both necessary and sufficient.

The remainder of the paper is organized as follows. 
\cref{sec:on_policy_npg} provides the main results that establish the almost sure $O(1/t)$ convergence rate of stochastic NPG with on-policy IS and state value baseline to a globally optimal policy.
\cref{sec:stochastic_policy_optimization_circles} then develops the new understanding of the role of the baseline by going beyond standard variance reduction arguments.
\cref{sec:simulations} provides some simulations to verify the results, and \cref{sec:discussions_conclusions_future_work} concludes the paper with a brief discussion.

\section{On-policy Stochastic Natural Policy Gradient}
\label{sec:on_policy_npg}

We first consider a one-state Markov Decision Process (MDP) defined by a finite action space $[K] \coloneqq \{ 1, 2, \dots, K \}$ where the true mean reward vector is $r \in [0, 1]^K$. The policy optimization problem is to maximize the expected reward,
{\small
\begin{align}
\label{eq:expected_reward_objective}
    \max_{\theta : [K] \to \sR}\;{ \expectation_{a \sim \pi_{\theta}(\cdot)}{ [ r(a) ]  } },
\end{align}
}%
where the policy $\pi_\theta$ is parameterized by $\theta$ using the standard softmax parameterization,
{\small
\begin{align}
\label{eq:softmax}
\pi_\theta(a) = \frac{ \exp\{ \theta(a) \} }{ \sum_{a^\prime \in [K]}{ \exp\{ \theta(a^\prime) } \} } \mbox{, \quad   for all } a \in [K].
\end{align}
}%
Our focus in this paper is on on-policy optimization, where at each iteration $t \ge 1$ the current policy $\pi_{\theta_t}$ is used to sample one action and perform one update. 

For the sampled action $a_t$, a noisy reward observation $x_t(a_t) \in \sR$ is drawn from an unknown distribution with expected value $r(a_t)$. 
We make the following assumption that the observed reward $x_t(a)$ is sampled from a bounded distribution: $x_t(a) \in [-R_{\max}, R_{\max}]$ with probability one.
\begin{assumption}[Bounded sampled reward]
\label{assump:bounded_reward}
For each action $a \in [K]$, the true mean reward $r(a)$ is the expectation of a bounded reward distribution, i.e.,
{\small
\begin{align}
\label{eq:true_mean_reward_expectation_bounded}
	r(a) &= \int_{-R_{\max}}^{R_{\max}}{ x \cdot P_a(x) \mu(d x)}
\end{align}
}%
where $\mu$ is a finite measure over $[-R_{\max}, R_{\max}]$, and $P_a(x) \ge 0$ is the probability density function with respect to $\mu$, and $R_{\max} > 0$ is the reward range.
We let $R_a$ denote the reward distribution for action $a$ defined by the density $P_a$ and base measure $\mu$.
\end{assumption}
Then, given a sampled reward observation $x_t(a)\sim R_a$, an unbiased estimate of the expected reward vector $r$ can be formed by on-policy importance sampling (IS).
\begin{definition}[On-policy importance sampling (IS)]
\label{def:on_policy_importance_sampling}
At iteration $t$, sample one action $a_t \sim \pi_{\theta_t}(\cdot)$ and observe one reward sample $x_t(a_t)\sim R_{a_t}$.
Let $x_t(a)=0$ for all $a \not= a_t$.
Then the IS reward estimate is constructed as $\hat{r}_t(a) = \frac{ \sI\left\{ a = a_t \right\} }{ \pi_{\theta_t}(a) } \cdot x_t(a)$ for all $a \in [K]$.
\end{definition}
If the true mean reward $r(a_t)$ is observed for sampled actions $a_t$, we have the simplified IS estimator.
\begin{definition}[Simplified on-policy importance sampling (IS)]
\label{def:simplified_on_policy_importance_sampling}
At iteration $t$, sample one action $a_t \sim \pi_{\theta_t}(\cdot)$. 
The IS reward estimate is then constructed as $\hat{r}_t(a) = \frac{ \sI\left\{ a_t = a \right\} }{ \pi_{\theta_t}(a) } \cdot r(a)$ for all $a \in [K]$.
\end{definition}
\cref{def:simplified_on_policy_importance_sampling} will be used for illustrating ideas and new understandings in \cref{sec:stochastic_policy_optimization_circles}, while the main results in \cref{sec:on_policy_npg} are based on \cref{def:on_policy_importance_sampling}.

\subsection{Failure Without a Baseline}

First, to establish context, we review an existing negative result for the representative algorithm, natural policy gradient (NPG) \citep{kakade2002natural}, which for the softmax parameterization is defined as follows.
\begin{update_rule}[NPG with on-policy stochastic gradient]
\label{update_rule:softmax_natural_pg_special_on_policy_stochastic_gradient}
$\theta_{t+1} \gets \theta_{t} + \eta \cdot \hat{r}_t$,
where $\pi_\theta(a)$ is by \cref{eq:softmax}.
\end{update_rule}
It is known that NPG behaves problematically with on-policy IS, even if the true mean reward $r(a_t)$ is observed.
In particular, NPG converges to a sub-optimal deterministic policy with a constant positive probability in this case, as shown by \citep{chung2020beyond,mei2021understanding}.
\begin{proposition}[Theorem 3 of \citep{mei2021understanding}]
\label{prop:failure_probability_softmax_natural_pg_special_on_policy_stochastic_gradient}
Using \cref{update_rule:softmax_natural_pg_special_on_policy_stochastic_gradient}, where $\hat{r}_t$ is from \cref{def:simplified_on_policy_importance_sampling}, and $r \in (0, 1]^K$, we have, with positive probability, $\sum_{a \not= a^*}{ \pi_{\theta_t}(a)} \to 1$ as $t \to \infty$.
\end{proposition}
Essentially \cref{prop:failure_probability_softmax_natural_pg_special_on_policy_stochastic_gradient} asserts that \cref{update_rule:softmax_natural_pg_special_on_policy_stochastic_gradient} is too aggressive: 
if sub-optimal actions are sampled $t$ times successively, their probabilities will become exponentially close to $1$; i.e., $1 - \sum_{a \not= a^*}{ \pi_{\theta_t}(a)} \in O(e^{- c \cdot t})$. 
It follows that $\prod_{t=1}^{\infty}{ \sum_{a \not= a^*}{ \pi_{\theta_t}(a)} } > 0$; that is, the on-policy sampling process $a_t \sim \pi_{\theta_t}(\cdot)$ has a non-zero probability of sampling sub-optimal actions forever, which implies that there is a positive probability that $\pi_{\theta_t}$ fails to converge to an optimal deterministic policy.

\subsection{Global Convergence with a Value Baseline}

Despite the above failure, we now prove that subtracting a value baseline rectifies the problem for NPG.
%
Consider the modified update that includes a baseline.
\begin{update_rule}[NPG, on-policy stochastic gradient with value baseline]
\label{update_rule:softmax_natural_pg_special_on_policy_stochastic_gradient_value_baseline}
$\theta_{t+1} \gets \theta_{t} + \eta \cdot \big( \hat{r}_t - \hat{b}_t \big)$, where $\pi_\theta(a)$ is by \cref{eq:softmax}, $\hat{b}_t(a) = \left( \frac{ \sI\left\{ a_t = a \right\} }{ \pi_{\theta_t}(a) } - 1 \right) \cdot b_t$ for all $a \in [K]$, and $b_t \coloneqq \pi_{\theta_t}^\top r$.
\end{update_rule}
Since $\softmax(\theta) = \softmax(\theta + c \cdot \rvone)$ for all $c \in \sR$, \cref{update_rule:softmax_natural_pg_special_on_policy_stochastic_gradient_value_baseline} is equivalent to the following update if $\hat{r}_t$ is by \cref{def:on_policy_importance_sampling}. Given the same $\pi_{\theta_t}$, \cref{update_rule:softmax_natural_pg_special_on_policy_stochastic_gradient_value_baseline,update_rule:equivalent_update_softmax_natural_pg_special_on_policy_stochastic_gradient_value_baseline} produce the same next policy $\pi_{\theta_{t+1}}$.
\begin{update_rule}
\label{update_rule:equivalent_update_softmax_natural_pg_special_on_policy_stochastic_gradient_value_baseline}
$\theta_{t+1}(a) \gets \theta_t(a) + \eta \cdot \frac{ \sI\left\{ a_t = a \right\} }{ \pi_{\theta_t}(a) } \cdot \left( x_t(a) - \pi_{\theta_t}^\top r \right)$, i.e., $\theta_{t+1}(a_t) \gets \theta_t(a_t) + \eta \cdot \frac{ x_t(a_t) - \pi_{\theta_t}^\top r }{ \pi_{\theta_t}(a_t) } $, and $\theta_{t+1}(a) \gets \theta_t(a)$ for all $a \not= a_t$.
\end{update_rule}
Unfortunately, the variance of this update is not uniformly bounded whenever $\pi_{\theta_t}(a)$ is close to $0$ for at least one action $a \in [K]$ (\cref{prop:softmax_natural_pg_variances}), therefore standard stochastic gradient analysis for bounded variance estimators \citep{nemirovski2009robust,zhang2020sample,lan2021policy,zhang2021convergence} cannot be applied. 
Instead, we develop two new techniques to establish global convergence results, both of which rely heavily on using baselines.

\cref{lem:non_uniform_lojasiewicz_stochastic_npg_value_baseline_special} provides the first key technique, which we refer to as the stochastic N\L{} inequality.
\begin{lemma}[Stochastic non-uniform \L{}ojasiewciz (N\L{})]
\label{lem:non_uniform_lojasiewicz_stochastic_npg_value_baseline_special}
Suppose \cref{assump:bounded_reward} holds. 
Let $r \in [0,1]^K$, $a^* \coloneqq \argmax_{a \in [K]}{ r(a) }$, and $\Delta \coloneqq r(a^*) - \max_{a \not= a^*}{ r(a) }$. Using \cref{update_rule:softmax_natural_pg_special_on_policy_stochastic_gradient_value_baseline} with on-policy sampling $a_t \sim \pi_{\theta_t}(\cdot)$ and IS estimator $\hat{r}_t$,
\begin{description}[style=unboxed,leftmargin=0cm]
    \item[(1)] if $\hat{r}_t$ is from \cref{def:simplified_on_policy_importance_sampling}, then with constant learning rate $\eta > 0$, we have, for all $t \ge 1$,
{\small
\begin{align}
\label{eq:non_uniform_lojasiewicz_stochastic_npg_value_baseline_special_deterministic_reward_result_1}
    \pi_{\theta_{t+1}}^\top r - \pi_{\theta_t}^\top r &\ge 0, \qquad \text{almost surely (a.s.),} \qquad \text{and} \\
    \EEt{\pi_{\theta_{t+1}}^\top r} - \pi_{\theta_t}^\top r &\ge \frac{ \eta }{ 1 + \eta } \cdot \textcolor{red}{ \pi_{\theta_t}(a^*) } \cdot \big( r(a^*) - \pi_{\theta_t}^\top r \big)^2,
\end{align}
}%
where $\EEt{\cdot}$ is on randomness from on-policy sampling $a_t \sim \pi_{\theta_t}(\cdot)$.
    \item[(2)] if $\hat{r}_t$ is from \cref{def:on_policy_importance_sampling}, then with learning rate,
{\small
\begin{align}
\label{eq:non_uniform_lojasiewicz_stochastic_npg_value_baseline_special_stochastic_reward_result_1}
    \eta = \frac{\pi_{\theta_t}(a_t) \cdot \left| r(a_t) - \pi_{\theta_t}^\top r \right|}{8 \cdot R_{\max}^2},
\end{align}
}%
\noindent we have, for all $t \ge 1$,
{\small 
\begin{align}
    \EEt{\pi_{\theta_{t+1}}^\top r} - \pi_{\theta_t}^\top r &\ge \frac{1}{16 \cdot R_{\max}^2} \cdot \sum_{i = 1}^{K} \pi_{\theta_t}(i)^2 \cdot \big| r(i) - \pi_{\theta_t}^\top r \big|^3
    \label{eq:non_uniform_lojasiewicz_stochastic_npg_value_baseline_special_stochastic_reward_result_2a}
    \\
    &\ge \frac{ 1 }{16 \cdot R_{\max}^2} \cdot \frac{\Delta}{K-1} \cdot  \textcolor{red}{\pi_{\theta_t}(a^*)^2} \cdot \big( r(a^*) - \pi_{\theta_t}^\top r \big)^2,
    \label{eq:non_uniform_lojasiewicz_stochastic_npg_value_baseline_special_stochastic_reward_result_2b}
\end{align}
}%
where $\EEt{\cdot}$ is on randomness from on-policy sampling $a_t \sim \pi_{\theta_t}(\cdot)$ and reward sampling $x \sim R_{a_t}$.
\end{description}
\end{lemma}
\begin{remark}
\label{rmk:learning_rate_decay_speed_special}
We have $\eta \in O(1/t)$ in \cref{eq:non_uniform_lojasiewicz_stochastic_npg_value_baseline_special_stochastic_reward_result_1} after knowing the convergence rate later.
\end{remark}
We refer to $\pi_{\theta_t}(a^*)^2$ in \cref{eq:non_uniform_lojasiewicz_stochastic_npg_value_baseline_special_stochastic_reward_result_2b} the \textbf{stochastic N\L{} coefficient}. 
\cref{lem:non_uniform_lojasiewicz_stochastic_npg_value_baseline_special} is a stochastic generalization of the N\L{} inequality, which has been widely used in proving global convergence of softmax PG variants \citep{mei2020global,mei2020escaping,mei2021leveraging,mei2021understanding,zhang2022effect}. 
It is stochastic since \cref{eq:non_uniform_lojasiewicz_stochastic_npg_value_baseline_special_stochastic_reward_result_2a} contains an expectation. 
It is non-uniform because \cref{eq:non_uniform_lojasiewicz_stochastic_npg_value_baseline_special_stochastic_reward_result_2b} depends on $\theta_t$, which cannot be uniformly lower bounded away from $0$ across the entire domain of $\theta \in \sR^K$ (that is, one can always find $\theta$ such that $\pi_{\theta}(a^*)$ is arbitrarily close to $0$).

The key idea of \cref{lem:non_uniform_lojasiewicz_stochastic_npg_value_baseline_special} is as follows. If $\hat{r}_t$ is from \cref{def:simplified_on_policy_importance_sampling}, then by algebra we have,
{\small 
\begin{align}
\label{eq:idea_of_non_uniform_lojasiewicz_stochastic_npg_value_baseline_special_stochastic_reward_1}
\textstyle
    \EEt{\pi_{\theta_{t+1}}^\top r} - \pi_{\theta_t}^\top r &=  \sum_{i = 1}^{K} \pi_{\theta_t}(i) \cdot \frac{ \left[ \exp\Big\{ \eta \cdot \frac{ \textcolor{blue}{ r(i) - \pi_{\theta_t}^\top r} }{\pi_{\theta_t}(i)} \Big\} - 1 \right] \cdot \left( \textcolor{blue}{r(i) - \pi_{\theta_t}^\top r }\right) }{ \exp\Big\{ \eta \cdot \frac{r(i) - \pi_{\theta_t}^\top r}{\pi_{\theta_t}(i)} \Big\} + \frac{ 1 - \pi_{\theta_t}(i) }{ \pi_{\theta_t}(i) } }.
\end{align}
}%
Since $\left( e^{c \cdot \textcolor{blue}{y}} - 1 \right) \cdot \textcolor{blue}{y} \ge 0$ for all $y \in \sR$ and $c > 0$, \cref{eq:idea_of_non_uniform_lojasiewicz_stochastic_npg_value_baseline_special_stochastic_reward_1} is non-negative (letting $y \coloneqq r(i) - \pi_{\theta_t}^\top r$ and $c \coloneqq \nicefrac{\eta}{ \pi_{\theta_t}(i) }$). However, this is not true if $\hat{r}_t$ is from \cref{def:on_policy_importance_sampling}, where we have,
{\small 
\begin{align}
\label{eq:idea_of_non_uniform_lojasiewicz_stochastic_npg_value_baseline_special_stochastic_reward_2}
    \EEt{\pi_{\theta_{t+1}}^\top r} - \pi_{\theta_t}^\top r = \sum_{i = 1}^{K} \pi_{\theta_t}(i) \cdot \int_{-R_{\max}}^{R_{\max}}{ \frac{ \left[ \exp\Big\{ \eta \cdot \frac{ \textcolor{red}{ x - \pi_{\theta_t}^\top r} }{\pi_{\theta_t}(i)} \Big\} - 1 \right] \cdot \left( \textcolor{blue}{ r(i) - \pi_{\theta_t}^\top r } \right) }{ \exp\Big\{ \eta \cdot \frac{x - \pi_{\theta_t}^\top r}{\pi_{\theta_t}(i)} \Big\} + \frac{ 1 - \pi_{\theta_t}(i) }{ \pi_{\theta_t}(i) } } \cdot P_i(x) \mu(d x)}.
\end{align}
}%
Note that $( e^{c \cdot \textcolor{red}{y^\prime}} - 1 ) \cdot \textcolor{blue}{y} < 0$ if $y^\prime \cdot y < 0$ and $c >0$ (letting $y^\prime \coloneqq x - \pi_{\theta_t}^\top r$, $y \coloneqq r(i) - \pi_{\theta_t}^\top r$, and $c \coloneqq \nicefrac{\eta}{ \pi_{\theta_t}(i) }$). For a ``good'' action ($r(i) - \pi_{\theta_t}^\top r > 0$), if unfortunately its sampled reward is ``bad'' ($x - \pi_{\theta_t}^\top r < 0$), then the update will make negative progress. Similar things happen for a ``bad'' action ($r(i) - \pi_{\theta_t}^\top r < 0$) with ``good'' sampled reward ($x - \pi_{\theta_t}^\top r > 0$). It is then necessary to use $\eta$ like \cref{eq:non_uniform_lojasiewicz_stochastic_npg_value_baseline_special_stochastic_reward_result_1}, to control the non-linear sigmoid-like functions in the progress by piecewise linear functions (\cref{lem:piecewise_linear_domination}) to get non-negative \textbf{expected} progresses. According to \cref{eq:non_uniform_lojasiewicz_stochastic_npg_value_baseline_special_stochastic_reward_result_2b}, we have  
\begin{align}
\label{eq:subm}
\EEt{\pi_{\theta_{t+1}}^\top r} - \pi_{\theta_t}^\top r \ge 0\,,
\end{align}
which implies that \cref{{update_rule:softmax_natural_pg_special_on_policy_stochastic_gradient_value_baseline}} achieves non-negative progress \textit{in expectation}. Combining \cref{lem:non_uniform_lojasiewicz_stochastic_npg_value_baseline_special} with Doob's supermartingale convergence theorem 
then leads to the following result.
\begin{corollary}
\label{cor:almost_sure_convergence_stochastic_npg_value_baseline_special}
The sequence $\{\pi_{\theta_t}^\top r\}_{t\ge 1}$ converges with probability one.
\end{corollary}
\cref{cor:almost_sure_convergence_stochastic_npg_value_baseline_special} asserts that, the random sequence $\pi_{\theta_t}^\top r$ produced by \cref{update_rule:softmax_natural_pg_special_on_policy_stochastic_gradient_value_baseline} asymptotically approaches some finite value (since $\pi_{\theta}^\top r \in [0, 1]$), ruling out the possibility of divergence (oscillating forever). However, this does not necessarily imply that $\pi_{\theta_t}^\top r \to r(a^*)$ as $t \to \infty$. A subtlety arises in bounding the stochastic N\L{} coefficient in \cref{eq:non_uniform_lojasiewicz_stochastic_npg_value_baseline_special_stochastic_reward_result_2a} away from $0$, which requires a second key technique.
\begin{lemma}[Non-vanishing stochastic N\L{} coefficient / ``automatic exploration'']
\label{lem:non_vanishing_nl_coefficient_stochastic_npg_value_baseline_special}
Using \cref{update_rule:softmax_natural_pg_special_on_policy_stochastic_gradient_value_baseline} with conditions in \cref{lem:non_uniform_lojasiewicz_stochastic_npg_value_baseline_special} and $\hat{r}_t$ from \cref{def:on_policy_importance_sampling}, for an arbitrary initialization $\theta_1 \in \sR^K$, we have,
\begin{align}
\label{eq:non_vanishing_nl_coefficient_stochastic_npg_value_baseline_special_result_1}
    c \coloneqq \inf_{t \ge 1} \pi_{\theta_t}(a^*) > 0, \qquad \text{almost surely (a.s.).}
\end{align}
\end{lemma}
\cref{lem:non_uniform_lojasiewicz_stochastic_npg_value_baseline_special,lem:non_vanishing_nl_coefficient_stochastic_npg_value_baseline_special} together guarantee that  $\pi_{\theta_t}^\top r \to r(a^*)$ as $t \to \infty$. In fact, using the ``variance-like'' expected progress (\cref{eq:non_uniform_lojasiewicz_stochastic_npg_value_baseline_special_stochastic_reward_result_2a}), \cref{cor:almost_sure_convergence_stochastic_npg_value_baseline_special} implies that $\pi_{\theta_t}$ approaches a ``generalized one-hot policy'' as $t \to \infty$. \cref{lem:non_vanishing_nl_coefficient_stochastic_npg_value_baseline_special} then argues by contradiction that $\pi_{\theta_t}$ cannot approach a sub-optimal ``generalized one-hot policy'' as $t \to \infty$, which will imply that the optimal action's probability must approach $1$ and achieve \cref{eq:non_vanishing_nl_coefficient_stochastic_npg_value_baseline_special_result_1}. 
Proof details in the appendix and intuitions in \cref{sec:stochastic_policy_optimization_circles} reveal that  \cref{update_rule:softmax_natural_pg_special_on_policy_stochastic_gradient_value_baseline} achieves a form of ``automatic exploration'' by using a baseline, i.e., maintaining $\pi_{\theta_t}(a)$ decay no faster than $O(1/t)$, such that every action will be sampled infinitely many times in a long run.
Finally, combining \cref{lem:non_uniform_lojasiewicz_stochastic_npg_value_baseline_special,lem:non_vanishing_nl_coefficient_stochastic_npg_value_baseline_special}, we establish not only asymptotic convergence of NPG to a global optimum, but also a global convergence rate of $O(1/t)$ in terms of the sub-optimality gap.
\begin{theorem}[Almost sure global convergence rate]
\label{thm:almost_sure_convergence_rate_stochastic_npg_special_value_baseline}
Using \cref{update_rule:softmax_natural_pg_special_on_policy_stochastic_gradient_value_baseline} with on-policy sampling $a_t \sim \pi_{\theta_t}(\cdot)$, the IS estimator $\hat{r}_t$ in \cref{def:on_policy_importance_sampling},  $\eta$ in \cref{eq:non_uniform_lojasiewicz_stochastic_npg_value_baseline_special_stochastic_reward_result_1}, and any initialization $\theta_1 \in \sR^K$ , we have,
\begin{align}
    \EE{ \left( \pi^* - \pi_{\theta_t} \right)^\top r } \le \frac{16 \cdot R_{\max}^2 }{ \Delta \cdot \EE{ c^2 } } \cdot \frac{K-1}{t},& \qquad \text{and} \\
    \limsup_{t \ge 1} \bigg\{ \frac{ \Delta \cdot c^2 }{16 \cdot R_{\max}^2 } \cdot \frac{t}{K-1} \cdot \left( \pi^* - \pi_{\theta_t} \right)^\top r \bigg\} < \infty,& \qquad \text{a.s.},
\end{align}
where $\pi^* \coloneqq \argmax_{\pi \in \Delta(K)}{ \pi^\top r}$ is the optimal policy, $R_{\max}$ is the sampled reward range from \cref{assump:bounded_reward}, $\Delta \coloneqq r(a^*) - \max_{a \not= a^*}{ r(a) }$ is the reward gap of $r$, and $c > 0$ is from \cref{lem:non_vanishing_nl_coefficient_stochastic_npg_value_baseline_special}.
\end{theorem}

\subsection{General MDPs}

Next, we generalize these results to finite Markov decision processes (MDPs).
Given a finite set $\gX$, let $\Delta(\gX)$ denote the set of all probability distributions on $\gX$. 
A finite MDP is defined as a tuple $\gM \coloneqq (\gS, \gA, r, \gP, \gamma)$, where $\gS$ and $\gA$ are finite state and action spaces, respectively. $r: \gS \times \gA \to \sR$ is the expected reward function, $\gP: \gS \times \gA \to \Delta(\gS)$ is the probability transition function, and $\gamma \in [0, 1)$ is the discount factor. 
We also extend \cref{assump:bounded_reward} to every $(s,a)\in\gS\times\gA$ and assume there is a reward distribution $R_{s,a}$ with expectation $r(s,a)$, uniformly bounded within $[-R_{\max}, R_{\max}]$.
Given a policy $\pi: \gS \to \Delta(\gA)$, at each time $t \ge 0$, an agent is given a state $s_t \in \gS$, takes an action $a_t \sim \pi(\cdot | s_t)$, then receives a scalar reward observation $x(s_t, a_t)\sim R_{s_t,a_t}$ and a next-state $s_{t+1} \sim \gP( \cdot | s_t, a_t)$. 
The value function of $\pi$ at state $s$ is defined as
{\small 
\begin{align}
\label{eq:state_value_function}
    V^\pi(s) \coloneqq \expectation_{\substack{a_t \sim \pi(\cdot | s_t), \\ s_{t+1} \sim \gP( \cdot | s_t, a_t)}}{\left[ \sum_{t=0}^{\infty}{\gamma^t r(s_t, a_t)} \ \bigg| \ s_0 = s \right]}.
\end{align}
}%
The policy optimization problem for a general MDP is to maximize the expected value of the policy,
{\small
\begin{align}
\label{eq:expected_state_value_function}
    \max_{\theta : \gS \times \gA \to \sR}{ V^{\pi_\theta}(\rho) } \coloneqq \max_{\theta : \gS \times \gA \to \sR}\;{ \expectation_{s \sim \rho(\cdot)}{ \left[ V^{\pi_\theta}(s) \right]} },
\end{align}
}%
where $\rho \in \Delta(\gS)$ is an initial state distribution, and $\pi_\theta(\cdot | s) = \softmax(\theta(s, \cdot))$,
{\small
\begin{align}
\label{eq:softmax_transform_general}
    \pi_\theta(a | s) = \frac{ \exp\{ \theta(s, a) \} }{ \sum_{a^\prime \in \gA}{ \exp\{ \theta(s, a^\prime) \} } }, \text{ for all } (s, a) \in \gS \times \gA.
\end{align}
}%

Given a policy $\pi$, its state-action value is defined as $Q^\pi(s, a) \coloneqq r(s, a) + \gamma \cdot \sum_{s^\prime}{ \gP( s^\prime | s, a) \cdot V^\pi(s^\prime) }$, and its advantage function is defined as $A^\pi(s,a) \coloneqq Q^\pi(s, a) - V^\pi(s)$, for $(s,a) \in \gS \times \gA$. 
The state distribution of $\pi$ is defined as $d_{s_0}^{\pi}(s) \coloneqq (1 - \gamma) \cdot \sum_{t=0}^{\infty}{ \gamma^t \cdot \probability(s_t = s | s_0, \pi, \gP) }$. We also denote $d_{\rho}^{\pi}(s) \coloneqq \expectation_{s_0 \sim \rho(\cdot)}{\left[ d_{s_0}^{\pi}(s) \right]}$. Given $\rho$, there exists an optimal policy $\pi^*$ such that $V^{\pi^*}(\rho) = \max_{\pi : \gS \to \Delta(\gA)}{ V^{\pi}(\rho) }$. We denote $V^*(\rho) \coloneqq V^{\pi^*}(\rho)$ for conciseness.

For a general MDP, we assume the initial state distribution $\mu$ is ``sufficiently exploratory'' \citep{agarwal2021theory,mei2020global,laroche2021dr}.
\begin{assumption}[Sufficient exploration]
\label{assump:pos_init} 
The initial state distribution satisfies $\min_s \mu(s) > 0$.
\end{assumption}
At iteration $t$, the NPG method uses the current state distribution to sample one state $s_t \sim d_{\mu}^{\pi_{\theta_t}}(\cdot)$, then uses on-policy sampling to sample one action $a_t \sim \pi_{\theta_t}(\cdot | s)$. 
For the sampled state action pair $(s_t, a_t) \in \gS \times \gA$, the state-action value $Q^{\pi_{\theta_t}}(s_t,a_t)$ is then used to perform update. 
The current state value function $V^{\pi_{\theta_t}}(s_t)$ is used as the baseline, as shown in \cref{alg:softmax_natural_pg_general_on_policy_stochastic_gradient_deterministic_value}.

\begin{figure}[ht]
\centering
\vskip -0.2in
\begin{minipage}{.6\linewidth}
    \begin{algorithm}[H]
    \caption{NPG, on-policy stochastic natural gradient}
    \label{alg:softmax_natural_pg_general_on_policy_stochastic_gradient_deterministic_value}
    \begin{algorithmic}
    \STATE {\bfseries Input:} Learning rate $\eta > 0$.
    \STATE {\bfseries Output:} Policies $\pi_{\theta_t} = \softmax(\theta_t)$.
    \STATE Initialize parameter $\theta_1(s,a)$ for all $(s,a) \in \gS \times \gA$.
    \WHILE{$t \ge 1$}
    \STATE Sample $s_t \sim d_{\mu}^{\pi_{\theta_t}}(\cdot)$, and $a_t \sim \pi_{\theta_t}(\cdot | s_t)$.
    \STATE $\theta_{t+1}(s_t, a_t) \gets \theta_{t}(s_t, a_t) + \eta \cdot \frac{ \textcolor{red}{ Q^{\pi_{\theta_{t}}}(s_t, a_t)} - V^{\pi_{\theta_t}}(s_t)}{ \pi_{\theta_t}(a_t | s_t) }$.
    \ENDWHILE
    \end{algorithmic}
    \end{algorithm}
\end{minipage}
\end{figure}

According to the performance difference lemma, 
we have,
{\small 
\begin{align}
    V^{\pi_{\theta_{t+1}}}(\mu) - V^{\pi_{\theta_t}}(\mu) = \frac{1}{1 - \gamma} \cdot \sum_{s} d_{\mu}^{\pi_{\theta_{t+1}}}(s) \cdot \sum_{a} \left( \pi_{\theta_{t+1}}(a | s) - \pi_{\theta_{t}}(a | s) \right) \cdot Q^{\pi_{\theta_{t}}}(s, a),
\end{align}
}%
where the inner summation over actions is similar to $\left( \pi_{\theta_{t+1}} - \pi_{\theta_t} \right)^\top r$ in one-state MDPs. 
This connection allows us to generalize \cref{lem:non_uniform_lojasiewicz_stochastic_npg_value_baseline_special} to the following result.
\begin{lemma}[Stochastic N\L{}]
\label{lem:non_uniform_lojasiewicz_stochastic_npg_value_baseline_general}
Using \cref{alg:softmax_natural_pg_general_on_policy_stochastic_gradient_deterministic_value} with constant $\eta > 0$, we have, for all $t \ge 1$, 
\begin{equation}
    V^{\pi_{\theta_{t+1}}}(s_0) - V^{\pi_{\theta_t}}(s_0) \ge 0,  \qquad \text{ a.s., } \qquad \forall s_0 \in \gS, \qquad \text{and}
\vspace{-15pt}
\end{equation}
\begin{equation}
\resizebox{.9\hsize}{!}{$\EEt{ V^{\pi_{\theta_{t+1}}}(\mu) } - V^{\pi_{\theta_t}}(\mu) \ge \frac{\eta \cdot \left( 1 - \gamma \right)^4 \cdot \min_{s}{\mu(s) } }{1 + \eta}  \cdot \Big\| \frac{d_{\mu}^{\pi^*}}{\mu} \Big\|_\infty^{-1} \cdot \frac{ \textcolor{red}{ \min_{s}{ \pi_{\theta_t}(a^*(s) | s)^2 } } }{S} \cdot \big( V^{\pi^*}(\mu) - V^{\pi_{\theta_t}}(\mu) \big)^2.$}
\end{equation}
where $\EEt{\cdot}$ is on randomness from state sampling $s_t \sim d_\mu^{\pi_{\theta_t}}(\cdot)$, on-policy sampling $a_t \sim \pi_{\theta_t}(\cdot | s_t)$, and $a^*(s)$ is the action selected by the optimal policy $\pi^*$ under state $s$.
\end{lemma}
Next, similar to \cref{lem:non_vanishing_nl_coefficient_stochastic_npg_value_baseline_special}, we can develop a set of contradictions that establish the following result.
\begin{lemma}[Non-vanishing stochastic N\L{} coefficient / ``automatic exploration'']
\label{lem:non_vanishing_nl_coefficient_stochastic_npg_value_baseline_general}
Using \cref{alg:softmax_natural_pg_general_on_policy_stochastic_gradient_deterministic_value} with the conditions in \cref{lem:non_uniform_lojasiewicz_stochastic_npg_value_baseline_general}, with arbitrary initialization $\theta_1 \in \sR^{\gS \times \gA}$, we have,
\begin{align}
\label{eq:non_vanishing_nl_coefficient_stochastic_npg_value_baseline_general_result_1}
    c \coloneqq \inf_{t \ge 1, s \in \gS} \pi_{\theta_t}(a^*(s) | s) > 0, \qquad \text{a.s.}
\end{align}
\end{lemma}
By combining \cref{lem:non_uniform_lojasiewicz_stochastic_npg_value_baseline_general,lem:non_vanishing_nl_coefficient_stochastic_npg_value_baseline_general}, we obtain the following result that generalizes \cref{thm:almost_sure_convergence_rate_stochastic_npg_special_value_baseline}.
\begin{theorem}[Almost sure global convergence rate]
\label{thm:almost_sure_convergence_rate_stochastic_npg_general_value_baseline}
Using \cref{alg:softmax_natural_pg_general_on_policy_stochastic_gradient_deterministic_value} with any initialization $\theta_1 \in \sR^K$,  under the same assumptions as Lemmas \ref{lem:non_uniform_lojasiewicz_stochastic_npg_value_baseline_general}, there exists a $C>0$ such that for all $t \ge 1$,
{\small 
\begin{align}
    \EE{ V^*(\mu) - V^{\pi_{\theta_t}}(\mu) } \le \frac{1 + \eta}{\eta \cdot \left( 1 - \gamma \right)^4 \cdot \min_{s}{\mu(s) } }  \cdot \bigg\| \frac{d_{\mu}^{\pi^*}}{\mu} \bigg\|_\infty \cdot \frac{S}{ \EE{ c^2 } } \cdot \frac{1}{t},& \qquad \text{and} \\
    \limsup_{t \ge 1} \bigg\{ \frac{\eta \cdot \left( 1 - \gamma \right)^4 \cdot \min_{s}{\mu(s) } }{1 + \eta}  \cdot \bigg\| \frac{d_{\mu}^{\pi^*}}{\mu} \bigg\|_\infty^{-1} \cdot \frac{ c^2  \cdot t}{S} \cdot \left( V^*(\mu) - V^{\pi_{\theta_t}}(\mu) \right) \bigg\} < \infty,& \qquad \text{a.s.},
\end{align}
}%
where $\pi^*$ is the global optimal policy, $S$ is the state number, $\min_{s}{\mu(s)} > 0$ by \cref{assump:pos_init}, and $c \coloneqq \inf_{t \ge 1, s \in \gS} \pi_{\theta_t}(a^*(s) | s) > 0$ is from \cref{lem:non_vanishing_nl_coefficient_stochastic_npg_value_baseline_general}.
\end{theorem}

\section{Understanding Baselines in On-policy Stochastic Policy Optimization}
\label{sec:stochastic_policy_optimization_circles}

\cref{sec:on_policy_npg} shows that using a value function baseline in on-policy stochastic NPG can ensure convergence to a globally optimal policy. 
However, the mechanism behind this finding requires further elucidation. 
Preliminary studies \citep{chung2020beyond,mei2021understanding} have observed that subtracting a baseline can reduce the committal behavior of PG-based estimators, suggesting that this effect might be more important than variance reduction. 
A mathematical characterization of ``committal behavior'' is from using the following concept of ``committal rate'' \citep{mei2021understanding}.
\begin{definition}[Committal Rate, Definition 2 of \cite{mei2021understanding}]
\label{def:committal_rate}
Fix $r \in (0, 1]^K$ 
and $\theta_1 \in \sR^K$.
Consider a policy optimization algorithm $\gA$. Let action $a$ be the sampled action \textbf{forever} after initialization and let $\theta_t$ be produced by $\gA$ on the first $t$ observations.
The committal rate of algorithm $\gA$ on action $a$ (given $r$ and $\theta_1$) is,  
\begin{align}
    \kappa(\gA, a) = \sup\left\{ \alpha \ge 0: \limsup_{t \to \infty}{ t^\alpha \cdot  \left[ 1 - \pi_{\theta_t}(a) \right] < \infty} \right\}.
\end{align}
\end{definition}
The larger the committal rate $\kappa$ is, the more aggressive one update is.
In this section, we provide a new, deeper understanding of how a baseline improves the convergence behaviour of a stochastic PG based method using \cref{def:committal_rate}. However, \cite{mei2021understanding} only studied the deterministic reward setting, i.e., $\hat{r}_t$ is from \cref{def:simplified_on_policy_importance_sampling}. We follow the same deterministic reward setting in this section.

\subsection{Baselines Do Not Control Update Variance in NPG}

We begin from the well known result that value baselines have no effect on exact policy gradients.
\begin{proposition}[Unbiasedness of NPG]
\label{prop:softmax_natural_pg_unbiased}
For NPG with and without a state value baseline, corresponding to \cref{update_rule:softmax_natural_pg_special_on_policy_stochastic_gradient,update_rule:softmax_natural_pg_special_on_policy_stochastic_gradient_value_baseline} respectively,
we have 
$\expectation_{a_t \sim \pi_{\theta_t}(\cdot)}{ \left[ \hat{r}_t \right] } = \expectation_{a_t \sim \pi_{\theta_t}(\cdot)}{ [ \hat{r}_t - \hat{b}_t ] } = r$. 
\end{proposition}
According to \cref{prop:softmax_natural_pg_unbiased}, \cref{update_rule:softmax_natural_pg_special_on_policy_stochastic_gradient,update_rule:softmax_natural_pg_special_on_policy_stochastic_gradient_value_baseline} become identical if the exact policy gradient is available,
hence both enjoy an
$O(e^{-c \cdot t})$ convergence rate to a global optimum ($c > 0$) \citep{khodadadian2021linear,mei2021understanding}. 
Therefore, a state value baseline can only have an effect if the policy gradient has to be estimated from a stochastic sample. 
However, we find that the variance of the NPG updates remains unbounded in the stochastic setting, regardless of whether a state value baseline is used. 
\begin{proposition}[Unboundedness of NPG]
\label{prop:softmax_natural_pg_variances}
For NPG without a baseline, \cref{update_rule:softmax_natural_pg_special_on_policy_stochastic_gradient},
we have $\expectation_{a_t \sim \pi_{\theta_t}(\cdot)}{ \left\| \hat{r}_t \right\|_2^2 } = \sum_{a \in [K]}{ \frac{ r(a)^2 }{ \pi_{\theta_t}(a) }  }$. 
For NPG with a state value baseline, \cref{update_rule:softmax_natural_pg_special_on_policy_stochastic_gradient_value_baseline},
we have $\expectation_{a_t \sim \pi_{\theta_t}(\cdot)}{ \| \hat{r}_t - \hat{b}_t \|_2^2 } = \sum_{a \in [K]}{ \frac{ ( r(a) - \pi_{\theta_t}^\top r )^2 }{ \pi_{\theta_t}(a) } } - K \cdot ( \pi_{\theta_t}^\top r)^2 + 2 \cdot ( \pi_{\theta_t}^\top r) \cdot ( r^\top \rvone )$.
\end{proposition}
According to \cref{prop:softmax_natural_pg_variances}, whenever $\pi_{\theta_t}$ nears a one-hot probability distribution over $[K]$ (which it must converge to),
there exists at least one action $a \in [K]$ such that both $\frac{ r(a)^2 }{ \pi_{\theta_t}(a) }$ and $\frac{ ( r(a) - \pi_{\theta_t}^\top r )^2 }{ \pi_{\theta_t}(a) }$ become unbounded, implying an unbounded scale for both \cref{update_rule:softmax_natural_pg_special_on_policy_stochastic_gradient,update_rule:softmax_natural_pg_special_on_policy_stochastic_gradient_value_baseline}. 
Yet we know from \cref{prop:failure_probability_softmax_natural_pg_special_on_policy_stochastic_gradient} that not using a baseline fails with positive probability, while from \cref{thm:almost_sure_convergence_rate_stochastic_npg_special_value_baseline} subtracting a state value baseline ensures almost sure convergence to a global optimum.
The fact that the variance of both updates is unbounded suggests that it is difficult to draw conclusions on the effect of the baseline from a variance reduction perspective alone.
An alternative analysis is required to explain the fundamental difference between \cref{update_rule:softmax_natural_pg_special_on_policy_stochastic_gradient,update_rule:softmax_natural_pg_special_on_policy_stochastic_gradient_value_baseline}.

\subsection{Coupled Sampling and Updating}

In on-policy stochastic policy optimization, sampling and updating are coupled as shown in \cref{fig:on_policy_coupling}. 
At iteration $t$, the data collected depends
on the current policy, since on-policy sampling is used $a_t \sim \pi_{\theta_t}(\cdot)$,
while the policy is updated from the observations collected based on $a_t$.
%
\begin{wrapfigure}{r}{0.3\textwidth}
  \begin{center}
    \includegraphics[width=0.3\textwidth]{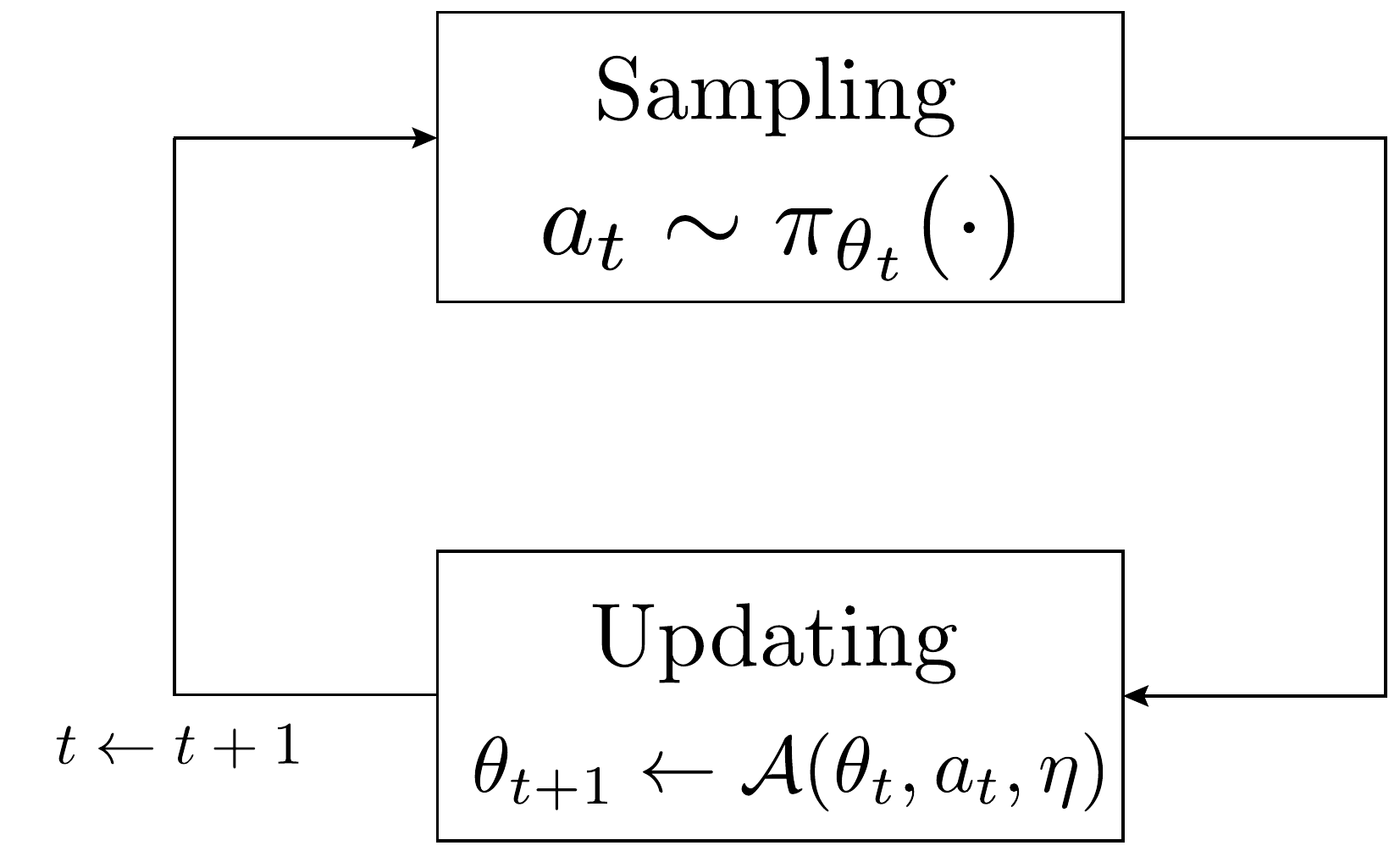}
  \end{center}
\caption{Coupled on-policy sampling and updating \citep[Figure 2]{mei2021understanding}.}
\label{fig:on_policy_coupling}
\end{wrapfigure}
This coupling introduces complexity in the optimization process as well as in the analysis. 
However, this coupling is also fundamental to understanding the circular interaction created by any on-policy stochastic optimization method.
That is, on-policy stochastic optimization faces an exploration-exploitation dilemma: 
a learning algorithm can improve the policy and increase the probability of choosing actions that yield higher rewards (exploitation),
but it must not do so too aggressively lest it fail to identify possibly higher-reward actions (exploration). 
Striking a proper balance between exploration and exploitation is key to achieving good convergence properties.
Different levels of update aggression create different circular effects between sampling and updating, which is central to determining almost sure convergence to a global optimum.

\subsection{The ``Vicious Circle'' of Being Too Aggressive}

First we illustrate a negative effect, the ``vicious circle'' of being too aggressive.
\begin{lemma}[Bad sampling]
\label{lem:positive_infinite_product}
Let $\pi_{\theta_t}(a) \in (0, 1)$ be the probability of sampling action $a$ using online sampling $a_t \sim \pi_{\theta_t}(\cdot)$, for all $t \ge 1$. If $1 - \pi_{\theta_t}(a) \in O(1/t^{1+ \epsilon})$, where $\epsilon > 0$, then $\prod_{t=1}^{\infty}{ \pi_{\theta_t}(a) } > 0$.
\end{lemma}
Note that \cref{lem:positive_infinite_product} characterizes sampling behaviour under general conditions that do not otherwise depend on specific updates. 
However, according to \cref{lem:positive_infinite_product}, if an action's probability approaches $1$ strictly faster than $O(1/t)$, by whatever means, it becomes possible to not sample any other action forever, which creates a ``lack of exploration'' phenomenon as it is known in RL. 
In particular, on-policy stochastic NPG without a baseline can  produce such a sequence of $\{ \pi_{\theta_t}(a) \}_{t \ge 1}$.
\begin{lemma}[NPG aggressiveness]
\label{lem:npg_aggressiveness}
Fix sampling $a_t = a$ for all $t \ge 1$, using \cref{update_rule:softmax_natural_pg_special_on_policy_stochastic_gradient} with constant learning rate $\eta > 0$, where $\hat{r}_t$ is from \cref{def:simplified_on_policy_importance_sampling}, we have $1 - \pi_{\theta_t}(a) \in O(e^{-c \cdot t})$ for all $t \ge 1$, where $c > 0$.
\end{lemma}
According to \cref{def:committal_rate}, we have $\kappa(\text{NPG}, a) = \infty$, meaning that NPG without baseline is very aggressive.
Note that \cref{lem:npg_aggressiveness} only characterizes the aggressiveness of 
\cref{update_rule:softmax_natural_pg_special_on_policy_stochastic_gradient} with the sampling fixed to be $a_t = a$ for all $t \ge 1$. 
\cref{lem:positive_infinite_product,lem:npg_aggressiveness} together describe the ``vicious circle'' between sampling and updating that can be created by overly aggressive updates. 
\textcolor{blue}{First}, in on-policy sampling, there will always be a non-zero probability of ``bad luck''; that is, with positive probability a set of sub-optimal actions can be sequentially sampled for multiple steps. 
\textcolor{blue}{Second}, an overly aggressive update will only exaggerate the weakness of the sampling procedure by increasing the sampled sub-optimal actions' probabilities rapidly (\cref{{lem:npg_aggressiveness}}). 
\textcolor{blue}{Third}, this exaggeration can worsen data collection for subsequent updating by further increasing the prevalence of sub-optimal actions.
Such a vicious circular interaction between sampling and updating can happen repeatedly, and its self-reinforcing nature can create a non-zero probability that the cycle occurs forever (\cref{{lem:positive_infinite_product}}), resulting in convergence to a sub-optimal deterministic policy (a stationary point for both sampling and updating).

\subsection{The ``Virtuous Circle'' of Not Being Too Aggressive}

Next, we demonstrate a positive effect, the ``virtuous circle'' of not being too aggressive.
\begin{lemma}[Good sampling]
\label{lem:zero_infinite_product}
Let $\pi_{\theta_t}(a) \in (0, 1)$ and $a_t \sim \pi_{\theta_t}(\cdot)$, for all $t \ge 1$. If $\sum_{t=1}^{\infty}{ \left( 1 - \pi_{\theta_t}(a) \right) }$ $= \infty$ (e.g., $1 - \pi_{\theta_t}(a) \in \Omega(1/t)$), then $\prod_{t=1}^{\infty}{ \pi_{\theta_t}(a) } = 0$.
\end{lemma}
As in \cref{lem:positive_infinite_product}, \cref{lem:zero_infinite_product} only characterizes the effect of sampling behaviour under general conditions that do not otherwise depend on specific updates. 
Here we see that if an action's probability approaches $1$ no faster than $O(1/t)$, it is no longer possible to avoid sampling any other action forever; that is, sufficiently slow modification of the sampling probabilities forces persistent exploration such that every action is sampled within some finite time with probability $1$. 
In particular, subtracting a value baseline in on-policy stochastic NPG produces such a sequence $\{ \pi_{\theta_t}(a) \}_{t \ge 1}$.
\begin{lemma}[Value baselines reduce NPG aggressiveness]
\label{lem:npg_aggressiveness_value_baseline}
Fix sampling $a_t = a$ for all $t \ge 1$. 
Then using \cref{update_rule:softmax_natural_pg_special_on_policy_stochastic_gradient_value_baseline} with a constant learning rate $\eta > 0$ and $\hat{r}_t$ from \cref{def:simplified_on_policy_importance_sampling} obtains $1 - \pi_{\theta_t}(a) \in \Omega(1/t)$ for all $t \ge 1$.
\end{lemma}
According to \cref{def:committal_rate}, with value baselines, we have $\kappa(\text{NPG}, a) = 1$, meaning that the aggressiveness of NPG update is reduced.
As in \cref{lem:npg_aggressiveness}, \cref{lem:npg_aggressiveness_value_baseline} only characterizes the conservativeness of \cref{update_rule:softmax_natural_pg_special_on_policy_stochastic_gradient_value_baseline} with fixed sampling of $a_t = a$ for all $t \ge 1$. 
\cref{lem:zero_infinite_product,lem:npg_aggressiveness_value_baseline} now describe a ``virtuous circle'' between sampling and updating that is created by using not too aggressive updates. 
\textcolor{blue}{First}, even in a worst case situation (e.g., an adversarial initialization), where a sub-optimal action has a dominant probability $\pi_{\theta_t}(a) \approx 1$, under on-policy sampling all actions will eventually be sampled. \textcolor{blue}{Second}, conservative updating will mitigate the effect of the extreme sampler by not increasing the sub-optimal action's probability too rapidly (\cref{{lem:npg_aggressiveness_value_baseline}}). 
\textcolor{blue}{Third}, sustained diversity in sampling will eventually draw a better action than the current dominating sub-optimal action (\cref{{lem:zero_infinite_product}}). 
\textcolor{blue}{Finally}, once better actions are sampled, the update will improve subsequent sampling by decreasing the probability of the dominating sub-optimal action. In particular, this is achieved by increasing value baselines to be larger than the dominating sub-optimal action's true mean reward, such that the dominating sub-optimal action will start losing probabilities.
This virtuous circular interaction between sampling and updating ensures sufficient exploration,
which prevents the iteration from converging to a sub-optimal deterministic policy.

\subsection{How a State Value Baseline Reduces Update Aggressiveness} 
Based on \cref{lem:positive_infinite_product,lem:zero_infinite_product}, the boundary between ``too aggressive'' and ``not too aggressive'' is precisely $\Theta(1/t)$. 
We now explain how a state value baseline in NPG will control update aggressiveness. 
\textcolor{blue}{First}, without a baseline, sampling a sub-optimal action $a \in [K]$ for $t$ times makes its parameter behave as $\theta_t(a) \in \Theta(t)$, since $r(a) \in \Theta(1)$. 
On the other hand, other action parameters will behave as $\theta_t(a^\prime) \in \Theta(1)$ if they are only sampled a constant number of times. 
Under the softmax parameterization \cref{eq:softmax}, this will imply that $1 - \pi_{\theta_t}(a) \in O(e^{-c \cdot t})$, which is far too aggressive. 
\textcolor{blue}{Second},
using a state value baseline, under repeated sampling the parameter increase for a sub-optimal action $a \in [K]$ will be damped. 
In particular, whenever the policy is close to deterministic, say $\pi_{\theta_t}(a) \approx 1$, 
we also have $\pi_{\theta_t}^\top r \approx r(a)$. 
Therefore, since
\begin{align}
    r(a) - \pi_{\theta_t}^\top r = \sum_{a^\prime \not= a}{ \pi_{\theta_t}(a^\prime) \cdot \left( r(a) - r(a^\prime) \right) } \le 1 - \pi_{\theta_t}(a),
\end{align}
the closer $1 - \pi_{\theta_t}(a)$ is to $0$, the smaller $r(a) - \pi_{\theta_t}^\top r$ will be.
This means even if $a$ is sampled repeatedly for $t$ times, we obtain $\theta_t(a) \in O(\log{t})$ and $1 - \pi_{\theta_t}(a) \in \Omega(1/t)$ (\cref{{lem:npg_aggressiveness_value_baseline}}). 
Thus, the effect of baseline is to modify the sampling to lie exactly on the boundary of being good enough. 
From this argument the key role of the value baseline is to reduce update aggressiveness to achieve a particular  effect on long-term sampling, rather than simply reduce variance.
It also shows how using an appropriately un-aggressive update is both necessary (\cref{lem:positive_infinite_product}) and sufficient (\cref{lem:zero_infinite_product}) to achieve almost sure convergence to a global optimum in on-policy stochastic policy optimization.

\section{Simulations}
\label{sec:simulations}

We conducted simulations to verify the two main results above: asymptotic convergence toward globally optimal policy $\pi^*$ in \cref{lem:non_vanishing_nl_coefficient_stochastic_npg_value_baseline_special}, and the $O(1/t)$ convergence rate in \cref{thm:almost_sure_convergence_rate_stochastic_npg_special_value_baseline}.

\subsection{Asymptotic Convergence}

We first consider a one-state MDP with $K = 20$ actions and true mean reward vector $r \in (0, 1)^K$, where the optimal action is $a^* = 1$ with true mean reward $r(1) \approx 0.97$ and best sub-optimal action's true mean reward $r(2) \approx 0.95$. The sampled reward is observed with a large noise, e.g., $x \approx -2.03$ and $x \approx 3.97$ with both $0.5$ probability for the optimal action, such that $r(1) \approx 0.5 \cdot (-2.03) + 0.5 \cdot 3.97$. Details about $r$ and the reward distributions can be found in the appendix.

To verify asymptotic convergence to a globally optimal policy in \cref{lem:non_vanishing_nl_coefficient_stochastic_npg_value_baseline_special}, we consider the iteration behaviors of \cref{update_rule:softmax_natural_pg_special_on_policy_stochastic_gradient_value_baseline} under an adversarial initialization, where $\pi_{\theta_1}(2) \approx 0.88$, i.e., a sub-optimal action starts with a dominating probability. 
This is the worst case scenario for \cref{lem:non_vanishing_nl_coefficient_stochastic_npg_value_baseline_special}, where the optimal action only has a small chance to be sampled, while the sampled reward noise is very large.

As shown in \cref{fig:adversarial_initialization_expected_reward}, the expected reward $\pi_{\theta_t}^\top r$ quickly approaches and remains stuck around $r(2) \approx 0.95$ initially,  as expected. 
However, after about $8 \times 10^6$ iterations, the policy $\pi_{\theta_t}$ finally escapes the sub-optimal plateau and approaches the optimal reward $r(1) \approx0.97$. 
This simulation result is consistent with  \cref{lem:non_vanishing_nl_coefficient_stochastic_npg_value_baseline_special}, i.e., for an arbitrary initialization, the introduction of a value baseline eventually makes $\pi_{\theta_t}$ approach a globally optimal policy within finite time, while additionally the optimal action's probability never vanishes, $\inf_{t \ge 1}{ \pi_{\theta_t}(a^*) } > 0$, as shown in \cref{fig:adversarial_initialization_optimal_action_probability}.

\subsection{Convergence Rate}

We run \cref{update_rule:softmax_natural_pg_special_on_policy_stochastic_gradient_value_baseline} with a uniform initialization, i.e., $\pi_{\theta_1}(a) = 1/K$ for all $a \in [K]$, and calculate averaged sub-optimality gap $\left( \pi^* - \pi_{\theta_t} \right)^\top r$ across $20$ independent runs, using deterministic reward settings where $\hat{r}_t$ is from \cref{def:simplified_on_policy_importance_sampling}. As shown in \cref{fig:uniform_initialization_sub_optimality_gap}, where both axes are in $\log$ scale, the slope is approximately $-1$, indicating that $\log{ \left( \pi^* - \pi_{\theta_t}\right)^\top r } = - \log{t} + C$, or equivalently $\left( \pi^* - \pi_{\theta_t}\right)^\top r = C^\prime/ t$, which is consistent with \cref{thm:almost_sure_convergence_rate_stochastic_npg_special_value_baseline}.
\begin{figure}
\centering
\begin{subfigure}[b]{.325\linewidth}
\includegraphics[width=\linewidth]{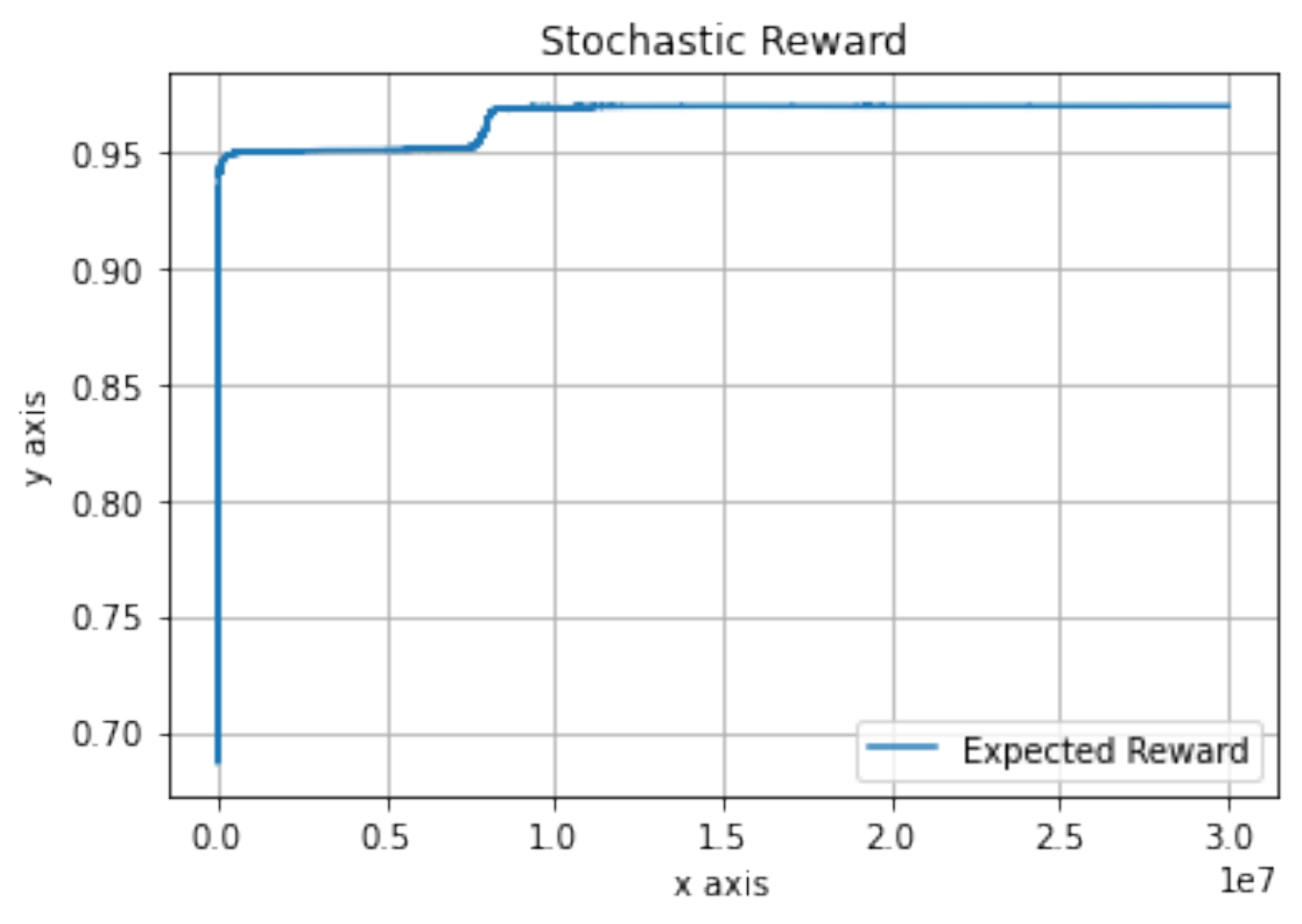}
\caption{$\pi_{\theta_t}^\top r$.}\label{fig:adversarial_initialization_expected_reward}
\end{subfigure}
\begin{subfigure}[b]{.325\linewidth}
\includegraphics[width=\linewidth]{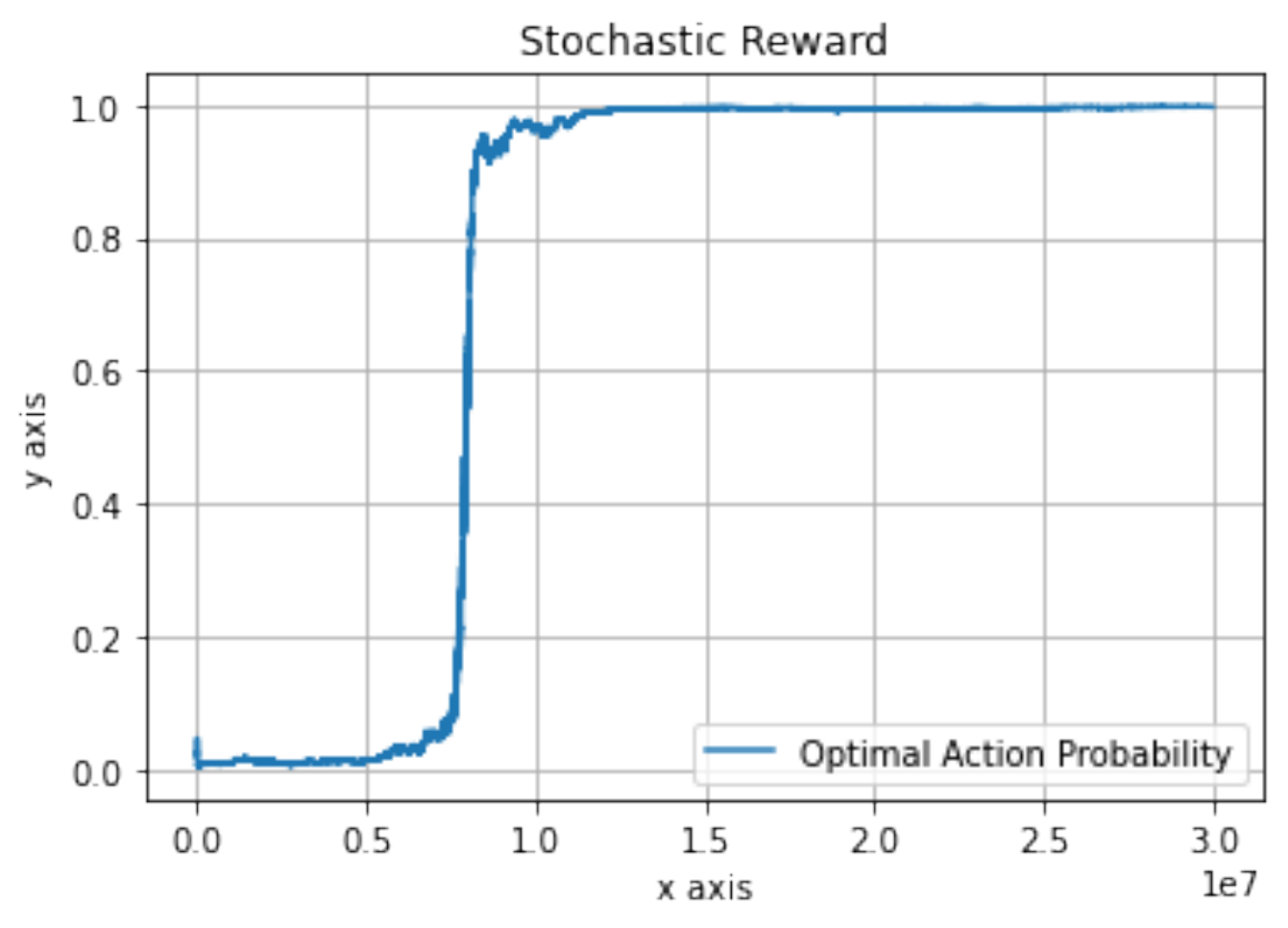}
\caption{$\pi_{\theta_t}(a^*)$.}\label{fig:adversarial_initialization_optimal_action_probability}
\end{subfigure}
\begin{subfigure}[b]{.325\linewidth}
\includegraphics[width=\linewidth]{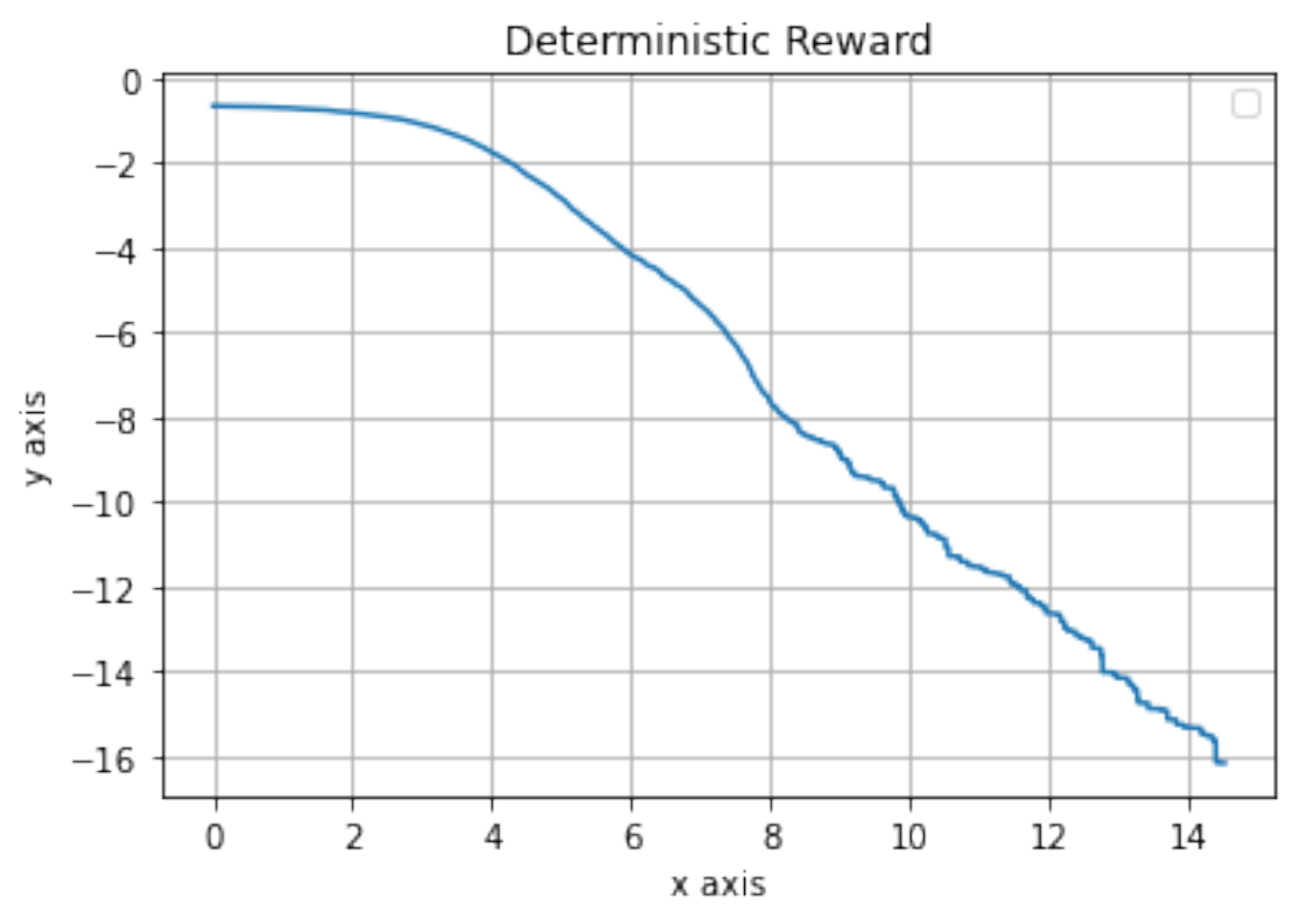}
\caption{$\log{\left( \pi^* - \pi_{\theta_t}\right)^\top r}$.}\label{fig:uniform_initialization_sub_optimality_gap}
\end{subfigure}
\caption{Adversarial initialization (a) and (b); uniform initialization (c).}
\label{fig:adversarial_initialization_uniform_initialization}
\vspace{-10pt}
\end{figure}

\section{Conclusion}
\label{sec:discussions_conclusions_future_work}

This work clarifies some of the longstanding mysteries those have separated the theory and practice of policy gradient optimization. The major finding is a state value baseline reduces the aggressiveness of the on-policy stochastic NPG update, which turns out to be necessary and sufficient for achieving almost sure convergence to a global optimum. 
The deeper understanding of the circular dependence between on-policy sampling and updating also dispels a common misconception about variance reduction, showing that bounded variance estimators are not necessary for achieving global convergence. 
The main technical innovation is the stochastic N\L{} inequality, and the subsequent arguments that establish global convergence, both of which depend critically on the value baseline. 

This work leaves open a number of interesting questions. 
\textit{First}, the $O(1/t)$ convergence rate contains an initialization dependent constant in \cref{lem:non_vanishing_nl_coefficient_stochastic_npg_value_baseline_special}, resulting from plateaus as observed in \cref{fig:adversarial_initialization_expected_reward}, which does not appear in results that use the direct parameterization \citep{denisov2020regret}. Thus the difficulty appears due to the non-linear softmax transform. 
Removing or improving this constant would impact practical performance, so investigating other techniques, such as regularization, optimism or momentum might be helpful. 
\textit{Second}, the results in this paper use the true state values as the baselines. 
It would be interesting to consider the effect of estimating the value baseline or using alternative baselines in policy optimization.
\textit{Finally}, the $O(1/t)$ last iteration convergence rate implies an optimal $O(\log{T})$ regret in stochastic bandit problems \citep{lai1985asymptotically}. 
The explanation of the circular dependence between sampling and updating is specific to on-policy PG optimization, but it is also consistent with the exploration exploitation dilemma in RL. 
In other words, this work suggests a completely new approach to the exploration-exploitation trade-off,
achieving provable bounds with ever requiring explicit uncertainty estimates, nor any concrete instantiation of the principle of optimism under uncertainty.

\begin{ack}
The authors would like to thank anonymous reviewers for their valuable comments. Jincheng Mei thanks Alekh Agarwal for reviewing a draft of this work. Csaba Szepesv\'ari and Dale Schuurmans gratefully acknowledge funding from the Canada
CIFAR AI Chairs Program, Amii and NSERC.
\end{ack}

{\small
\bibliography{neurips_refs}
}
\bibliographystyle{plain}

\newpage
\appendix

\begin{center}
\LARGE \textbf{Appendix}
\end{center}

The appendix is organized as follows.

\etocdepthtag.toc{mtappendix}
\etocsettagdepth{mtmainpaper}{none}
\etocsettagdepth{mtappendix}{subsubsection}
\begingroup
\parindent=0em
\etocsettocstyle{\rule{\linewidth}{\tocrulewidth}\vskip0.5\baselineskip}{\rule{\linewidth}{\tocrulewidth}}
\tableofcontents 
\endgroup

\section{Proofs for One-state MDPs}

\textbf{\cref{lem:non_uniform_lojasiewicz_stochastic_npg_value_baseline_special} }(Stochastic non-uniform \L{}ojasiewicz (N\L{}))\textbf{.}
Suppose \cref{assump:bounded_reward} holds. 
Let $r \in [0,1]^K$, $a^* \coloneqq \argmax_{a \in [K]}{ r(a) }$ denote the optimal action, and $\Delta \coloneqq r(a^*) - \max_{a \not= a^*}{ r(a) }$ denote the reward gap. Using \cref{update_rule:softmax_natural_pg_special_on_policy_stochastic_gradient_value_baseline} with on-policy sampling $a_t \sim \pi_{\theta_t}(\cdot)$ and IS estimator $\hat{r}_t$,
\begin{description}[style=unboxed,leftmargin=0cm]
    \item[(1)] if $\hat{r}_t$ is from \cref{def:simplified_on_policy_importance_sampling}, then with constant learning rate $\eta > 0$, we have, for all $t \ge 1$,
\begin{align}
\label{eq:non_uniform_lojasiewicz_stochastic_npg_value_baseline_special_result_1a_appendix}
    \pi_{\theta_{t+1}}^\top r - \pi_{\theta_t}^\top r &\ge 0, \qquad \text{almost surely (a.s.),} \qquad \text{and} \\
\label{eq:non_uniform_lojasiewicz_stochastic_npg_value_baseline_special_result_1b_appendix}
    \EEt{\pi_{\theta_{t+1}}^\top r} - \pi_{\theta_t}^\top r &\ge \frac{ \eta }{ 1 + \eta } \cdot \pi_{\theta_t}(a^*) \cdot \left( r(a^*) - \pi_{\theta_t}^\top r \right)^2,
\end{align}
where $\EEt{\cdot}$ is on randomness from on-policy sampling $a_t \sim \pi_{\theta_t}(\cdot)$.
    \item[(2)] if $\hat{r}_t$ is from \cref{def:on_policy_importance_sampling}, then with learning rate,
\begin{align}
\label{eq:non_uniform_lojasiewicz_stochastic_npg_value_baseline_special_result_2a_appendix}
    \eta = \frac{\pi_{\theta_t}(a_t) \cdot \left| r(a_t) - \pi_{\theta_t}^\top r \right|}{8 \cdot R_{\max}^2},
\end{align}
\noindent we have, for all $t \ge 1$,
\begin{align}
\label{eq:non_uniform_lojasiewicz_stochastic_npg_value_baseline_special_result_2b_appendix}
    \EEt{\pi_{\theta_{t+1}}^\top r} - \pi_{\theta_t}^\top r &\ge \frac{1}{16 \cdot R_{\max}^2} \cdot \sum_{i = 1}^{K} \pi_{\theta_t}(i)^2 \cdot \left| r(i) - \pi_{\theta_t}^\top r \right|^3 \\
    &\ge \frac{ 1 }{16 \cdot R_{\max}^2} \cdot \frac{\Delta}{K-1} \cdot \pi_{\theta_t}(a^*)^2 \cdot \left( r(a^*) - \pi_{\theta_t}^\top r \right)^2,
\end{align}
where $\EEt{\cdot}$ is on randomness from on-policy sampling $a_t \sim \pi_{\theta_t}(\cdot)$ and reward sampling $x \sim R_{a_t}$.
\end{description}
\begin{proof}
\textbf{First part. (1)} If $\hat{r}_t$ is from \cref{def:simplified_on_policy_importance_sampling}.

Since the results are concerned with the policies $\{ \pi_{\theta_t} \}_{t \ge 1}$ underlying the parameter $\{ \theta_t \}_{t \ge 1}$ and not the parameter vectors themselves, as noted after 
\cref{update_rule:softmax_natural_pg_special_on_policy_stochastic_gradient_value_baseline},
without loss of generality, in the rest of the proof we assume that the
update over parameter vectors is according to,
\begin{align}
\label{eq:non_uniform_lojasiewicz_stochastic_npg_value_baseline_special_deterministic_reward_intermediate_0}
\theta_{t+1}(a) \gets \theta_t(a) + \eta \cdot \frac{ \sI\left\{ a_t = a \right\} }{ \pi_{\theta_t}(a) } \cdot \left( r(a) - \pi_{\theta_t}^\top r \right)\,.
\end{align}

For all $t \ge 1$, for any action $ i \in [K]$, denote
\begin{align}
\label{eq:non_uniform_lojasiewicz_stochastic_npg_value_baseline_special_deterministic_reward_intermediate_1}
    \big[ \pi_{\theta_{t+1} }^\top r \ | \ a_t = i \big]
\end{align}
as the the value of $\pi_{\theta_{t+1}}^\top r$ given the sampled action $a_t = i$. 

According to \cref{eq:non_uniform_lojasiewicz_stochastic_npg_value_baseline_special_deterministic_reward_intermediate_0,def:simplified_on_policy_importance_sampling}, we have,
\begin{align}
\label{eq:non_uniform_lojasiewicz_stochastic_npg_value_baseline_special_deterministic_reward_intermediate_2}
\MoveEqLeft
	\big[ \pi_{\theta_{t+1}}^\top r \ | \ a_t = i \big] = \frac{\exp\Big\{ \theta_t(i) + \eta \cdot \frac{ r(i) - \pi_{\theta_t}^\top r }{\pi_{\theta_t}(i)} \Big\} \cdot r(i) + \sum_{j \not= i}{ \exp\{ \theta_t(j) \} \cdot r(j) } }{ \exp\Big\{ \theta_t(i) + \eta \cdot \frac{ r(i) - \pi_{\theta_t}^\top r }{\pi_{\theta_t}(i)} \Big\} + \sum_{j \not= i}{ \exp\{ \theta_t(j) \}} } \\
	&= \frac{ \pi_{\theta_t}(i) \cdot \exp\Big\{ \eta \cdot \frac{r(i) - \pi_{\theta_t}^\top r }{\pi_{\theta_t}(i)} \Big\} \cdot r(i) + \sum_{j \not= i}{ \pi_{\theta_t}(j) \cdot r(j) } }{ \pi_{\theta_t}(i) \cdot \exp\Big\{ \eta \cdot \frac{r(i) - \pi_{\theta_t}^\top r }{\pi_{\theta_t}(i)} \Big\} + \sum_{j \not= i}{ \pi_{\theta_t}(j) } },
\end{align}
where the last equation is by dividing $\sum_{a \in [K]}{ \exp\big\{ \theta_t(a) \big\} }$ from both the numerator and the denominator. Therefore, by algebra we have,
\begin{align}
\label{eq:non_uniform_lojasiewicz_stochastic_npg_value_baseline_special_deterministic_reward_intermediate_3}
	\big[ \pi_{\theta_{t+1}}^\top r \ | \ a_t = i \big] - \pi_{\theta_t}^\top r &= \frac{ \left[ \pi_{\theta_t}(i) \cdot \exp\Big\{ \eta \cdot \frac{ r(i) - \pi_{\theta_t}^\top r}{\pi_{\theta_t}(i)} \Big\} - \pi_{\theta_t}(i) \right] \cdot \left( r(i) - \pi_{\theta_t}^\top r \right) }{ \pi_{\theta_t}(i) \cdot \exp\Big\{ \eta \cdot \frac{ r(i) - \pi_{\theta_t}^\top r}{\pi_{\theta_t}(i)} \Big\} + \sum_{j \not= i}{ \pi_{\theta_t}(j) } } \\
	&= \frac{ \left[ \exp\Big\{ \eta \cdot \frac{r(i) - \pi_{\theta_t}^\top r}{\pi_{\theta_t}(i)} \Big\} - 1 \right] \cdot \left( r(i) - \pi_{\theta_t}^\top r \right) }{ \exp\Big\{ \eta \cdot \frac{r(i) - \pi_{\theta_t}^\top r}{\pi_{\theta_t}(i)} \Big\} + \frac{ 1 - \pi_{\theta_t}(i) }{ \pi_{\theta_t}(i) } } \ge 0,
\end{align}
where the last inequality is from $\left( e^{ c \cdot y} - 1 \right) \cdot y \ge 0$ for all $y \in \sR$ with $c \coloneqq \frac{\eta}{\pi_{\theta_t}(i)} > 0$. This proves \cref{eq:non_uniform_lojasiewicz_stochastic_npg_value_baseline_special_result_1a_appendix}, because of $i \in [K]$ is arbitrary.

For all $t \ge 1$, given current policy $\pi_{\theta_t}$, the expected reward of next policy $\pi_{\theta_{t+1}}^\top r$ is a random variable, and the randomness is from on-policy sampling $a_t \sim \pi_{\theta_t}(\cdot)$. The expected progress is,
\begin{align}
\label{eq:non_uniform_lojasiewicz_stochastic_npg_value_baseline_special_deterministic_reward_intermediate_4}
\MoveEqLeft
	\EEt{\pi_{\theta_{t+1}}^\top r} - \pi_{\theta_t}^\top r = \sum_{i = 1}^{K} \pi_{\theta_t}(i) \cdot \EEt{\pi_{\theta_{t+1} }^\top r \ | \ a_t = i } - \pi_{\theta_t}^\top r \qquad \left( a_t \sim \pi_{\theta_t}(\cdot) \right) \\
	&= \sum_{i = 1}^{K} \pi_{\theta_t}(i)  \cdot \left(  \big[ \pi_{\theta_{t+1} }^\top r \ | \ a_t = i \big] - \pi_{\theta_t}^\top r  \right) \\
	&= \sum_{i = 1}^{K} \pi_{\theta_t}(i) \cdot \frac{ \left[ \exp\Big\{ \eta \cdot \frac{r(i) - \pi_{\theta_t}^\top r}{\pi_{\theta_t}(i)} \Big\} - 1 \right] \cdot \left( r(i) - \pi_{\theta_t}^\top r \right) }{ \exp\Big\{ \eta \cdot \frac{r(i) - \pi_{\theta_t}^\top r}{\pi_{\theta_t}(i)} \Big\} + \frac{ 1 - \pi_{\theta_t}(i) }{ \pi_{\theta_t}(i) } } \qquad \left( \text{by \cref{eq:non_uniform_lojasiewicz_stochastic_npg_value_baseline_special_deterministic_reward_intermediate_3}} \right)
\end{align}
where $\big[ \pi_{\theta_{t+1} }^\top r \ | \ a_t = i \big]$ means the value of $\pi_{\theta_{t+1}}^\top r$ given the sampled action $a_i = i$. 

Partition the action set $[K]$ into three parts using $\pi_{\theta_t}^\top r $ as follows,
\begin{align}
\label{eq:non_uniform_lojasiewicz_stochastic_npg_value_baseline_special_deterministic_reward_intermediate_5}
    \gA_t^0 &\coloneqq \left\{ a^0 \in [K]: r(a^0) = \pi_{\theta_t}^\top r \right\}, \\
    \gA_t^+ &\coloneqq \left\{ a^+ \in [K]: r(a^+) > \pi_{\theta_t}^\top r \right\}, \\
    \gA_t^- &\coloneqq \left\{ a^- \in [K]: r(a^-) < \pi_{\theta_t}^\top r \right\}.
\end{align}
From \cref{eq:non_uniform_lojasiewicz_stochastic_npg_value_baseline_special_deterministic_reward_intermediate_4}, we have,
\begin{align}
\label{eq:non_uniform_lojasiewicz_stochastic_npg_value_baseline_special_deterministic_reward_intermediate_6}
    \EEt{\pi_{\theta_{t+1}}^\top r} - \pi_{\theta_t}^\top r &= \sum_{a^+ \in \gA_t^+} \pi_{\theta_t}(a^+) \cdot \frac{ \left[ \exp\Big\{ \eta \cdot \frac{r(a^+) - \pi_{\theta_t}^\top r}{\pi_{\theta_t}(a^+)} \Big\} - 1 \right] \cdot \left( r(a^+) - \pi_{\theta_t}^\top r \right) }{ \exp\Big\{ \eta \cdot \frac{r(a^+) - \pi_{\theta_t}^\top r}{\pi_{\theta_t}(a^+)} \Big\} + \frac{ 1 - \pi_{\theta_t}(a^+) }{ \pi_{\theta_t}(a^+) } } \\
    &\qquad + \sum_{a^- \in \gA_t^-} \pi_{\theta_t}(a^-) \cdot \frac{ \left[ \exp\Big\{ \eta \cdot \frac{r(a^-) - \pi_{\theta_t}^\top r}{\pi_{\theta_t}(a^-)} \Big\} - 1 \right] \cdot \left( r(a^-) - \pi_{\theta_t}^\top r \right) }{ \exp\Big\{ \eta \cdot \frac{r(a^-) - \pi_{\theta_t}^\top r}{\pi_{\theta_t}(a^-)} \Big\} + \frac{ 1 - \pi_{\theta_t}(a^-) }{ \pi_{\theta_t}(a^-) } }.
\end{align}
For any $a^+ \in \gA_t^+$, we have,
\begin{align}
\MoveEqLeft
\label{eq:non_uniform_lojasiewicz_stochastic_npg_value_baseline_special_deterministic_reward_intermediate_7}
    \frac{ \left[ \exp\Big\{ \eta \cdot \frac{r(a^+) - \pi_{\theta_t}^\top r}{\pi_{\theta_t}(a^+)} \Big\} - 1 \right] \cdot \left( r(a^+) - \pi_{\theta_t}^\top r \right) }{ \exp\Big\{ \eta \cdot \frac{r(a^+) - \pi_{\theta_t}^\top r}{\pi_{\theta_t}(a^+)} \Big\} + \frac{ 1 - \pi_{\theta_t}(a^+) }{ \pi_{\theta_t}(a^+) } } \ge \frac{ \eta \cdot \frac{r(a^+) - \pi_{\theta_t}^\top r}{\pi_{\theta_t}(a^+)} \cdot \left( r(a^+) - \pi_{\theta_t}^\top r \right) }{ \eta \cdot \frac{r(a^+) - \pi_{\theta_t}^\top r}{\pi_{\theta_t}(a^+)} + \frac{ 1 }{ \pi_{\theta_t}(a^+) } } \qquad \left( e^x - 1 \ge x > 0 \right) \\
    &= \frac{\eta \cdot \left( r(a^+) - \pi_{\theta_t}^\top r \right)^2 }{ \eta \cdot \left( r(a^+) - \pi_{\theta_t}^\top r \right) + 1} \ge \frac{\eta}{1 + \eta} \cdot \left( r(a^+) - \pi_{\theta_t}^\top r \right)^2. \qquad \left( r \in [0,1]^K \right)
\end{align}
For any $a^- \in \gA_t^-$, we have,
\begin{align}
\MoveEqLeft
\label{eq:non_uniform_lojasiewicz_stochastic_npg_value_baseline_special_deterministic_reward_intermediate_8}
    \frac{ \left[ \exp\Big\{ \eta \cdot \frac{r(a^-) - \pi_{\theta_t}^\top r}{\pi_{\theta_t}(a^-)} \Big\} - 1 \right] \cdot \left( r(a^-) - \pi_{\theta_t}^\top r \right) }{ \exp\Big\{ \eta \cdot \frac{r(a^-) - \pi_{\theta_t}^\top r}{\pi_{\theta_t}(a^-)} \Big\} + \frac{ 1 - \pi_{\theta_t}(a^-) }{ \pi_{\theta_t}(a^-) } } = \frac{ \left[ \exp\Big\{ \eta \cdot \frac{ \pi_{\theta_t}^\top r - r(a^-) }{\pi_{\theta_t}(a^-)} \Big\} - 1 \right] \cdot \left( \pi_{\theta_t}^\top r - r(a^-)  \right) }{ \left[ \exp\Big\{ \eta \cdot \frac{ \pi_{\theta_t}^\top r - r(a^-) }{\pi_{\theta_t}(a^-)} \Big\} - 1 \right] \cdot \frac{ 1 - \pi_{\theta_t}(a^-) }{ \pi_{\theta_t}(a^-) } + \frac{1}{\pi_{\theta_t}(a^-)} } \\
    &\ge \frac{ \eta \cdot \frac{ \pi_{\theta_t}^\top r - r(a^-) }{\pi_{\theta_t}(a^-)} \cdot \left( \pi_{\theta_t}^\top r - r(a^-)  \right) }{ \eta \cdot \frac{ \pi_{\theta_t}^\top r - r(a^-) }{\pi_{\theta_t}(a^-)} \cdot \frac{ 1 - \pi_{\theta_t}(a^-) }{ \pi_{\theta_t}(a^-) } + \frac{1}{\pi_{\theta_t}(a^-)} } \qquad \left( e^x - 1 \ge x > 0 \right) \\
    &= \frac{\eta \cdot \pi_{\theta_t}(a^-) \cdot \left( \pi_{\theta_t}^\top r - r(a^-)  \right)^2 }{ \eta \cdot \left( \pi_{\theta_t}^\top r - r(a^-)  \right)\cdot \big( 1 - \pi_{\theta_t}(a^-) \big) + \pi_{\theta_t}(a^-)} \\
    &\ge \frac{\eta}{1 + \eta} \cdot \pi_{\theta_t}(a^-) \cdot \left( \pi_{\theta_t}^\top r - r(a^-)  \right)^2 \qquad \left( r \in [0,1]^K, \ \pi_{\theta_t}(a^-) \in (0, 1) \right)
\end{align}
Combining \cref{eq:non_uniform_lojasiewicz_stochastic_npg_value_baseline_special_deterministic_reward_intermediate_6,eq:non_uniform_lojasiewicz_stochastic_npg_value_baseline_special_deterministic_reward_intermediate_7,eq:non_uniform_lojasiewicz_stochastic_npg_value_baseline_special_deterministic_reward_intermediate_8}, we have,
\begin{align}
\label{eq:non_uniform_lojasiewicz_stochastic_npg_value_baseline_special_deterministic_reward_intermediate_9}
    \EEt{\pi_{\theta_{t+1}}^\top r} - \pi_{\theta_t}^\top r &\ge \sum_{a^+ \in \gA_t^+} \pi_{\theta_t}(a^+) \cdot \frac{\eta}{1 + \eta} \cdot \left( r(a^+) - \pi_{\theta_t}^\top r  \right)^2 \\
    &\qquad + \sum_{a^- \in \gA_t^-} \pi_{\theta_t}(a^-) \cdot \frac{\eta}{1 + \eta} \cdot \pi_{\theta_t}(a^-) \cdot \left( \pi_{\theta_t}^\top r - r(a^-)  \right)^2 \\
    &\ge \frac{ \eta }{ 1 + \eta } \cdot \pi_{\theta_t}(a^*) \cdot \left( r(a^*) - \pi_{\theta_t}^\top r \right)^2. \qquad \left( a^* \in \gA_t^+ \right)
\end{align}

\textbf{Second part. (2)} If $\hat{r}_t$ is from \cref{def:on_policy_importance_sampling}.

As noted after 
\cref{update_rule:softmax_natural_pg_special_on_policy_stochastic_gradient_value_baseline}, we analyze \cref{update_rule:equivalent_update_softmax_natural_pg_special_on_policy_stochastic_gradient_value_baseline}, which is duplicated as follows,
\begin{align}
\label{eq:non_uniform_lojasiewicz_stochastic_npg_value_baseline_special_stochastic_reward_intermediate_0}
\theta_{t+1}(a) \gets \theta_t(a) + \eta \cdot \frac{ \sI\left\{ a_t = a \right\} }{ \pi_{\theta_t}(a) } \cdot \left( x_t(a) - \pi_{\theta_t}^\top r \right)\,.
\end{align}

For all $t \ge 1$, given current policy $\pi_{\theta_t}$, the expected reward of next policy $\pi_{\theta_{t+1}}^\top r$ is a random variable, and the randomness is from on-policy sampling $a_t \sim \pi_{\theta_t}(\cdot)$ and reward sampling $x \sim R_{a_t}$. The expected progress after one update is,
\begin{align}
\label{eq:non_uniform_lojasiewicz_stochastic_npg_value_baseline_special_stochastic_reward_intermediate_1}
\MoveEqLeft
	\EEt{\pi_{\theta_{t+1}}^\top r} - \pi_{\theta_t}^\top r = \sum_{i = 1}^{K} \pi_{\theta_t}(i) \cdot \EEt{\pi_{\theta_{t+1} }^\top r \ | \ a_t = i } - \pi_{\theta_t}^\top r \qquad \left( a_t \sim \pi_{\theta_t}(\cdot) \right) \\
	&= \sum_{i = 1}^{K} \pi_{\theta_t}(i)  \cdot \underbrace{ \left(  \EEt{\pi_{\theta_{t+1} }^\top r \ | \ a_t = i } - \pi_{\theta_t}^\top r  \right) }_{\text{expected progress of } a_t = i} \\
	&= \sum_{i = 1}^{K} \pi_{\theta_t}(i) \cdot \left(  \int_{-R_{\max}}^{R_{\max}}{ \big[ \pi_{\theta_{t+1}}^\top r \ | \ a_t = i, \ R_t = x \big] \cdot P_i(x) \mu(d x)} - \pi_{\theta_t}^\top r  \right) \\
	&= \sum_{i = 1}^{K} \pi_{\theta_t}(i) \cdot \int_{-R_{\max}}^{R_{\max}}{ \underbrace{ \left( \big[ \pi_{\theta_{t+1}}^\top r \ | \ a_t = i, \ R_t = x \big] - \pi_{\theta_t}^\top r \right)}_{\text{progress of } a_t = i, \ R_t = x} \cdot P_i(x) \mu(d x)},
\end{align}
where $\big[ \pi_{\theta_{t+1}}^\top r \ | \ a_t = i, \ \ R_t = x \big]$ means the value of $\pi_{\theta_{t+1}}^\top r$ given the sampled action $a_i = i$ and sampled reward $R_t = x$. According to \cref{eq:non_uniform_lojasiewicz_stochastic_npg_value_baseline_special_stochastic_reward_intermediate_0,def:on_policy_importance_sampling}, we have,
\begin{align}
\label{eq:non_uniform_lojasiewicz_stochastic_npg_value_baseline_special_stochastic_reward_intermediate_2}
	\big[ \pi_{\theta_{t+1}}^\top r \ | \ a_t = i, \ R_t = x \big] &= \frac{\exp\Big\{ \theta_t(i) + \eta \cdot \frac{ x - \pi_{\theta_t}^\top r }{\pi_{\theta_t}(i)} \Big\} \cdot r(i) + \sum_{j \not= i}{ \exp\{ \theta_t(j) \} \cdot r(j) } }{ \exp\Big\{ \theta_t(i) + \eta \cdot \frac{ x - \pi_{\theta_t}^\top r }{\pi_{\theta_t}(i)} \Big\} + \sum_{j \not= i}{ \exp\{ \theta_t(j) \}} } \\
	&= \frac{ \pi_{\theta_t}(i) \cdot \exp\Big\{ \eta \cdot \frac{x - \pi_{\theta_t}^\top r }{\pi_{\theta_t}(i)} \Big\} \cdot r(i) + \sum_{j \not= i}{ \pi_{\theta_t}(j) \cdot r(j) } }{ \pi_{\theta_t}(i) \cdot \exp\Big\{ \eta \cdot \frac{x - \pi_{\theta_t}^\top r }{\pi_{\theta_t}(i)} \Big\} + \sum_{j \not= i}{ \pi_{\theta_t}(j) } },
\end{align}
where the last equation is by dividing $\sum_{a \in [K]}{ \exp\big\{ \theta_t(a) \big\} }$ from both the numerator and the denominator. Therefore, by algebra we have,
\begin{align}
\label{eq:non_uniform_lojasiewicz_stochastic_npg_value_baseline_special_stochastic_reward_intermediate_3}
	\big[ \pi_{\theta_{t+1}}^\top r \ | \ a_t = i, \ R_t = x \big] - \pi_{\theta_t}^\top r &= \frac{ \left[ \pi_{\theta_t}(i) \cdot \exp\Big\{ \eta \cdot \frac{ x - \pi_{\theta_t}^\top r}{\pi_{\theta_t}(i)} \Big\} - \pi_{\theta_t}(i) \right] \cdot \left( r(i) - \pi_{\theta_t}^\top r \right) }{ \pi_{\theta_t}(i) \cdot \exp\Big\{ \eta \cdot \frac{ x - \pi_{\theta_t}^\top r}{\pi_{\theta_t}(i)} \Big\} + \sum_{j \not= i}{ \pi_{\theta_t}(j) } } \\
	&= \frac{ \left[ \exp\Big\{ \eta \cdot \frac{x - \pi_{\theta_t}^\top r}{\pi_{\theta_t}(i)} \Big\} - 1 \right] \cdot \left( r(i) - \pi_{\theta_t}^\top r \right) }{ \exp\Big\{ \eta \cdot \frac{x - \pi_{\theta_t}^\top r}{\pi_{\theta_t}(i)} \Big\} + \frac{ 1 - \pi_{\theta_t}(i) }{ \pi_{\theta_t}(i) } }.
\end{align}
Combining \cref{eq:non_uniform_lojasiewicz_stochastic_npg_value_baseline_special_stochastic_reward_intermediate_1,eq:non_uniform_lojasiewicz_stochastic_npg_value_baseline_special_stochastic_reward_intermediate_3}, we have,
\begin{align}
\label{eq:non_uniform_lojasiewicz_stochastic_npg_value_baseline_special_stochastic_reward_intermediate_4}
\MoveEqLeft
	\EEt{\pi_{\theta_{t+1}}^\top r} - \pi_{\theta_t}^\top r = \sum_{i = 1}^{K} \pi_{\theta_t}(i) \cdot \int_{-R_{\max}}^{R_{\max}}{ \frac{ \left[ \exp\Big\{ \eta \cdot \frac{x - \pi_{\theta_t}^\top r}{\pi_{\theta_t}(i)} \Big\} - 1 \right] \cdot \left( r(i) - \pi_{\theta_t}^\top r \right) }{ \exp\Big\{ \eta \cdot \frac{x - \pi_{\theta_t}^\top r}{\pi_{\theta_t}(i)} \Big\} + \frac{ 1 - \pi_{\theta_t}(i) }{ \pi_{\theta_t}(i) } } \cdot P_i(x) \mu(d x)} \\
	&= \sum_{i = 1}^{K} \pi_{\theta_t}(i) \cdot \left( r(i) - \pi_{\theta_t}^\top r \right) \cdot \Bigg[ \int_{ x \in \gX_t^+ } \frac{ \exp\Big\{ \eta \cdot \frac{x - \pi_{\theta_t}^\top r}{\pi_{\theta_t}(i)} \Big\} - 1 }{ \exp\Big\{ \eta \cdot \frac{x - \pi_{\theta_t}^\top r}{\pi_{\theta_t}(i)} \Big\} + \frac{ 1 - \pi_{\theta_t}(i) }{ \pi_{\theta_t}(i) } } \cdot P_i(x) \mu(d x) \\
	&\qquad + \int_{ x \in \gX_t^- } \frac{ \exp\Big\{ \eta \cdot \frac{x - \pi_{\theta_t}^\top r}{\pi_{\theta_t}(i)} \Big\} - 1 }{ \exp\Big\{ \eta \cdot \frac{x - \pi_{\theta_t}^\top r}{\pi_{\theta_t}(i)} \Big\} + \frac{ 1 - \pi_{\theta_t}(i) }{ \pi_{\theta_t}(i) } } \cdot P_i(x) \mu(d x) \Bigg],
\end{align}
where $\gX_t^+$ and $\gX_t^-$ are defined by partitioning the sampled reward range $[ - R_{\max}, R_{\max} ]$ into two parts for the current iteration,
\begin{align}
\label{eq:non_uniform_lojasiewicz_stochastic_npg_value_baseline_special_stochastic_reward_intermediate_5}
    \gX_t^+ &\coloneqq \left\{ x \in [- R_{\max}, R_{\max}]: x - \pi_{\theta_t}^\top r \ge 0 \right\} =  [ \pi_{\theta_t}^\top r, \ R_{\max} ] , \\
    \gX_t^- &\coloneqq \left\{ x \in [- R_{\max}, R_{\max}]: x - \pi_{\theta_t}^\top r < 0 \right\} = [ - R_{\max}, \ \pi_{\theta_t}^\top r).
\end{align}
We next prove that, in \cref{eq:non_uniform_lojasiewicz_stochastic_npg_value_baseline_special_stochastic_reward_intermediate_4}, for any sampled action $a_t = i \in [K]$, we have,
\begin{align}
    \int_{-R_{\max}}^{R_{\max}}{ \frac{ \left[ \exp\Big\{ \eta \cdot \frac{x - \pi_{\theta_t}^\top r}{\pi_{\theta_t}(i)} \Big\} - 1 \right] \cdot \left( r(i) - \pi_{\theta_t}^\top r \right) }{ \exp\Big\{ \eta \cdot \frac{x - \pi_{\theta_t}^\top r}{\pi_{\theta_t}(i)} \Big\} + \frac{ 1 - \pi_{\theta_t}(i) }{ \pi_{\theta_t}(i) } } \cdot P_i(x) \mu(d x)} \ge \frac{\eta}{2} \cdot \left( r(i) - \pi_{\theta_t}^\top r \right)^2.
\end{align}
There are three cases of sampled action $a_t = i \in [K]$.

\textbf{Case (a).} $i \in [K]$ is a ``good'' action at the current iteration, i.e., $r(i) - \pi_{\theta_t}^\top r > 0$.

According to \cref{eq:piecewise_linear_domination_result_1} in \cref{lem:piecewise_linear_domination}, given any fixed $p \in (0, 1]$, and any fixed $\epsilon \in [0, 1]$, we have,
\begin{align}
\label{eq:non_uniform_lojasiewicz_stochastic_npg_value_baseline_special_stochastic_reward_intermediate_6}
	f_p(y) \coloneqq \frac{e^y - 1}{e^y + \frac{1 - p}{ p } } \ge \left( 1 - \epsilon \right) \cdot p \cdot y, \text{ for all } y \in  [0, \epsilon].
\end{align}
Let $p = \pi_{\theta_t}(i) \in (0, 1]$ according to the softmax parameterization. Let
\begin{align}
\label{eq:non_uniform_lojasiewicz_stochastic_npg_value_baseline_special_stochastic_reward_intermediate_7}
    \epsilon = \frac{1}{2} \cdot \frac{ r(i) - \pi_{\theta_t}^\top r }{  \int_{-R_{\max}}^{R_{\max}}{ \left| x - \pi_{\theta_t}^\top r \right| \cdot P_i(x) \mu(d x)} } > 0,
\end{align}
where the inequality is because of $r(i) - \pi_{\theta_t}^\top r > 0$. Also note that,
\begin{align}
\label{eq:non_uniform_lojasiewicz_stochastic_npg_value_baseline_special_stochastic_reward_intermediate_8}
\MoveEqLeft
    \epsilon = \frac{1}{2} \cdot \frac{ \left | r(i) - \pi_{\theta_t}^\top r \right| }{  \int_{-R_{\max}}^{R_{\max}}{ \left| x - \pi_{\theta_t}^\top r \right| \cdot P_i(x) \mu(d x)} } \qquad \left( r(i) - \pi_{\theta_t}^\top r > 0 \right) \\
    &= \frac{1}{2} \cdot \frac{ \left | \int_{-R_{\max}}^{R_{\max}}{ x \cdot P_i(x) \mu(d x)} - \pi_{\theta_t}^\top r \right| }{  \int_{-R_{\max}}^{R_{\max}}{ \left| x - \pi_{\theta_t}^\top r \right| \cdot P_i(x) \mu(d x)} } \qquad \left( \text{by \cref{assump:bounded_reward}} \right) \\
    &= \frac{1}{2} \cdot \frac{ \left | \int_{-R_{\max}}^{R_{\max}}{ \left( x - \pi_{\theta_t}^\top r \right) \cdot P_i(x) \mu(d x)} \right| }{  \int_{-R_{\max}}^{R_{\max}}{ \left| x - \pi_{\theta_t}^\top r \right| \cdot P_i(x) \mu(d x)} } \\
    &\le \frac{1}{2} \cdot \frac{ \int_{-R_{\max}}^{R_{\max}}{ \left| x - \pi_{\theta_t}^\top r \right| \cdot P_i(x) \mu(d x)} }{  \int_{-R_{\max}}^{R_{\max}}{ \left| x - \pi_{\theta_t}^\top r \right| \cdot P_i(x) \mu(d x)} } \qquad \left( \text{by triangle inequality} \right) \\
    &= 1 / 2 \le 1,  
\end{align}
which means $\epsilon \in (0, 1]$. Let
\begin{align}
\label{eq:non_uniform_lojasiewicz_stochastic_npg_value_baseline_special_stochastic_reward_intermediate_9}
    y = \eta \cdot \frac{x - \pi_{\theta_t}^\top r}{\pi_{\theta_t}(i)}.
\end{align}
We have,
\begin{align}
\MoveEqLeft
    \left| y \right| = \frac{\pi_{\theta_t}(i) \cdot \left| r(i) - \pi_{\theta_t}^\top r \right|}{8 \cdot R_{\max}^2} \cdot \frac{ \left| x - \pi_{\theta_t}^\top r \right| }{\pi_{\theta_t}(i)} \qquad \left( \text{by \cref{eq:non_uniform_lojasiewicz_stochastic_npg_value_baseline_special_stochastic_reward_result_1}} \right) \\
    &\le \frac{ \left| r(i) - \pi_{\theta_t}^\top r \right| }{ 4 \cdot R_{\max} } \qquad \left( \left| x - \pi_{\theta_t}^\top r \right| \le 2 \cdot R_{\max} \right) \\
    &\le \frac{1}{2} \cdot \frac{ \left| r(i) - \pi_{\theta_t}^\top r \right| }{  \int_{-R_{\max}}^{R_{\max}}{ \left| x - \pi_{\theta_t}^\top r \right| \cdot P_i(x) \mu(d x)} } \qquad \left( \int_{-R_{\max}}^{R_{\max}}{ \left| x - \pi_{\theta_t}^\top r \right| \cdot P_i(x) \mu(d x)} \le 2 \cdot R_{\max} \right) \\
    &= \epsilon.
\end{align}
Therefore, we have,
\begin{align}
\label{eq:non_uniform_lojasiewicz_stochastic_npg_value_baseline_special_stochastic_reward_intermediate_10}
\MoveEqLeft
    \int_{ x \in \gX_t^+ } \frac{ \exp\Big\{ \eta \cdot \frac{x - \pi_{\theta_t}^\top r}{\pi_{\theta_t}(i)} \Big\} - 1 }{ \exp\Big\{ \eta \cdot \frac{x - \pi_{\theta_t}^\top r}{\pi_{\theta_t}(i)} \Big\} + \frac{ 1 - \pi_{\theta_t}(i) }{ \pi_{\theta_t}(i) } } \cdot P_i(x) \mu(d x) \\
    &\ge \int_{ x \in \gX_t^+ } \left( 1 - \epsilon \right) \cdot \pi_{\theta_t}(i) \cdot \eta \cdot \frac{x - \pi_{\theta_t}^\top r}{\pi_{\theta_t}(i)} \cdot P_i(x) \mu(d x) \qquad \left( \text{by \cref{eq:non_uniform_lojasiewicz_stochastic_npg_value_baseline_special_stochastic_reward_intermediate_6}} \right) \\
    &= \eta \cdot \int_{ x \in \gX_t^+ } \left( 1 - \epsilon \right) \cdot \left( x - \pi_{\theta_t}^\top r \right) \cdot P_i(x) \mu(d x).
\end{align}
According to \cref{eq:piecewise_linear_domination_result_2} in \cref{lem:piecewise_linear_domination}, given any fixed $p \in (0, 1]$, and any fixed $\epsilon \in [0, 1]$, we have,
\begin{align}
\label{eq:non_uniform_lojasiewicz_stochastic_npg_value_baseline_special_stochastic_reward_intermediate_11}
	\frac{e^y - 1}{e^y + \frac{1 - p}{ p } } \ge \left( 1 + \epsilon \right) \cdot p \cdot y, \text{ for all } y \in  [- \epsilon, 0].
\end{align}
Using the same values of $p = \pi_{\theta_t}(i)$, $\epsilon$ in \cref{eq:non_uniform_lojasiewicz_stochastic_npg_value_baseline_special_stochastic_reward_intermediate_7}, and $y$ in \cref{eq:non_uniform_lojasiewicz_stochastic_npg_value_baseline_special_stochastic_reward_intermediate_9}, we have,
\begin{align}
\label{eq:non_uniform_lojasiewicz_stochastic_npg_value_baseline_special_stochastic_reward_intermediate_12}
\MoveEqLeft
    \int_{ x \in \gX_t^- } \frac{ \exp\Big\{ \eta \cdot \frac{x - \pi_{\theta_t}^\top r}{\pi_{\theta_t}(i)} \Big\} - 1 }{ \exp\Big\{ \eta \cdot \frac{x - \pi_{\theta_t}^\top r}{\pi_{\theta_t}(i)} \Big\} + \frac{ 1 - \pi_{\theta_t}(i) }{ \pi_{\theta_t}(i) } } \cdot P_i(x) \mu(d x) \\
    &\ge \int_{ x \in \gX_t^- } \left( 1 + \epsilon \right) \cdot \pi_{\theta_t}(i) \cdot \eta \cdot \frac{x - \pi_{\theta_t}^\top r}{\pi_{\theta_t}(i)} \cdot P_i(x) \mu(d x) \qquad \left( \text{by \cref{eq:non_uniform_lojasiewicz_stochastic_npg_value_baseline_special_stochastic_reward_intermediate_11}} \right) \\
    &= \eta \cdot \int_{ x \in \gX_t^- } \left( 1 + \epsilon \right) \cdot \left( x - \pi_{\theta_t}^\top r \right) \cdot P_i(x) \mu(d x).
\end{align}
Combining \cref{eq:non_uniform_lojasiewicz_stochastic_npg_value_baseline_special_stochastic_reward_intermediate_4,eq:non_uniform_lojasiewicz_stochastic_npg_value_baseline_special_stochastic_reward_intermediate_10,eq:non_uniform_lojasiewicz_stochastic_npg_value_baseline_special_stochastic_reward_intermediate_12}, we have,
\begin{align}
\label{eq:non_uniform_lojasiewicz_stochastic_npg_value_baseline_special_stochastic_reward_intermediate_13}
\MoveEqLeft
	\int_{-R_{\max}}^{R_{\max}}{ \frac{ \left[ \exp\Big\{ \eta \cdot \frac{x - \pi_{\theta_t}^\top r}{\pi_{\theta_t}(i)} \Big\} - 1 \right] \cdot \left( r(i) - \pi_{\theta_t}^\top r \right) }{ \exp\Big\{ \eta \cdot \frac{x - \pi_{\theta_t}^\top r}{\pi_{\theta_t}(i)} \Big\} + \frac{ 1 - \pi_{\theta_t}(i) }{ \pi_{\theta_t}(i) } } \cdot P_i(x) \mu(d x)} \\
	&\ge \left( r(i) - \pi_{\theta_t}^\top r \right) \cdot \eta \cdot \Bigg[ \int_{ x \in \gX_t^+ } \left( 1 - \epsilon \right) \cdot \left( x - \pi_{\theta_t}^\top r \right) \cdot P_i(x) \mu(d x) \\
	&\qquad + \int_{ x \in \gX_t^- } \left( 1 + \epsilon \right) \cdot \left( x - \pi_{\theta_t}^\top r \right) \cdot P_i(x) \mu(d x) \Bigg] \qquad \left( \text{since } r(i) - \pi_{\theta_t}^\top r > 0 \right) \\
	&= \left( r(i) - \pi_{\theta_t}^\top r \right) \cdot \eta \cdot \Bigg[ \int_{-R_{\max}}^{R_{\max}} \left( x - \pi_{\theta_t}^\top r \right) \cdot P_i(x) \mu(d x) \qquad \left( \text{by \cref{eq:non_uniform_lojasiewicz_stochastic_npg_value_baseline_special_stochastic_reward_intermediate_5}} \right) \\
	&\qquad - \epsilon \cdot \left( \int_{ x \in \gX_t^+ } \left( x - \pi_{\theta_t}^\top r \right) \cdot P_i(x) \mu(d x) - \int_{ x \in \gX_t^- } \left( x - \pi_{\theta_t}^\top r \right) \cdot P_i(x) \mu(d x) \right) \Bigg] \\
	&= \left( r(i) - \pi_{\theta_t}^\top r \right) \cdot \eta \cdot \Bigg[ \left( r(i) - \pi_{\theta_t}^\top r \right) \qquad \left( \text{by \cref{assump:bounded_reward}} \right) \\
	&\qquad - \epsilon \cdot \int_{-R_{\max}}^{R_{\max}}{ \left| x - \pi_{\theta_t}^\top r \right| \cdot P_i(x) \mu(d x)} \Bigg] \qquad \left( \text{by \cref{eq:non_uniform_lojasiewicz_stochastic_npg_value_baseline_special_stochastic_reward_intermediate_5}} \right) \\
	&= \left( r(i) - \pi_{\theta_t}^\top r \right) \cdot \eta \cdot \left[ \left( r(i) - \pi_{\theta_t}^\top r \right) - \frac{1}{2} \cdot \left( r(i) - \pi_{\theta_t}^\top r \right) \right] \qquad \left( \text{by \cref{eq:non_uniform_lojasiewicz_stochastic_npg_value_baseline_special_stochastic_reward_intermediate_7}} \right) \\
	&= \frac{\eta}{2} \cdot \left( r(i) - \pi_{\theta_t}^\top r \right)^2.
\end{align}

\textbf{Case (b).} $i \in [K]$ is a ``bad'' action at the current iteration, i.e., $r(i) - \pi_{\theta_t}^\top r < 0$.

According to \cref{eq:piecewise_linear_domination_result_1} in \cref{lem:piecewise_linear_domination}, given any fixed $p \in (0, 1]$, and any fixed $\epsilon \in [0, 1]$, we have,
\begin{align}
\label{eq:non_uniform_lojasiewicz_stochastic_npg_value_baseline_special_stochastic_reward_intermediate_14}
	\frac{e^y - 1}{e^y + \frac{1 - p}{ p } } \le \left( 1 + \epsilon \right) \cdot p \cdot y, \text{ for all } y \in  [0, \epsilon].
\end{align}
Let $p = \pi_{\theta_t}(i) \in (0, 1]$ according to the softmax parameterization. Let
\begin{align}
\label{eq:non_uniform_lojasiewicz_stochastic_npg_value_baseline_special_stochastic_reward_intermediate_15}
    \epsilon = \frac{1}{2} \cdot \frac{ - \left( r(i) - \pi_{\theta_t}^\top r \right) }{  \sum_{m=1}^{M} P_i(m) \cdot \left| R_i(m) - \pi_{\theta_t}^\top r \right| } > 0.
\end{align}
We have $\epsilon \le 1$ according to \cref{eq:non_uniform_lojasiewicz_stochastic_npg_value_baseline_special_stochastic_reward_intermediate_8}. Using the same value of $y$ in \cref{eq:non_uniform_lojasiewicz_stochastic_npg_value_baseline_special_stochastic_reward_intermediate_9}, we have,
\begin{align}
\label{eq:non_uniform_lojasiewicz_stochastic_npg_value_baseline_special_stochastic_reward_intermediate_16}
\MoveEqLeft
    \int_{ x \in \gX_t^+ } \frac{ \exp\Big\{ \eta \cdot \frac{x - \pi_{\theta_t}^\top r}{\pi_{\theta_t}(i)} \Big\} - 1 }{ \exp\Big\{ \eta \cdot \frac{x - \pi_{\theta_t}^\top r}{\pi_{\theta_t}(i)} \Big\} + \frac{ 1 - \pi_{\theta_t}(i) }{ \pi_{\theta_t}(i) } } \cdot P_i(x) \mu(d x) \\
    &\le \int_{ x \in \gX_t^+ } \left( 1 + \epsilon \right) \cdot \pi_{\theta_t}(i) \cdot \eta \cdot \frac{x - \pi_{\theta_t}^\top r}{\pi_{\theta_t}(i)} \cdot P_i(x) \mu(d x) \qquad \left( \text{by \cref{eq:non_uniform_lojasiewicz_stochastic_npg_value_baseline_special_stochastic_reward_intermediate_14}} \right) \\
    &= \eta \cdot \int_{ x \in \gX_t^+ } \left( 1 + \epsilon \right) \cdot \left( x - \pi_{\theta_t}^\top r \right) \cdot P_i(x) \mu(d x).
\end{align}
According to \cref{eq:piecewise_linear_domination_result_2} in \cref{lem:piecewise_linear_domination}, given any fixed $p \in (0, 1]$, and any fixed $\epsilon \in [0, 1]$, we have,
\begin{align}
\label{eq:non_uniform_lojasiewicz_stochastic_npg_value_baseline_special_stochastic_reward_intermediate_17}
	\frac{e^y - 1}{e^y + \frac{1 - p}{ p } } \le \left( 1 - \epsilon \right) \cdot p \cdot y, \text{ for all } y \in  [- \epsilon, 0].
\end{align}
Using the same values of $p = \pi_{\theta_t}(i)$, $\epsilon$ in \cref{eq:non_uniform_lojasiewicz_stochastic_npg_value_baseline_special_stochastic_reward_intermediate_15}, and $y$ in \cref{eq:non_uniform_lojasiewicz_stochastic_npg_value_baseline_special_stochastic_reward_intermediate_9}, we have,
\begin{align}
\label{eq:non_uniform_lojasiewicz_stochastic_npg_value_baseline_special_stochastic_reward_intermediate_18}
\MoveEqLeft
    \int_{ x \in \gX_t^- } \frac{ \exp\Big\{ \eta \cdot \frac{x - \pi_{\theta_t}^\top r}{\pi_{\theta_t}(i)} \Big\} - 1 }{ \exp\Big\{ \eta \cdot \frac{x - \pi_{\theta_t}^\top r}{\pi_{\theta_t}(i)} \Big\} + \frac{ 1 - \pi_{\theta_t}(i) }{ \pi_{\theta_t}(i) } } \cdot P_i(x) \mu(d x) \\
    &\le \int_{ x \in \gX_t^- } \left( 1 - \epsilon \right) \cdot \pi_{\theta_t}(i) \cdot \eta \cdot \frac{x - \pi_{\theta_t}^\top r}{\pi_{\theta_t}(i)} \cdot P_i(x) \mu(d x) \qquad \left( \text{by \cref{eq:non_uniform_lojasiewicz_stochastic_npg_value_baseline_special_stochastic_reward_intermediate_17}} \right) \\
    &= \eta \cdot \int_{ x \in \gX_t^- } \left( 1 - \epsilon \right) \cdot \left( x - \pi_{\theta_t}^\top r \right) \cdot P_i(x) \mu(d x).
\end{align}
Combining \cref{eq:non_uniform_lojasiewicz_stochastic_npg_value_baseline_special_stochastic_reward_intermediate_4,eq:non_uniform_lojasiewicz_stochastic_npg_value_baseline_special_stochastic_reward_intermediate_16,eq:non_uniform_lojasiewicz_stochastic_npg_value_baseline_special_stochastic_reward_intermediate_18}, we have,
\begin{align}
\label{eq:non_uniform_lojasiewicz_stochastic_npg_value_baseline_special_stochastic_reward_intermediate_19}
\MoveEqLeft
	\int_{-R_{\max}}^{R_{\max}}{ \frac{ \left[ \exp\Big\{ \eta \cdot \frac{x - \pi_{\theta_t}^\top r}{\pi_{\theta_t}(i)} \Big\} - 1 \right] \cdot \left( r(i) - \pi_{\theta_t}^\top r \right) }{ \exp\Big\{ \eta \cdot \frac{x - \pi_{\theta_t}^\top r}{\pi_{\theta_t}(i)} \Big\} + \frac{ 1 - \pi_{\theta_t}(i) }{ \pi_{\theta_t}(i) } } \cdot P_i(x) \mu(d x)} \\
	&\ge \left( r(i) - \pi_{\theta_t}^\top r \right) \cdot \eta \cdot \Bigg[ \int_{ x \in \gX_t^+ } \left( 1 + \epsilon \right) \cdot \left( x - \pi_{\theta_t}^\top r \right) \cdot P_i(x) \mu(d x) \\
	&\qquad + \int_{ x \in \gX_t^- } \left( 1 - \epsilon \right) \cdot \left( x - \pi_{\theta_t}^\top r \right) \cdot P_i(x) \mu(d x) \Bigg] \qquad \left( \text{since } r(i) - \pi_{\theta_t}^\top r < 0 \right) \\
	&= \left( r(i) - \pi_{\theta_t}^\top r \right) \cdot \eta \cdot \Bigg[ \int_{-R_{\max}}^{R_{\max}} \left( x - \pi_{\theta_t}^\top r \right) \cdot P_i(x) \mu(d x) \qquad \left( \text{by \cref{eq:non_uniform_lojasiewicz_stochastic_npg_value_baseline_special_stochastic_reward_intermediate_5}} \right) \\
	&\qquad + \epsilon \cdot \left( \int_{ x \in \gX_t^+ } \left( x - \pi_{\theta_t}^\top r \right) \cdot P_i(x) \mu(d x) - \int_{ x \in \gX_t^- } \left( x - \pi_{\theta_t}^\top r \right) \cdot P_i(x) \mu(d x) \right) \Bigg] \\
	&= \left( r(i) - \pi_{\theta_t}^\top r \right) \cdot \eta \cdot \Bigg[ \left( r(i) - \pi_{\theta_t}^\top r \right) \qquad \left( \text{by \cref{assump:bounded_reward}} \right) \\
	&\qquad + \epsilon \cdot \int_{-R_{\max}}^{R_{\max}}{ \left| x - \pi_{\theta_t}^\top r \right| \cdot P_i(x) \mu(d x)} \Bigg] \qquad \left( \text{by \cref{eq:non_uniform_lojasiewicz_stochastic_npg_value_baseline_special_stochastic_reward_intermediate_5}} \right) \\
	&= \left( r(i) - \pi_{\theta_t}^\top r \right) \cdot \eta \cdot \left[ \left( r(i) - \pi_{\theta_t}^\top r \right) - \frac{1}{2} \cdot \left( r(i) - \pi_{\theta_t}^\top r \right) \right] \qquad \left( \text{by \cref{eq:non_uniform_lojasiewicz_stochastic_npg_value_baseline_special_stochastic_reward_intermediate_15}} \right) \\
	&= \frac{\eta}{2} \cdot \left( r(i) - \pi_{\theta_t}^\top r \right)^2.
\end{align}

\textbf{Case (c).} $i \in [K]$ is an ``indifferent'' action at the current iteration, i.e., $r(i) - \pi_{\theta_t}^\top r = 0$.

According to \cref{eq:non_uniform_lojasiewicz_stochastic_npg_value_baseline_special_stochastic_reward_intermediate_4}, we have,
\begin{align}
\label{eq:non_uniform_lojasiewicz_stochastic_npg_value_baseline_special_stochastic_reward_intermediate_20}
\MoveEqLeft
	\int_{-R_{\max}}^{R_{\max}}{ \frac{ \left[ \exp\Big\{ \eta \cdot \frac{x - \pi_{\theta_t}^\top r}{\pi_{\theta_t}(i)} \Big\} - 1 \right] \cdot \left( r(i) - \pi_{\theta_t}^\top r \right) }{ \exp\Big\{ \eta \cdot \frac{x - \pi_{\theta_t}^\top r}{\pi_{\theta_t}(i)} \Big\} + \frac{ 1 - \pi_{\theta_t}(i) }{ \pi_{\theta_t}(i) } } \cdot P_i(x) \mu(d x)} \\
	&= 0 \ge \frac{\eta}{2} \cdot \left( r(i) - \pi_{\theta_t}^\top r \right)^2. \qquad \left( \text{since } r(i) - \pi_{\theta_t}^\top r = 0 \right)
\end{align}
Combining the three cases, i.e., \cref{eq:non_uniform_lojasiewicz_stochastic_npg_value_baseline_special_stochastic_reward_intermediate_13,eq:non_uniform_lojasiewicz_stochastic_npg_value_baseline_special_stochastic_reward_intermediate_19,eq:non_uniform_lojasiewicz_stochastic_npg_value_baseline_special_stochastic_reward_intermediate_20}, we have, for all action $i \in [K]$,
\begin{align}
\label{eq:non_uniform_lojasiewicz_stochastic_npg_value_baseline_special_stochastic_reward_intermediate_21}
\MoveEqLeft
    \int_{-R_{\max}}^{R_{\max}}{ \frac{ \left[ \exp\Big\{ \eta \cdot \frac{x - \pi_{\theta_t}^\top r}{\pi_{\theta_t}(i)} \Big\} - 1 \right] \cdot \left( r(i) - \pi_{\theta_t}^\top r \right) }{ \exp\Big\{ \eta \cdot \frac{x - \pi_{\theta_t}^\top r}{\pi_{\theta_t}(i)} \Big\} + \frac{ 1 - \pi_{\theta_t}(i) }{ \pi_{\theta_t}(i) } } \cdot P_i(x) \mu(d x)} \ge \frac{\eta}{2} \cdot \left( r(i) - \pi_{\theta_t}^\top r \right)^2 \\
    &= \frac{1}{2} \cdot \frac{\pi_{\theta_t}(i) \cdot \left| r(i) - \pi_{\theta_t}^\top r \right|}{8 \cdot R_{\max}^2} \cdot \left( r(i) - \pi_{\theta_t}^\top r \right)^2. \qquad \left( \text{by \cref{eq:non_uniform_lojasiewicz_stochastic_npg_value_baseline_special_stochastic_reward_result_1}} \right)
\end{align}
Combining \cref{eq:non_uniform_lojasiewicz_stochastic_npg_value_baseline_special_stochastic_reward_intermediate_4,eq:non_uniform_lojasiewicz_stochastic_npg_value_baseline_special_stochastic_reward_intermediate_21}, we have,
\begin{align}
\MoveEqLeft
\label{eq:non_uniform_lojasiewicz_stochastic_npg_value_baseline_special_stochastic_reward_intermediate_22}
	\EEt{\pi_{\theta_{t+1}}^\top r} - \pi_{\theta_t}^\top r = \sum_{i = 1}^{K} \pi_{\theta_t}(i) \cdot \int_{-R_{\max}}^{R_{\max}}{ \frac{ \left[ \exp\Big\{ \eta \cdot \frac{x - \pi_{\theta_t}^\top r}{\pi_{\theta_t}(i)} \Big\} - 1 \right] \cdot \left( r(i) - \pi_{\theta_t}^\top r \right) }{ \exp\Big\{ \eta \cdot \frac{x - \pi_{\theta_t}^\top r}{\pi_{\theta_t}(i)} \Big\} + \frac{ 1 - \pi_{\theta_t}(i) }{ \pi_{\theta_t}(i) } } \cdot P_i(x) \mu(d x)} \\
\label{eq:non_uniform_lojasiewicz_stochastic_npg_value_baseline_special_stochastic_reward_intermediate_22_a}
	&\ge \frac{1}{16 \cdot R_{\max}^2} \cdot \sum_{i = 1}^{K} \pi_{\theta_t}(i)^2 \cdot \left| r(i) - \pi_{\theta_t}^\top r \right|^3 \\
	&\ge \frac{ 1 }{16 \cdot R_{\max}^2} \cdot \frac{\Delta}{K-1} \cdot \pi_{\theta_t}(a^*)^2 \cdot \left( r(a^*) - \pi_{\theta_t}^\top r \right)^2, \qquad \left( \text{by \cref{lem:stochastic_natural_lojasiewicz_continuous_special}} \right)
\end{align}
thus finishing the proofs.
\end{proof}

\textbf{\cref{cor:almost_sure_convergence_stochastic_npg_value_baseline_special}.}
The sequence $\{\pi_{\theta_t}^\top r\}_{t\ge 1}$ converges with probability one.
\begin{proof}
Setting $Y_t =  r(a^*) - \pi_{\theta_t}^\top r$ we have $Y_t\in [0,1]$. 
Define $\cF_t$ as the $\sigma$-algebra generated by
$a_1, x_1(a_1), a_2, x_2(a_2), \dots, a_{t-1}, x_{t-1}(a_{t-1})$.
Note that $Y_t$ is $\cF_t$-measurable since $\theta_t$ is a deterministic function of $a_1, x_1(a_1), \dots, a_{t-1}, x_{t-1}(a_{t-1})$.
By \cref{lem:non_uniform_lojasiewicz_stochastic_npg_value_baseline_special}, $\EE{ Y_{t+1}|\cF_t }\le Y_t$.
Hence, the conditions of Doob's supermartingale theorem (\cref{thm:smc}) are satisfied and the result follows.
\end{proof}

\textbf{\cref{lem:non_vanishing_nl_coefficient_stochastic_npg_value_baseline_special} }(Non-vanishing stochastic N\L{} coefficient / ``automatic exploration'')\textbf{.}
Using \cref{update_rule:softmax_natural_pg_special_on_policy_stochastic_gradient_value_baseline} with the same settings as in \cref{lem:non_uniform_lojasiewicz_stochastic_npg_value_baseline_special}, with arbitrary policy parameter initialization $\theta_1 \in \sR^K$, we have,
\begin{align}
    c \coloneqq \inf_{t \ge 1} \pi_{\theta_t}(a^*) > 0, \qquad \text{almost surely (a.s.).}
\end{align}
\begin{proof}
Since the claim is concerned with the policies underlying the parameter vectors and not the parameter vectors themselves, as noted after 
\cref{update_rule:softmax_natural_pg_special_on_policy_stochastic_gradient_value_baseline},
without loss of generality, in the rest of the proof we assume that the parameter vector is updated
according to \cref{update_rule:equivalent_update_softmax_natural_pg_special_on_policy_stochastic_gradient_value_baseline} as follows,
\begin{align}
\theta_{t+1}(a) \gets \theta_t(a) + \eta \cdot \frac{ \sI\left\{ a_t = a \right\} }{ \pi_{\theta_t}(a) } \cdot \left( x_t(a) - \pi_{\theta_t}^\top r \right)\,.
\label{eq:sparseu}
\end{align}

Given $i \in [K]$, define the following set $\gP(i)$ of ``generalized one-hot policy'',
\begin{align}
\label{eq:non_vanishing_nl_coefficient_stochastic_npg_value_baseline_special_intermediate_1_a}
    \gA(i) &\coloneqq \left\{ j \in [K]: r(j) = r(i) \right\}, \\
\label{eq:non_vanishing_nl_coefficient_stochastic_npg_value_baseline_special_intermediate_1_b}
    \gP(i) &\coloneqq \bigg\{ \pi \in \Delta(K): \sum_{j \in \gA(i)}{ \pi(j) } = 1 \bigg\}.
\end{align}
We make the following two claims.
\begin{claim}
\label{cl:approaching_generalized_one_hot_policy}
Almost surely, $\pi_{\theta_t}$ approaches one ``generalized one-hot policy'', i.e., 
there exists (a possibly random) $i \in [K]$, such that $\sum_{ j \in \gA(i) }{\pi_{\theta_t}(j) } \to 1$ almost surely as $t \to \infty$.
\end{claim}

\begin{claim}
\label{cl:contradiction_approaching_sub_optimal_generalized_one_hot_policy}
Almost surely, $\pi_{\theta_t}$ cannot approach any ``sub-optimal generalized one-hot policies'', 
i.e., $i$ in the previous claim must be an optimal action. 
\end{claim}

From \cref{cl:contradiction_approaching_sub_optimal_generalized_one_hot_policy}, it follows that $\sum_{ j \in \gA(a^*) }{\pi_{\theta_t}(j) } \to 1$ almost surely, as $t \to \infty$ and thus the policy sequence obtained
almost surely convergences to a globally optimal policy $\pi^*$.

\textbf{Proof of \cref{cl:approaching_generalized_one_hot_policy}}. 

According to \cref{cor:almost_sure_convergence_stochastic_npg_value_baseline_special}, we have that for some (possibly random) $c \in [0, 1]$, almost surely,
\begin{align}
\label{eq:non_vanishing_nl_coefficient_stochastic_npg_value_baseline_special_claim_1_intermediate_1}
    \lim_{t \to \infty}{ \pi_{\theta_t}^\top r} = c\,.
\end{align}
Thanks to $\pi_{\theta_t}^\top r \in [0,1]$ and \cref{eq:subm}, 
$X_t = \pi_{\theta_t}^\top r$ ($t\ge 1$) satisfies the conditions of
\cref{cor:submnoiseconv}. Hence, by this result, 
almost surely, 
\begin{align}
\label{eq:non_vanishing_nl_coefficient_stochastic_npg_value_baseline_special_claim_1_intermediate_2}
    \lim_{t \to \infty}\,\,{  \EEt{\pi_{\theta_{t+1}}^\top r} - \pi_{\theta_{t+1}}^\top r } & = 0\,,
\end{align}
which,
combined with 
\cref{eq:non_vanishing_nl_coefficient_stochastic_npg_value_baseline_special_claim_1_intermediate_1}
 also gives that $\lim_{t\to\infty} \EEt{\pi_{\theta_{t+1}}^\top r} = c$ almost surely.
Hence, 
\begin{align}
\label{eq:non_vanishing_nl_coefficient_stochastic_npg_value_baseline_special_claim_1_intermediate_3}
    \lim_{t \to \infty} \,\, \EEt{\pi_{\theta_{t+1}}^\top r} - \pi_{\theta_{t}}^\top r  & = c-c = 0,\, 
    \qquad \text{a.s.}
\end{align}
According to \cref{eq:non_uniform_lojasiewicz_stochastic_npg_value_baseline_special_stochastic_reward_intermediate_22_a} in the proof of \cref{lem:non_uniform_lojasiewicz_stochastic_npg_value_baseline_special}, we have,
\begin{align}
\label{eq:non_vanishing_nl_coefficient_stochastic_npg_value_baseline_special_claim_1_intermediate_4}
\MoveEqLeft
	\EEt{\pi_{\theta_{t+1}}^\top r} - \pi_{\theta_t}^\top r \ge \frac{1}{16 \cdot R_{\max}^2} \cdot \sum_{i = 1}^{K} \pi_{\theta_t}(i)^2 \cdot \left| r(i) - \pi_{\theta_t}^\top r \right|^3 \qquad \text{a.s.}
\end{align}
Combining \cref{eq:non_vanishing_nl_coefficient_stochastic_npg_value_baseline_special_claim_1_intermediate_3,eq:non_vanishing_nl_coefficient_stochastic_npg_value_baseline_special_claim_1_intermediate_4}, we have, with probability $1$, 
\begin{align}
\label{eq:non_vanishing_nl_coefficient_stochastic_npg_value_baseline_special_claim_1_intermediate_5}
    \lim_{t \to \infty}{  \sum_{i = 1}^{K} \pi_{\theta_t}(i)^2 \cdot \left| r(i) - \pi_{\theta_t}^\top r \right|^3 } = 0,
\end{align}
which implies that, for all $i \in [K]$, almost surely,
\begin{align}
\label{eq:non_vanishing_nl_coefficient_stochastic_npg_value_baseline_special_claim_1_intermediate_6}
    \lim_{t \to \infty}{ \pi_{\theta_t}(i)^2 \cdot \left| r(i) - \pi_{\theta_t}^\top r \right|^3 } = 0\,.
\end{align}

We claim that $c$, the almost sure limit of $\pi_{\theta_t}^\top r$, is such that almost surely, for some (possibly random)
 $i\in [K]$, $c= r(i)$ almost surely. 
We prove this by contradiction.
Let $\mathcal{E}_i = \{ c = r(i) \}$. 
Hence, our goal is to show that $\mathbb{P}( \cup_i \mathcal{E}_i ) = 1$.
Clearly, this follows from $\mathbb{P}( \cap_i \mathcal{E}_i^c ) = 0$, hence, we prove this.
On $\mathcal{E}_i^c$, since $\lim_{t\to\infty} \pi_{\theta_t}^\top r \ne r(i)$, 
we also have
\begin{align}
\label{eq:non_vanishing_nl_coefficient_stochastic_npg_value_baseline_special_claim_1_intermediate_7}
    \lim_{t \to \infty}{ \left| r(i) - \pi_{\theta_t}^\top r \right|^3 } > 0, \quad \text{ almost surely on } \mathcal{E}_i^c\,.
\end{align}
This, together with \cref{eq:non_vanishing_nl_coefficient_stochastic_npg_value_baseline_special_claim_1_intermediate_6} gives that almost surely on $\mathcal{E}_i^c$,
\begin{align}
\label{eq:non_vanishing_nl_coefficient_stochastic_npg_value_baseline_special_claim_1_intermediate_8}
    \lim_{t \to \infty}{ \pi_{\theta_t}(i)^2 } = 0\,.
\end{align}
Hence, on $\cap_i \mathcal{E}_i^c$, almost surely, for all $i\in [K]$,
$\lim_{t \to \infty}{ \pi_{\theta_t}(i)^2 } = 0$. This contradicts
with that $\sum_i \pi_{\theta_t}(i)=1$ holds for all $t\ge 1$, and hence we must have that 
$\mathbb{P}(\cap_i \mathcal{E}_i^c)=0$, finishing the proof that 
$\mathbb{P}(\cup_i \mathcal{E}_i)=1$.

Now, let $i\in [K]$ be the (possibly random) index of the action for which $c=r(i)$ almost surely.
Recall that $\gA(i)$ contains all actions $j$ with $r(j)=r(i)$
(cf. \cref{eq:non_vanishing_nl_coefficient_stochastic_npg_value_baseline_special_intermediate_1_a}).
Clearly, it holds that 
%
for all $j \in \gA(i)$,
\begin{align}
\label{eq:non_vanishing_nl_coefficient_stochastic_npg_value_baseline_special_claim_1_intermediate_9}
    \lim_{t \to \infty}{ \pi_{\theta_t}^\top r } = r(j), \qquad \text{a.s.},
\end{align}
and we have, for all $k \not\in \gA(i)$,
\begin{align}
\label{eq:non_vanishing_nl_coefficient_stochastic_npg_value_baseline_special_claim_1_intermediate_10}
    \lim_{t \to \infty}{ \left| r(k) - \pi_{\theta_t}^\top r \right|^3 } > 0, \qquad \text{a.s.},
\end{align}
which implies that,
\begin{align}
\label{eq:non_vanishing_nl_coefficient_stochastic_npg_value_baseline_special_claim_1_intermediate_11}
    \lim_{t \to \infty}{ \sum_{k \not\in \gA(i)} \pi_{\theta_t}(k)^2 } = 0, \qquad \text{a.s.}
\end{align}
Therefore, we have,
\begin{align}
\label{eq:non_vanishing_nl_coefficient_stochastic_npg_value_baseline_special_claim_1_intermediate_12}
    \lim_{t \to \infty}{ \sum_{j \in \gA(i)} \pi_{\theta_t}(j) } = 1, \qquad \text{a.s.},
\end{align}
which means $\pi_{\theta_t}$ a.s. approaches the ``generalized one-hot policy'' $\gP(i)$ in \cref{eq:non_vanishing_nl_coefficient_stochastic_npg_value_baseline_special_intermediate_1_b} as $t \to \infty$, finishing the proof of the first claim.

\textbf{Proof of \cref{cl:contradiction_approaching_sub_optimal_generalized_one_hot_policy}}.
Recall that this claim stated that  
$\lim_{t \to \infty} \sum_{ j \in \gA(a^*) }{\pi_{\theta_t}(j) } = 1$.
The brief sketch of the proof is as follows:
By \cref{cl:approaching_generalized_one_hot_policy},  there exists a (possibly random) $i\in [K]$ such that 
$\sum_{ j \in \gA(i) }{\pi_{\theta_t}(j) } \to 1$ almost surely, as $t \to \infty$.
If $i=a^*$ almost surely, \cref{cl:contradiction_approaching_sub_optimal_generalized_one_hot_policy} follows. 
Hence, it suffices to consider the event that $\{ i\not = a^* \}$ and show that this event has zero probability mass. Hence,  in the rest of the proof we assume that we are on the event when $i\not =a^*$.

Since $i \not= a^*$, there exists at least one ``good'' action $a^+ \in [K]$ such that $r(a^+) > r(i)$. The two cases are as follows.
\begin{description}[style=unboxed,leftmargin=0cm]
    \item[2a)] All ``good'' actions are sampled finitely many times as $t \to \infty$. \label{cl:contradiction_approaching_sub_optimal_generalized_one_hot_policy:a}
    \item[2b)] At least one ``good'' action is sampled infinitely many times as $t \to \infty$.
\end{description}
In both cases, we show that $\sum_{ j \in \gA(i) }{ \exp\{ \theta_t(j) \} } < \infty$ as $t \to \infty$ (but for different reasons), \textcolor{red}{which is a contradiction with the assumption of $\sum_{ j \in \gA(i) }{\pi_{\theta_t}(j) } \to 1$ as $t \to \infty$}, given that a ``good'' action's parameter is almost surely lower bounded. Hence, $i\ne a^*$ almost surely does not happen, which means that almost surely $i=a^*$.

Let us now turn to the details of the proof. We start with 
 some useful extra notation.
For each action $a \in [K]$, for $t \ge 2$, we have the following decomposition,
\begin{align}
\label{eq:non_vanishing_nl_coefficient_stochastic_npg_value_baseline_special_claim_2_intermediate_1}
    \theta_{t}(a) = \underbrace{ \theta_{t}(a) - \chE_{t-1}{[ \theta_t(a)]} }_{ W_t(a) } + \underbrace { \chE_{t-1}{[ \theta_t(a)]} - \theta_{t-1}(a) }_{ P_{t-1}(a) } + \theta_{t-1}(a),
\end{align}
while we also have,
\begin{align}
\label{eq:non_vanishing_nl_coefficient_stochastic_npg_value_baseline_special_claim_2_intermediate_2}
    \theta_1(a) = \underbrace{ \theta_{1}(a) - \EE{\theta_{1}(a)} }_{ W_1(a) } + \EE{\theta_{1}(a)},
\end{align}
where $\EE{\theta_{1}(a)}$ accounts for possible randomness in initialization of $\theta_1$.

Define the following notations,
\begin{align}
\label{eq:non_vanishing_nl_coefficient_stochastic_npg_value_baseline_special_claim_2_intermediate_3a}
    Z_t(a) &\coloneqq W_1(a) + \cdots + W_t(a), \qquad \left( \text{``cumulative noise''} \right) \\
\label{eq:non_vanishing_nl_coefficient_stochastic_npg_value_baseline_special_claim_2_intermediate_3b}
    W_t(a) &\coloneqq \theta_t(a) - \chE_{t-1}{[ \theta_t(a)]}, \qquad \left( \text{``noise''} \right) \\
\label{eq:non_vanishing_nl_coefficient_stochastic_npg_value_baseline_special_claim_2_intermediate_3c}
    P_t(a) &\coloneqq \EEt{\theta_{t+1}(a)} - \theta_t(a). \qquad \left( \text{``progress''} \right)
\end{align}
Recursing \cref{eq:non_vanishing_nl_coefficient_stochastic_npg_value_baseline_special_claim_2_intermediate_1} gives,
\begin{align}
\label{eq:non_vanishing_nl_coefficient_stochastic_npg_value_baseline_special_claim_2_intermediate_4}
    \theta_t(a) = \EE{\theta_1(a)} + Z_t(a) + \underbrace{ P_1(a) + \cdots + P_{t-1}(a)}_{\text{``cumulative progress''}}.
\end{align}
We have that $\EEt{W_{t+1}(a)}=0$, for $t=0, 1, \dots$. Let
\begin{align}
\label{eq:non_vanishing_nl_coefficient_stochastic_npg_value_baseline_special_claim_2_intermediate_5}
    I_t(a) = \begin{cases}
		1, & \text{if } a_t = a\, , \\
		0, & \text{otherwise}\,.
	\end{cases}
\end{align}
The update rule (cf. \cref{eq:sparseu}) is,
\begin{align}
\label{eq:non_vanishing_nl_coefficient_stochastic_npg_value_baseline_special_claim_2_intermediate_6}
    \theta_{t+1}(a) = \theta_{t}(a) + \eta \cdot \frac{I_t(a)}{ \pi_{\theta_t}(a) } \cdot \left( x_t(a) - \pi_{\theta_t}^\top r \right),
\end{align}
where $a_t \sim \pi_{\theta_t}(\cdot)$, and $x_t(a) \sim P_a$. Let $\gF_t$ be the $\sigma$-algebra generated by $a_1$, $x_1(a_1)$, $\cdots$, $a_{t-1}$, $x_{t-1}(a_{t-1})$, $a_t$:
\begin{align}
\cF_t = \sigma( \{ 
 a_1, x_1(a_1), \cdots, a_{t-1}, x_{t-1}(a_{t-1}), a_t \} )\,.
\end{align}
Note that $\theta_{t},I_t$ are $\gF_t$-measurable and $\hat{x}_t$ is $\gF_{t+1}$-measurable for all $t\ge 1$. Let $\mathbb{E}_t$ denote the conditional expectation with respect to $\gF_t$: $\mathbb{E}_t[X] = \mathbb{E}[X|\gF_t]$.

Using the above notations, we have,
\begin{align}
\MoveEqLeft
\label{eq:non_vanishing_nl_coefficient_stochastic_npg_value_baseline_special_claim_2_intermediate_7}
    W_{t+1}(a) = \theta_{t+1}(a) - \EEt{\theta_{t+1}(a)} \\
    &= \bcancel{\theta_{t}(a)} + \eta \cdot \frac{I_t(a)}{ \pi_{\theta_t}(a) } \cdot \left( x_t(a) - \bcancel{\pi_{\theta_t}^\top r} \right) - \chE_{t}{ \left[ \bcancel{\theta_{t}(a)} + \eta \cdot \frac{I_t(a)}{ \pi_{\theta_t}(a) } \cdot \left( x_t(a) - \bcancel{\pi_{\theta_t}^\top r} \right) \right] } \\
    &= \eta \cdot \frac{I_t(a)}{ \pi_{\theta_t}(a) } \cdot \left( x_t(a) - r(a) \right),
\end{align}
which implies that,
\begin{align}
\label{eq:non_vanishing_nl_coefficient_stochastic_npg_value_baseline_special_claim_2_intermediate_8}
    Z_t(a) &= W_1(a) + \cdots + W_t(a) \\
    &= \sum_{s=1}^{t-1}{ \eta \cdot \frac{I_s(a)}{ \pi_{\theta_s}(a) } \cdot \left( x_s(a) - r(a) \right)}.
\end{align}
We also have,
\begin{align}
\label{eq:non_vanishing_nl_coefficient_stochastic_npg_value_baseline_special_claim_2_intermediate_9}
    P_t(a) &= \EEt{\theta_{t+1}(a)} - \theta_t(a) \\
    &= \chE_{t}{ \left[ \bcancel{\theta_{t}(a)} + \eta \cdot \frac{I_t(a)}{ \pi_{\theta_t}(a) } \cdot \left( x_t(a) - \pi_{\theta_t}^\top r \right) \right] } - \bcancel{\theta_t(a)} \\
    &= \eta \cdot \frac{I_t(a)}{ \pi_{\theta_t}(a) } \cdot \left( r(a) - \pi_{\theta_t}^\top r \right).
\end{align}
Using the learning rate of \cref{eq:non_uniform_lojasiewicz_stochastic_npg_value_baseline_special_stochastic_reward_result_1},
\begin{align}
\label{eq:non_vanishing_nl_coefficient_stochastic_npg_value_baseline_special_claim_2_intermediate_10}
    \eta = \frac{\pi_{\theta_t}(a_t) \cdot \left| r(a_t) - \pi_{\theta_t}^\top r \right| }{8 \cdot R_{\max}^2},
\end{align}
we have,
\begin{align}
\label{eq:non_vanishing_nl_coefficient_stochastic_npg_value_baseline_special_claim_2_intermediate_11}
    W_{t+1}(a) &= \frac{\pi_{\theta_t}(a_t) \cdot \left| r(a_t) - \pi_{\theta_t}^\top r \right| }{8 \cdot R_{\max}^2} \cdot \frac{I_t(a)}{ \pi_{\theta_t}(a) } \cdot \left( x_t(a) - r(a) \right) \qquad \left( \text{by \cref{eq:non_vanishing_nl_coefficient_stochastic_npg_value_baseline_special_claim_2_intermediate_7}} \right) \\
    &= \frac{ I_t(a)  }{8 \cdot R_{\max}^2} \cdot \left| r(a) - \pi_{\theta_t}^\top r \right| \cdot \left( x_t(a) - r(a) \right) \\
    &\in \bigg[ - \frac{1}{ 8 \cdot R_{\max} }, \frac{1}{ 8 \cdot R_{\max} } \bigg].
\end{align}
Similarly, we have,
\begin{align}
\label{eq:non_vanishing_nl_coefficient_stochastic_npg_value_baseline_special_claim_2_intermediate_12}
    P_t(a) = \frac{ I_t(a) }{8 \cdot R_{\max}^2} \cdot \left| r(a) - \pi_{\theta_t}^\top r \right| \cdot \left( r(a) - \pi_{\theta_t}^\top r \right),
\end{align}
and
\begin{align}
\label{eq:non_vanishing_nl_coefficient_stochastic_npg_value_baseline_special_claim_2_intermediate_13}
    Z_t(a) = \sum_{s=1}^{t-1}{ \frac{ I_s(a) }{8 \cdot R_{\max}^2} \cdot \left| r(a) - \pi_{\theta_s}^\top r \right| \cdot \left( x_s(a) - r(a) \right) }.
\end{align}
Define the following notations,
\begin{align}
\label{eq:non_vanishing_nl_coefficient_stochastic_npg_value_baseline_special_claim_2_intermediate_14a}
    N_t(a) &\coloneqq \sum_{s=1}^{t}{ I_s(a) }, \\
\label{eq:non_vanishing_nl_coefficient_stochastic_npg_value_baseline_special_claim_2_intermediate_14b}
    N_\infty(a) &\coloneqq \sum_{s=1}^{\infty}{ I_s(a) }, \\
\label{eq:non_vanishing_nl_coefficient_stochastic_npg_value_baseline_special_claim_2_intermediate_14c}
    N_{p:q}(a) &\coloneqq \sum_{s=p}^{q}{ I_s(a) }.
\end{align}

Recall that $i$ is the index of the (random) action $I\in [K]$  with 
\begin{align}
\label{eq:non_vanishing_nl_coefficient_stochastic_npg_value_baseline_special_claim_2_intermediate_15}
    \lim_{t \to \infty} \sum_{ j \in \gA(I) }{\pi_{\theta_t}(j) } = 1, \qquad \text{a.s.}
\end{align}
As noted earlier we consider the event $\{ I \ne a^* \}$, where $a^*$ is the index of an optimal action and we will show that this event has zero probability.
Since $\{ I \ne a^* \} = \cup_{i\in [K]} \{ I=i, i\ne a^* \}$, it suffices to show that for any fixed $i\in [K]$ index with $r(i)<r(a^*)$, $\{ I=i, i\ne a^* \}$ has zero probability.
Hence, in what follows we fix such a suboptimal action's index $i\in [K]$ and consider the event $\{I=i,i\ne a^*\}$.

Partition the action set $[K]$ into three parts using $r(i)$ as follows,
\begin{align}
\label{eq:non_vanishing_nl_coefficient_stochastic_npg_value_baseline_special_claim_2_intermediate_16a}
    \gA(i) &\coloneqq \left\{ j \in [K]: r(j) = r(i) \right\}, \qquad \left( \text{from \cref{eq:non_vanishing_nl_coefficient_stochastic_npg_value_baseline_special_intermediate_1_a}} \right) \\
\label{eq:non_vanishing_nl_coefficient_stochastic_npg_value_baseline_special_claim_2_intermediate_16b}
    \gA^+(i) &\coloneqq \left\{ a^+ \in [K]: r(a^+) > r(i) \right\}, \\
\label{eq:non_vanishing_nl_coefficient_stochastic_npg_value_baseline_special_claim_2_intermediate_16c}
    \gA^-(i) &\coloneqq \left\{ a^- \in [K]: r(a^-) < r(i) \right\}.
\end{align}
Because $i$ was the index of a sub-optimal action, we have $\gA^+(i)\ne \emptyset$.
According to \cref{eq:non_vanishing_nl_coefficient_stochastic_npg_value_baseline_special_claim_2_intermediate_15}, 
on $\{I=i\} \supset \{I=i,i\ne a^* \}$,
we have $\pi_{\theta_t}^\top r \to r(i)$ as $t \to \infty$ because 
\begin{align}
\label{eq:non_vanishing_nl_coefficient_stochastic_npg_value_baseline_special_claim_2_intermediate_17}
    \left| r(i) - \pi_{\theta_t}^\top r \right| &= \bigg| \sum_{k \not\in \gA(i)} \pi_{\theta_t}(k) \cdot \left( r(i) - r(k) \right) \bigg| \\
    &\le  \sum_{k \not\in \gA(i)} \pi_{\theta_t}(k) \cdot \left| r(i) - r(k) \right|  \\
    &\le 1 - \sum_{ j \in \gA(i) }{\pi_{\theta_t}(j) }. \qquad \left ( r \in [0, 1]^K \right)
\end{align}
Therefore, there exists $\tau\ge 1$ such that
almost surely on $\{I = i,i\ne a^* \}$ 
 $\tau < \infty$ while we also have
\begin{align}
\label{eq:non_vanishing_nl_coefficient_stochastic_npg_value_baseline_special_claim_2_intermediate_18}
    r(a^+) - c^\prime \ge \pi_{\theta_t}^\top r \ge r(a^-) + c^\prime, \qquad \text{for all } t \ge \tau,
\end{align}
for all $a^+ \in \gA^+(i)$, $a^- \in \gA^-(i)$, where $c^\prime > 0$.

Now, take any $a^- \in \gA^-(i)$. According to \cref{lem:bad_action_parameter_upper_bounded_almost_surely}, we have, almost surely on $\{I=i,i\ne a^*\}$,
\begin{align}
\label{eq:non_vanishing_nl_coefficient_stochastic_npg_value_baseline_special_claim_2_intermediate_19}
    c_1 \coloneqq \sup_{t \ge 1}{ \theta_t(a^-)} < \infty.
\end{align}

\textbf{First case. 2a).} Consider the event,
\begin{align}
\label{eq:non_vanishing_nl_coefficient_stochastic_npg_value_baseline_special_claim_2_case_1_intermediate_1}
    \gE_0 \coloneqq \bigcap\limits_{a^+ \in \gA^+(i)} 
    \underbrace{
    \left\{ N_\infty(a^+) < \infty \right\}}_{\mathcal{E}_0(a^+)},
\end{align}
i.e., any ``good'' action $a^+ \in \gA^+(i)$ has finitely many updates as $t \to \infty$. Pick $a^+ \in \gA^+(i)$, such that $\PP{\left( N_\infty(a^+) < \infty \right) } > 0$. According to the extended Borel-Cantelli lemma (\cref{lem:ebc}), we have, almost surely,
\begin{align}
\label{eq:non_vanishing_nl_coefficient_stochastic_npg_value_baseline_special_claim_2_case_1_intermediate_2}
    \Big\{ \sum_{t \ge 1} \pi_{\theta_t}(a^+)=\infty \Big\} = \left\{ N_\infty(a^+)=\infty \right\}.
\end{align}
Hence, taking complements, we have,
\begin{align}
\label{eq:non_vanishing_nl_coefficient_stochastic_npg_value_baseline_special_claim_2_case_1_intermediate_3}
    \Big\{ \sum_{t \ge 1} \pi_{\theta_t}(a^+)<\infty \Big\} = \left\{N_\infty(a^+)<\infty\right\}
\end{align}
also holds almost surely. 

On event $\gE_0(a^+)$, we also have,
\begin{align}
\label{eq:non_vanishing_nl_coefficient_stochastic_npg_value_baseline_special_claim_2_case_1_intermediate_6a}
    c_2 &\coloneqq \inf_{t \ge 1}{ \theta_t(a^+) } > - \infty, \\
\label{eq:non_vanishing_nl_coefficient_stochastic_npg_value_baseline_special_claim_2_case_1_intermediate_6b}
    c_3 &\coloneqq \sup_{t \ge 1}{ \theta_t(a^+) } < \infty,
\end{align}
which is because on this event the parameter corresponding to $a^+$ receives finitely many updates and each update is bounded, i.e., for any $a \in [K]$,
\begin{align}
\label{eq:non_vanishing_nl_coefficient_stochastic_npg_value_baseline_special_claim_2_case_1_intermediate_7}
    \big| \theta_{t+1}(a) - \theta_{t}(a) \big| &= \eta \cdot \frac{I_t(a)}{ \pi_{\theta_t}(a) } \cdot \left| x_t(a) - \pi_{\theta_t}^\top r \right| \qquad \left( \text{by \cref{eq:non_vanishing_nl_coefficient_stochastic_npg_value_baseline_special_claim_2_intermediate_6}} \right) \\
    &= \frac{\pi_{\theta_t}(a_t) \cdot \left| r(a_t) - \pi_{\theta_t}^\top r \right| }{8 \cdot R_{\max}^2} \cdot \frac{I_t(a)}{ \pi_{\theta_t}(a) } \cdot \left| x_t(a) - \pi_{\theta_t}^\top r \right| \qquad \left( \text{by \cref{eq:non_vanishing_nl_coefficient_stochastic_npg_value_baseline_special_claim_2_intermediate_10}} \right) \\
    &= \frac{I_t(a)  }{8 \cdot R_{\max}^2} \cdot \left| r(a) - \pi_{\theta_t}^\top r \right| \cdot \left| x_t(a) - \pi_{\theta_t}^\top r \right| \le \frac{1}{8 \cdot R_{\max}}.
\end{align}
Define
\begin{align}
\label{eq:non_vanishing_nl_coefficient_stochastic_npg_value_baseline_special_claim_2_case_1_intermediate_9}
    q_t = \sum_{a^+ \in \gA^+(i)}{ \pi_{\theta_t}(a^+) }.
\end{align}
On event $\gE^\prime \coloneqq \gE_0\cap \{ I=i,i\ne a^* \}$, 
and by the softmax parameterization, we have,
\begin{align}
\MoveEqLeft
\label{eq:non_vanishing_nl_coefficient_stochastic_npg_value_baseline_special_claim_2_case_1_intermediate_10}
    q_t = \frac{ \sum_{a^+ \in \gA^+(i)} e^{\theta_t(a^+)} }{ \sum_{j \in \gA(i)}{ e^{ \theta_t(j) } } + \sum_{a^+ \in \gA^+(i)}{ e^{ \theta_t(a^+) } } + \sum_{a^- \in \gA^-(i)}{ e^{ \theta_t(a^-) } } } \\
    &\ge \frac{ \sum_{a^+ \in \gA^+(i)} e^{c_2 } }{ \sum_{j \in \gA(i)}{ e^{ \theta_t(j) } } + \sum_{a^+ \in \gA^+(i)}{ e^{ c_2 } } + \sum_{a^- \in \gA^-(i)}{ e^{ \theta_t(a^-) } } } \qquad \left( \text{by \cref{eq:non_vanishing_nl_coefficient_stochastic_npg_value_baseline_special_claim_2_case_1_intermediate_6a}} \right) \\
    &\ge \frac{ \sum_{a^+ \in \gA^+(i)} e^{c_2 } }{ \sum_{j \in \gA(i)}{ e^{\theta_t(j) } } + \sum_{a^+ \in \gA^+(i)}{ e^{c_2 } } + \sum_{a^- \in \gA^-(i)}{ e^{c_1 } } } \qquad \left( \text{by \cref{eq:non_vanishing_nl_coefficient_stochastic_npg_value_baseline_special_claim_2_intermediate_19}} \right) \\
    &= \frac{ e^{c_2} \cdot \left|\gA^+(i) \right|  }{ \sum_{j \in \gA(i)}{ e^{ \theta_t(j) } } + e^{c_2} \cdot \left|\gA^+(i) \right| + e^{c_1} \cdot \left|\gA^-(i) \right| }.
\end{align}
Next, we have,
\begin{align}
\MoveEqLeft
\label{eq:non_vanishing_nl_coefficient_stochastic_npg_value_baseline_special_claim_2_case_1_intermediate_11}
    1 - \sum_{ j \in \gA(i) }{\pi_{\theta_t}(j) } = \frac{ \sum_{a^+ \in \gA^+(i)}{ e^{\theta_t(a^+)} } + \sum_{a^- \in \gA^-(i)}{ e^{\theta_t(a^-)} } }{ \sum_{j \in \gA(i)}{ e^{\theta_t(j)} } + \sum_{a^+ \in \gA^+(i)}{ e^{\theta_t(a^+)} } + \sum_{a^- \in \gA^-(i)}{ e^{\theta_t(a^-)} } } \\
    &\le \frac{ \sum_{a^+ \in \gA^+(i)}{ e^{c_3} } + \sum_{a^- \in \gA^-(i)}{ e^{c_1} } }{ \sum_{j \in \gA(i)}{ e^{ \theta_t(j) } } + \sum_{a^+ \in \gA^+(i)}{ e^{c_3} } + \sum_{a^- \in \gA^-(i)}{ e^{c_1} } } \qquad \left( \text{by \cref{eq:non_vanishing_nl_coefficient_stochastic_npg_value_baseline_special_claim_2_case_1_intermediate_6b,eq:non_vanishing_nl_coefficient_stochastic_npg_value_baseline_special_claim_2_intermediate_19}} \right) \\
    &= \frac{ e^{c_3} \cdot \left|\gA^+(i) \right| + e^{c_1} \cdot \left|\gA^-(i) \right| }{ \sum_{j \in \gA(i)}{ e^{ \theta_t(j) } } + e^{c_2} \cdot \left|\gA^+(i) \right| + e^{c_1} \cdot \left|\gA^-(i) \right| + \left( e^{c_3} - e^{c_2} \right) \cdot \left|\gA^+(i) \right| } \\
    &\le \frac{ e^{c_3} \cdot \left|\gA^+(i) \right| + e^{c_1} \cdot \left|\gA^-(i) \right| }{ \frac{e^{c_2} }{q_t} \cdot \left|\gA^+(i) \right| + \left( e^{c_3} - e^{c_2} \right) \cdot \left|\gA^+(i) \right| } \qquad \left( \text{by \cref{eq:non_vanishing_nl_coefficient_stochastic_npg_value_baseline_special_claim_2_case_1_intermediate_10}} \right) \\
    &= \frac{ e^{c_3} \cdot \left|\gA^+(i) \right| + e^{c_1} \cdot \left|\gA^-(i) \right| }{ e^{c_2} \cdot \left|\gA^+(i) \right| + \left( e^{c_3} - e^{c_2} \right) \cdot \left|\gA^+(i) \right| \cdot q_t } \cdot q_t \\
    &\le \frac{ e^{c_3} \cdot \left|\gA^+(i) \right| + e^{c_1} \cdot \left|\gA^-(i) \right| }{ e^{c_2} \cdot \left|\gA^+(i) \right| } \cdot q_t \,. \qquad \left( \text{because } q_t > 0 \right)
\label{eq:non_vanishing_nl_coefficient_stochastic_npg_value_baseline_special_claim_2_case_1_intermediate_111}
\end{align}
Denote $C^\prime \coloneqq  \frac{ e^{c_3} \cdot \left|\gA^+(i) \right| + e^{c_1} \cdot \left|\gA^-(i) \right| }{ e^{c_2} \cdot \left|\gA^+(i) \right| } $. We have,
\begin{align}
\label{eq:non_vanishing_nl_coefficient_stochastic_npg_value_baseline_special_claim_2_case_1_intermediate_12}
    \big| r(i) - \pi_{\theta_t}^\top r \big| &\le 1 - \sum_{ j \in \gA(i) }{\pi_{\theta_t}(j) } \qquad \left( r \in [0, 1]^K \right) \qquad \left( \text{by \cref{eq:non_vanishing_nl_coefficient_stochastic_npg_value_baseline_special_claim_2_intermediate_17}} \right) \\
    &\le C^\prime \cdot q_t. \qquad \left( \text{by \cref{eq:non_vanishing_nl_coefficient_stochastic_npg_value_baseline_special_claim_2_case_1_intermediate_111}} \right)
\end{align}
Take any $j \in \gA(i)$, according to \cref{eq:non_vanishing_nl_coefficient_stochastic_npg_value_baseline_special_claim_2_intermediate_4}, we have,
\begin{align}
\label{eq:non_vanishing_nl_coefficient_stochastic_npg_value_baseline_special_claim_2_case_1_intermediate_13}
    \theta_{t}(j) = \EE{\theta_1(j)} +Z_t(j) + \sum_{s=1}^{t-1}{P_s(j)}.
\end{align}
According to \cref{eq:non_vanishing_nl_coefficient_stochastic_npg_value_baseline_special_claim_2_intermediate_12}, we have,
\begin{align}
\label{eq:non_vanishing_nl_coefficient_stochastic_npg_value_baseline_special_claim_2_case_1_intermediate_14}
    P_s(j) = \frac{ I_s(j) }{8 \cdot R_{\max}^2} \cdot \left| r(j) - \pi_{\theta_s}^\top r \right| \cdot \left( r(j) - \pi_{\theta_s}^\top r \right).
\end{align}
Therefore, for all $s \ge 1$,
\begin{align}
\label{eq:non_vanishing_nl_coefficient_stochastic_npg_value_baseline_special_claim_2_case_1_intermediate_15}
    \left| P_s(j) \right| &\le \frac{1}{8 \cdot R_{\max}^2} \cdot \left( r(i) - \pi_{\theta_s}^\top r \right)^2 \qquad \left( j \in \gA(i), \ r(j) = r(i) \right) \\
    &\le \frac{C^\prime}{8 \cdot R_{\max}^2} \cdot q_s^2 \qquad \left( \text{by \cref{eq:non_vanishing_nl_coefficient_stochastic_npg_value_baseline_special_claim_2_case_1_intermediate_12}} \right) \\
    &\le \frac{C^\prime}{8 \cdot R_{\max}^2} \cdot q_s. \qquad \left( q_s \in (0,1) \right)
\label{eq:non_vanishing_nl_coefficient_stochastic_npg_value_baseline_special_claim_2_case_1_intermediate_15b}    
\end{align}
For any $j \in \gA(i)$, we have,
\begin{align}
\label{eq:non_vanishing_nl_coefficient_stochastic_npg_value_baseline_special_claim_2_case_1_intermediate_16_a}
    S_t^2(j) &\coloneqq \sum_{s=1}^{t} \left( r(j) - \pi_{\theta_s}^\top r \right)^2 \cdot I_s(j) \\
    &\le  \sum_{s=1}^{t} \left( r(j) - \pi_{\theta_s}^\top r \right)^2 \\
    &\le \sum_{s=1}^{t} q_s^2 \qquad \left( \text{by \cref{eq:non_vanishing_nl_coefficient_stochastic_npg_value_baseline_special_claim_2_case_1_intermediate_12}} \right) \\
    &\le \sum_{s=1}^{t} q_s \qquad \left( q_s \in [0, 1] \right) \\
\label{eq:non_vanishing_nl_coefficient_stochastic_npg_value_baseline_special_claim_2_case_1_intermediate_16_b}
    &\eqqcolon Q_t.
\end{align}
Fix $\delta \in [0, 1]$. According to \cref{lem:z_t_confidence_intervals_using_variance}, $\exists \ \gE_{\delta}$ with $\PP{\left( \gE_{\delta} \right)} \ge 1 - \delta$, and on $\gE_{\delta}$,
for all $t\ge 1$,
\begin{align}
\label{eq:non_vanishing_nl_coefficient_stochastic_npg_value_baseline_special_claim_2_case_1_intermediate_17}
    \left| Z_t(j) \right| \le \frac{1}{8 R_{\max}} \cdot \sqrt{ \left( 1+ S_t^2(j) \right) \cdot \Bigg( 1 + 2 \log{\bigg( \frac{\left(1+ S_t^2(j) \right)^{\frac{1}{2}}}{\delta} \bigg)} \Bigg) }.
\end{align}
Then, on $\gE^\prime \cap \gE_{\delta}$, \cref{eq:non_vanishing_nl_coefficient_stochastic_npg_value_baseline_special_claim_2_case_1_intermediate_16_b} holds and also,
\begin{align}
\label{eq:non_vanishing_nl_coefficient_stochastic_npg_value_baseline_special_claim_2_case_1_intermediate_18}
    \sum_{s=1}^{t-1}{P_s(j)} \le \frac{C^\prime}{8 \cdot  R_{\max}^2} \cdot Q_t. \qquad \left( \text{by \cref{eq:non_vanishing_nl_coefficient_stochastic_npg_value_baseline_special_claim_2_case_1_intermediate_15b}} \right)
\end{align}
According to \cref{eq:non_vanishing_nl_coefficient_stochastic_npg_value_baseline_special_claim_2_case_1_intermediate_13,eq:non_vanishing_nl_coefficient_stochastic_npg_value_baseline_special_claim_2_case_1_intermediate_17,eq:non_vanishing_nl_coefficient_stochastic_npg_value_baseline_special_claim_2_case_1_intermediate_18}, we have, on $\gE^\prime \cap \gE_{\delta}$,
\begin{align}
    \theta_t(j)
     &\le \EE{\theta_1(j)} + \frac{1}{8 R_{\max}} \cdot \sqrt{ \left( 1+ Q_t \right) \cdot \Bigg( 1 + 2 \log{\bigg( \frac{\left(1+ Q_t \right)^{\frac{1}{2}}}{\delta} \bigg)} \Bigg) } + \frac{C^\prime}{8 R_{\max}^2} \cdot Q_t\\
     &\le \EE{\theta_1(j)} + \frac{1}{8 R_{\max}} \cdot \sqrt{ \left( 1+ Q \right) \cdot \Bigg( 1 + 2 \log{\bigg( \frac{\left(1+ Q \right)^{\frac{1}{2}}}{\delta} \bigg)} \Bigg) } + \frac{C^\prime}{8 R_{\max}^2} \cdot Q\,,\label{eq:non_vanishing_nl_coefficient_stochastic_npg_value_baseline_special_claim_2_case_1_intermediate_19}
\end{align}
where 
$Q = \lim_{t \to \infty}{Q_t}$ and 
the inequality follows because $(Q_t)$ is increasing. Note that on $\gE^\prime$, $Q$ is finite almost surely, according to \cref{eq:non_vanishing_nl_coefficient_stochastic_npg_value_baseline_special_claim_2_case_1_intermediate_3,eq:non_vanishing_nl_coefficient_stochastic_npg_value_baseline_special_claim_2_case_1_intermediate_9,eq:non_vanishing_nl_coefficient_stochastic_npg_value_baseline_special_claim_2_case_1_intermediate_16_b}.

Now take any $\omega \in \gE^\prime$. Because $\PP{\left( \gE^\prime \setminus \left( \gE^\prime \cap \gE_{\delta} \right) \right)} \le \PP{\left( \Omega \setminus \gE_{\delta} \right)} \le \delta \to 0$ as $\delta \to 0$, we have that $\PP$-almost surely for all $\omega \in \gE^\prime$
there exists $\delta > 0$ such that $\omega \in \gE^\prime\cap \gE_{\delta}$
while
\cref{eq:non_vanishing_nl_coefficient_stochastic_npg_value_baseline_special_claim_2_case_1_intermediate_19} also holds for this $\delta$.
Take such a $\delta$. By \cref{eq:non_vanishing_nl_coefficient_stochastic_npg_value_baseline_special_claim_2_case_1_intermediate_19},
\begin{align}
\label{eq:non_vanishing_nl_coefficient_stochastic_npg_value_baseline_special_claim_2_case_1_intermediate_20}
    \limsup_{t \to \infty}{\theta_t(j)(\omega)} < \infty.
\end{align}
Hence, almost surely on $\gE^\prime$,
\begin{align}
\label{eq:non_vanishing_nl_coefficient_stochastic_npg_value_baseline_special_claim_2_case_1_intermediate_21}
    c_4 \coloneqq \limsup_{t \to \infty}{\theta_t(j)} < \infty.
\end{align}
Therefore, we have, almost surely on $\gE^\prime$,
\begin{align}
\label{eq:non_vanishing_nl_coefficient_stochastic_npg_value_baseline_special_claim_2_case_1_intermediate_22}
    \sum_{j \in \gA(i)}{ \pi_{\theta_t}(j) } &= \frac{ \sum_{j \in \gA(i)}{ e^{ \theta_t(j) } } }{ \sum_{j \in \gA(i)}{ e^{ \theta_t(j) } } + \sum_{a^+ \in \gA^+(i)}{ e^{ \theta_t(a^+) } } + \sum_{a^- \in \gA^-(i)}{ e^{ \theta_t(a^-) } } } \\
    &\le \frac{ \sum_{j \in \gA(i)}{ e^{ \theta_t(j) } } }{ \sum_{j \in \gA(i)}{ e^{ \theta_t(j) } } + \sum_{a^+ \in \gA^+(i)}{ e^{ \theta_t(a^+) } } } \qquad \big( e^{ \theta_t(a^-) } > 0 \big) \\
    &\le \frac{ \sum_{j \in \gA(i)}{ e^{ \theta_t(j) } } }{ \sum_{j \in \gA(i)}{ e^{ \theta_t(j) } } + e^{c_2} \cdot \left|\gA^+(i) \right| } \qquad \left( \text{by \cref{eq:non_vanishing_nl_coefficient_stochastic_npg_value_baseline_special_claim_2_case_1_intermediate_6a}} \right) \\
    &\le \frac{ e^{c_4} \cdot \left|\gA(i) \right| }{ e^{c_4} \cdot \left|\gA(i) \right| + e^{c_2} \cdot \left|\gA^+(i) \right| } \qquad \left( \text{by \cref{eq:non_vanishing_nl_coefficient_stochastic_npg_value_baseline_special_claim_2_case_1_intermediate_21}} \right) \\
    &\not\to 1,
\end{align}
which is a contradiction with the assumption of \cref{eq:non_vanishing_nl_coefficient_stochastic_npg_value_baseline_special_claim_2_intermediate_15},
showing that $\mathbb{P}(\gE^\prime)=0$.

\textbf{Second case. 2b).} Consider the complement $\gE_0^c$ of $\gE_0$, where $\gE_0$ is by \cref{eq:non_vanishing_nl_coefficient_stochastic_npg_value_baseline_special_claim_2_case_1_intermediate_1}. $\gE_0^c$ indicates the event for at least one ``good'' action $a^+ \in \gA^+(i)$ has infinitely many updates as $t \to \infty$.

We now show that also $\PP(\gE^{\prime\prime})=0$ where
 $\gE^{\prime\prime}=\gE_0^c \cap \{ I=i,i\ne a^* \}
 = (\cup_{a^+\in \cA(i)} \{ N_\infty(a^+)=\infty \}) \cap \{ I=i,i\ne a^* \}$.
It suffices to show that for any $a^+\in \gA^+(i)$, 
$\PP(  \{ N_\infty(a^+)=\infty \}) \cap \{ I=i,i\ne a^* \} )=0$.

Thus, fix an arbitrary $a^+ \in \gA^+(i)$ and let
\[
\gE^\prime\coloneqq\gE_\infty(a^+)\cap \{I=i,i\ne a^*\},
\]
where for $a\in [K]$, $\gE_\infty(a) = \{ N_\infty(a)=\infty \}$. With this notation, the goal is to show that $\PP(\gE^\prime)=0$.%
\footnote{Here, $\gE^\prime$ is redefined to minimize clutter; the previous definition is not used in this part of the proof.}
Since $\gE^\prime\subset \gE_\infty(a^+)$,
the statement follows if 
$\mathbb{P}{\left( \gE_\infty(a^+) \right) } = 0$.
Hence, assume that 
$\mathbb{P}{\left( \gE_\infty(a^+)\right) } > 0$.

Fix $\delta \in [0, 1]$. According to \cref{cor:z_t_confidence_intervals}, there exists an event 
$\gE_{\delta}$ such that $\PP{\left( \gE_{\delta} \right)} \ge 1 - \delta$, and on $\gE_{\delta}$, for all $t\ge 1$,
\begin{align}
\label{eq:non_vanishing_nl_coefficient_stochastic_npg_value_baseline_special_claim_2_case_2_intermediate_1}
    \left| Z_t(a^+) \right| \le \frac{1}{8 R_{\max}} \cdot  \sqrt{ \left( 1+ N_t(a^+) \right) \cdot \Bigg( 1 + 2 \log{\bigg( \frac{\left(1+N_t(a^+) \right)^{\frac{1}{2}}}{\delta} \bigg)} \Bigg) }.
\end{align}
Using a similar calculation as in the proof of \cref{lem:bad_action_parameter_upper_bounded_almost_surely}, we have, on $\gE_{\delta} \cap \gE_{\infty}(a^+)$ that
\begin{align}
\label{eq:non_vanishing_nl_coefficient_stochastic_npg_value_baseline_special_claim_2_case_2_intermediate_2}
    \theta_t(a^+) &\ge \EE{\theta_1(a^+)} - \frac{1}{8 R_{\max}} \cdot  \sqrt{ \left( 1+ N_t(a^+) \right) \cdot \Bigg( 1 + 2 \log{\bigg( \frac{\left(1+N_t(a^+) \right)^{\frac{1}{2}}}{\delta} \bigg)} \Bigg) } \\
    &\qquad + \frac{c}{8 \cdot R_{\max}^2} \cdot \underbrace{ N_{t-1}(a^+) }_{\to \infty} - \frac{c}{8 \cdot R_{\max}^2} \cdot (\tau-1) + P_1(a^+) + \cdots + P_{\tau-1}(a^+).
\end{align}
On $\gE_{\infty}(a^+) \cap \gE_{\delta}$, $N_{t-1}(a^+) \to \infty$ as $t \to \infty$, we have $\theta_t(a^+) \to \infty$ as $t \to \infty$.

Since $\PP{\left( \gE_{\infty}(a^+) \setminus \left( \gE_{\infty}(a^+) \cap \gE_{\delta} \right) \right)} \to 0$ as $\delta \to 0$, we have, almost surely on $\gE_{\infty}(a^+)$, 
\begin{align}
\label{eq:non_vanishing_nl_coefficient_stochastic_npg_value_baseline_special_claim_2_case_2_intermediate_3}
    \lim_{t \to \infty}{ \theta_t(a^+) } = \infty,
\end{align}
which implies that there exists $\tau\ge 1$ such that on 
$\gE^\prime(=\gE_\infty(a^+)\cap \{I=i,i\ne a^*\})$
we have almost surely that $\tau<+\infty$ while we also have that for all $t\ge \tau$,
\begin{align}
\label{eq:non_vanishing_nl_coefficient_stochastic_npg_value_baseline_special_claim_2_case_2_intermediate_3b}
    \sum_{a^- \in \gA^-(i)}  \frac{ r(i) - r(a^-) }{ \exp\{ \theta_t(a^+) - c_1 \}} < \frac{r(a^+)-r(i)}{2}\,.
\end{align}
Hence, on $\gE^\prime$, for $t \ge \tau$, almost surely,
\begin{align}
\MoveEqLeft
\label{eq:non_vanishing_nl_coefficient_stochastic_npg_value_baseline_special_claim_2_case_2_intermediate_4}
    \pi_{\theta_t}^\top r = \sum_{j \in \gA(i)} \pi_{\theta_t}(j) \cdot r(i) + \sum_{a^- \in \gA^-(i)} \pi_{\theta_t}(a^-) \cdot r(a^-) + \sum_{\tilde{a}^+ \in \gA^+(i)} \pi_{\theta_t}(\tilde{a}^+) \cdot r(\tilde{a}^+) \\
    &= r(i) - \sum_{a^- \in \gA^-(i)} \pi_{\theta_t}(a^-) \cdot \left( r(i) - r(a^-) \right) + \sum_{\tilde{a}^+ \in \gA^+(i)} \pi_{\theta_t}(\tilde{a}^+) \cdot \left( r(\tilde{a}^+) - r(i) \right) \\
    &\ge r(i) - \sum_{a^- \in \gA^-(i)} \pi_{\theta_t}(a^-) \cdot \left( r(i) - r(a^-) \right) + \pi_{\theta_t}(a^+) \cdot \left( r(a^+) - r(i) \right) \qquad \left ( r(\tilde{a}^+) - r(i) > 0, \text{ \cref{eq:non_vanishing_nl_coefficient_stochastic_npg_value_baseline_special_claim_2_intermediate_16b}}\right) \\
    &= r(i) + \pi_{\theta_t}(a^+) \cdot \bigg[ \left( r(a^+) - r(i) \right)  - \sum_{a^- \in \gA^-(i)} \frac{ \pi_{\theta_t}(a^-) }{\pi_{\theta_t}(a^+) } \cdot \left( r(i) - r(a^-) \right) \bigg] \\
    &= r(i) + \pi_{\theta_t}(a^+) \cdot \bigg[ \left( r(a^+) - r(i) \right) - \sum_{a^- \in \gA^-(i)}  \frac{ r(i) - r(a^-) }{ \exp\{ \theta_t(a^+) - \theta_t(a^-) \} }  \bigg] \\
    &\ge r(i) + \pi_{\theta_t}(a^+) \cdot \bigg[ \left( r(a^+) - r(i) \right) - \sum_{a^- \in \gA^-(i)}  \frac{ r(i) - r(a^-) }{ \exp\{ \theta_t(a^+) - c_1 \} }  \bigg] \qquad \left( \text{by \cref{eq:non_vanishing_nl_coefficient_stochastic_npg_value_baseline_special_claim_2_intermediate_19}} \right) \\
    &> r(i) + \frac{r(a^+)-r(i)}{2} \cdot \pi_{\theta_t}(a^+)\,. \qquad \left( \text{by \cref{eq:non_vanishing_nl_coefficient_stochastic_npg_value_baseline_special_claim_2_case_2_intermediate_3b}} \right)
\end{align}
Therefore, on 
$\gE^\prime$, for all $s \ge \tau$, for any $j \in \gA(i)$, almost surely,
\begin{align}
\label{eq:non_vanishing_nl_coefficient_stochastic_npg_value_baseline_special_claim_2_case_2_intermediate_5}
    P_s(j) &= \frac{ I_s(j) }{8 \cdot R_{\max}^2} \cdot \left| r(j) - \pi_{\theta_s}^\top r \right| \cdot \left( r(j) - \pi_{\theta_s}^\top r \right) \qquad \left( \text{by \cref{eq:non_vanishing_nl_coefficient_stochastic_npg_value_baseline_special_claim_2_intermediate_12}} \right) \\
    &= - \frac{ I_s(j) }{8 \cdot R_{\max}^2} \cdot \left( r(j) - \pi_{\theta_s}^\top r \right)^2. \qquad \left( \text{by \cref{eq:non_vanishing_nl_coefficient_stochastic_npg_value_baseline_special_claim_2_case_2_intermediate_4}, } r(i) - \pi_{\theta_s}^\top r < 0 \right)
\end{align}
From now on assume that $\gE^\prime$ holds.
Therefore, we have, for all $t \ge \tau$,
\begin{align}
\MoveEqLeft
    \sum_{s=1}^{t-1}{ P_s(j) } = \sum_{s=1}^{\tau - 1}{ P_s(j) } + \sum_{s=\tau}^{t}{ P_s(j) } - P_t(j) \\
    &= \sum_{s=1}^{\tau - 1}{ P_s(j) } - \frac{ 1 }{8 \cdot R_{\max}^2} \cdot \sum_{s=\tau}^{t} \left( r(j) - \pi_{\theta_s}^\top r \right)^2 \cdot I_s(j) -P_t(j) \qquad \left( \text{by \cref{eq:non_vanishing_nl_coefficient_stochastic_npg_value_baseline_special_claim_2_case_2_intermediate_5}} \right) \\
    &= \sum_{s=1}^{\tau - 1}{ P_s(j) } - \frac{ 1 }{8 \cdot R_{\max}^2} \cdot \bigg[ S_t^2(j) - \sum_{s=1}^{\tau-1} \left( r(j) - \pi_{\theta_s}^\top r \right)^2 \cdot I_s(j) \bigg] -P_t(j) \\
    &= - \frac{ 1 }{8 \cdot R_{\max}^2} \cdot S_t^2(j) + \sum_{s=1}^{\tau - 1}{ \bigg[ P_s(j) + \frac{ \left( r(j) - \pi_{\theta_s}^\top r \right)^2 \cdot I_s(j) }{8 \cdot R_{\max}^2} \cdot  \bigg] } - P_t(j) \\
    &\le - \frac{ 1 }{8 \cdot R_{\max}^2} \cdot S_t^2(j) 
     + \frac{ \tau - 1 }{4 \cdot R_{\max}^2} + \frac{ 1}{8 \cdot R_{\max}^2}, \qquad \left( |P_t(j)| \le \frac{ 1}{8 \cdot R_{\max}^2}, \text{ \cref{eq:non_vanishing_nl_coefficient_stochastic_npg_value_baseline_special_claim_2_case_2_intermediate_5}} \right)
\end{align}
where $S_t^2(j) = \sum_{s=1}^{t} \left( r(j) - \pi_{\theta_s}^\top r \right)^2 \cdot I_s(j)$. According to \cref{lem:z_t_confidence_intervals_using_variance}, 
for any $\delta\in [0,1]$, there exist an event $\gE_\delta$ such that $\PP(\gE_\delta)\ge 1-\delta$ and on $\gE_\delta \cap \gE^\prime$,
we have,
\begin{align}
\label{eq:non_vanishing_nl_coefficient_stochastic_npg_value_baseline_special_claim_2_case_2_intermediate_6}
     \theta_t(j)
     &\le \EE{\theta_1(j)} + Z_t(j) + \sum_{s=1}^{t-1}{ P_s(j) } \qquad \left( \text{by \cref{eq:non_vanishing_nl_coefficient_stochastic_npg_value_baseline_special_claim_2_intermediate_4}} \right) \\
     &\le \EE{\theta_1(j)} + \frac{1}{8 R_{\max}} \cdot \sqrt{ \left( 1+ S_t^2(j) \right) \cdot \Bigg( 1 + 2 \log{\bigg( \frac{\left(1+ S_t^2(j) \right)^{\frac{1}{2}}}{\delta} \bigg)} \Bigg) } \\
     &\qquad - \frac{ 1 }{8 \cdot R_{\max}^2} \cdot \left( 1 + S_t^2(j) \right) + \frac{ \tau }{4 \cdot  R_{\max}^2}.
\end{align}
Note that,
\begin{align}
\label{eq:non_vanishing_nl_coefficient_stochastic_npg_value_baseline_special_claim_2_case_2_intermediate_7}
    M(\delta) &\coloneqq \sup_{s \ge 0}{ \frac{1}{8 R_{\max}} \cdot \sqrt{ \left( 1+ s \right) \cdot \Bigg( 1 + 2 \log{\bigg( \frac{\left(1+ s \right)^{\frac{1}{2}}}{\delta} \bigg)} \Bigg) } - \frac{ 1 }{8 \cdot R_{\max}^2} \cdot \left( 1 + s \right)} \\
    &< \infty.
\end{align}
Therefore, on $\gE^\prime \cap \gE_\delta$ for $t\ge \tau$ we have,
\begin{align}
\label{eq:non_vanishing_nl_coefficient_stochastic_npg_value_baseline_special_claim_2_case_2_intermediate_8}
    \theta_t(j) \le \EE{\theta_1(j)} + M(\delta)  + \frac{ \tau }{4 \cdot  R_{\max}^2}\,.
\end{align}
Since $\PP(\gE_\delta^c)\to 0$ as $\delta\to 0$, with an argument parallel to that used in the proof of the first part (cf. the argument around \cref{eq:non_vanishing_nl_coefficient_stochastic_npg_value_baseline_special_claim_2_case_1_intermediate_20}), 
we get that there exists a random constant $c_5(j)$ such that almost surely on $\gE^\prime$, $c_5(j)<\infty$ and
$\sup_{t\ge \tau} \theta_t(j) \le c_5(j)$.
Define $c_5 \coloneqq \max_{j\in \gA(i)}{ c_5(j) }$. Then,
almost surely on $\gE^\prime$, $c_5<\infty$ and
\begin{align}
\label{eq:non_vanishing_nl_coefficient_stochastic_npg_value_baseline_special_claim_2_case_2_intermediate_9}
\sup_{t\ge \tau}\max_{j\in \gA(i)} \theta_t(j) \le c_5\,.
\end{align}

By
\cref{eq:non_vanishing_nl_coefficient_stochastic_npg_value_baseline_special_claim_2_case_2_intermediate_3}, there exists $\tau'\ge 1$, such that almost surely on $\gE^\prime$,
$\tau'<\infty$ while we also have 
\begin{align}
\inf_{t\ge \tau'} \theta_t(a^+)\ge 0,
\label{eq:ttapp}
\end{align}
 for all $t\ge \tau'$.
Hence, on $\gE^\prime$, almost surely for all $t\ge \max(\tau,\tau')$, 
\begin{align}
\label{eq:non_vanishing_nl_coefficient_stochastic_npg_value_baseline_special_claim_2_case_2_intermediate_10}
    \sum_{j \in \gA(i)}{ \pi_{\theta_t}(j) } &= \frac{ \sum_{j \in \gA(i)}{ e^{ \theta_t(j) } } }{ \sum_{j \in \gA(i)}{ e^{ \theta_t(j) } } + \sum_{\tilde{a}^+ \in \gA^+(i)}{ e^{ \theta_t(\tilde{a}^+) } } + \sum_{a^- \in \gA^-(i)}{ e^{ \theta_t(a^-) } } } \\
    &\le \frac{ \sum_{j \in \gA(i)}{ e^{ \theta_t(j) } } }{ \sum_{j \in \gA(i)}{ e^{ \theta_t(j) } } + 	{ e^{ \theta_t(a^+) } } } \qquad \big( e^{ \theta_t(k) } > 0 \text{ for any } k\in [K] \big) \\
    &\le \frac{ \sum_{j \in \gA(i)}{ e^{ \theta_t(j) } } }
    { \sum_{j \in \gA(i)}{ e^{ \theta_t(j) } } 
    + 1 } \qquad \left( \text{by \cref{eq:ttapp} } \right) \\
    &\le \frac{ e^{c_5} \cdot \left|\gA(i) \right| }{ e^{c_5} \cdot \left|\gA(i) \right| + 1 } \qquad \left( \text{by \cref{eq:non_vanishing_nl_coefficient_stochastic_npg_value_baseline_special_claim_2_case_2_intermediate_8}} \right) \\
    &\not\to 1\,.
\end{align}
Hence, $\mathbb{P}( \gE^\prime )=0$, finishing the proof.
\end{proof}

Let us now turn to the proof of the results that were used in the above proof.
\begin{lemma}
\label{lem:bad_action_parameter_upper_bounded_almost_surely}
Let $I$ be as in \cref{eq:non_vanishing_nl_coefficient_stochastic_npg_value_baseline_special_claim_2_intermediate_15},
let $i$ be a sub-optimal action,
and let $\tau$ be as in \cref{eq:non_vanishing_nl_coefficient_stochastic_npg_value_baseline_special_claim_2_intermediate_18},
Then,
on $\{ I=i, i\ne a^* \}$, for any action $a^- \in \gA^-(i)$ (using \cref{update_rule:softmax_natural_pg_special_on_policy_stochastic_gradient_value_baseline})
we have, almost surely,
\begin{align}
\label{eq:bad_action_parameter_upper_bounded_almost_surely_result_1}
    c_1 \coloneqq \sup_{t \ge 1}{ \theta_t(a^-)} < \infty.
\end{align}
\end{lemma}
\begin{proof}
According to \cref{eq:non_vanishing_nl_coefficient_stochastic_npg_value_baseline_special_claim_2_intermediate_12}, we have, for all $t \ge \tau$,
\begin{align}
\label{eq:bad_action_parameter_upper_bounded_almost_surely_intermediate_1}
    P_t(a^-)& = \frac{ I_t(a^-) }{8 \cdot R_{\max}^2} \cdot \left| r(a^-) - \pi_{\theta_t}^\top r \right| \cdot \left( r(a^-) - \pi_{\theta_t}^\top r \right) \\
    &\le - c \cdot \frac{ I_t(a^-) }{8 \cdot R_{\max}^2}, \qquad \left( \text{by \cref{eq:non_vanishing_nl_coefficient_stochastic_npg_value_baseline_special_claim_2_intermediate_18}} \right)
\end{align}
which implies that,
\begin{align}
\label{eq:bad_action_parameter_upper_bounded_almost_surely_intermediate_2}
    \theta_{t}(a^-) &= \EE{\theta_1(a^-)} + Z_t(a^-) + P_1(a^-) + \cdots + P_{\tau-1}(a^-) \qquad \left( \text{by \cref{eq:non_vanishing_nl_coefficient_stochastic_npg_value_baseline_special_claim_2_intermediate_4}} \right) \\
    &\qquad + P_{\tau}(a^-) + \cdots +  P_{t-1}(a^-) \\
    &\le \EE{\theta_1(a^-)} + Z_t(a^-) + P_1(a^-) + \cdots + P_{\tau-1}(a^-) \\
    &\qquad - \frac{c}{8 \cdot R_{\max}^2} \cdot \left( I_{\tau}(a^-) + \cdots + I_{t-1}(a^-) \right) \qquad \left( \text{by \cref{eq:bad_action_parameter_upper_bounded_almost_surely_intermediate_1}} \right) \\
    &= \EE{\theta_1(a^-)} + Z_t(a^-) + P_1(a^-) + \cdots + P_{\tau-1}(a^-) \\
    &\qquad - \frac{c}{8 \cdot R_{\max}^2} \cdot N_{\tau:t-1}(a^-) \qquad \left( \text{\cref{eq:non_vanishing_nl_coefficient_stochastic_npg_value_baseline_special_claim_2_intermediate_14c}} \right)
\end{align}
Denote $\gE_{\infty}(a) \coloneqq \{ N_\infty(a) = \infty \}$, where $N_\infty(a)$ is defined in \cref{eq:non_vanishing_nl_coefficient_stochastic_npg_value_baseline_special_claim_2_intermediate_14b}.

Fix $\delta \in [0, 1]$. Take $\gE_{\delta}$ from \cref{cor:z_t_confidence_intervals}. Consider on event $\gE_{\infty}(a^-) \cap \gE_{\delta}$, we have,
\begin{align}
\label{eq:bad_action_parameter_upper_bounded_almost_surely_intermediate_3}
    \theta_{t}(a^-) &\le \EE{\theta_1(a^-)} + \frac{1}{8 R_{\max}} \cdot  \sqrt{ \left( 1+ N_t(a) \right) \cdot \Bigg( 1 + 2 \log{\bigg( \frac{\left(1+N_t(a) \right)^{\frac{1}{2}}}{\delta} \bigg)} \Bigg) } \\
    &\qquad - \frac{c}{8 \cdot R_{\max}^2} \cdot N_{\tau:t-1}(a^-) + P_1(a^-) + \cdots + P_{\tau-1}(a^-).
\end{align}
Note that,
\begin{align}
\label{eq:bad_action_parameter_upper_bounded_almost_surely_intermediate_4}
    N_{\tau:t-1}(a^-) &= N_{t-1}(a^-) - N_{1:\tau-1}(a^-) \qquad \left( \text{\cref{eq:non_vanishing_nl_coefficient_stochastic_npg_value_baseline_special_claim_2_intermediate_14a,eq:non_vanishing_nl_coefficient_stochastic_npg_value_baseline_special_claim_2_intermediate_14c}} \right)\\
    &\ge N_{t-1}(a^-) - \left( \tau - 1 \right).
\end{align}
We have,
\begin{align}
\label{eq:bad_action_parameter_upper_bounded_almost_surely_intermediate_5}
    \theta_{t}(a^-) &\le \EE{\theta_1(a^-)} + \frac{1}{8 R_{\max}} \cdot  \sqrt{ \left( 1+ N_t(a) \right) \cdot \Bigg( 1 + 2 \log{\bigg( \frac{\left(1+N_t(a) \right)^{\frac{1}{2}}}{\delta} \bigg)} \Bigg) } \\
    &\qquad - \frac{c}{8 \cdot R_{\max}^2} \cdot \underbrace{ N_{t-1}(a^-) }_{\to \infty} + \frac{c}{8 \cdot R_{\max}^2} \cdot ( \tau - 1) + P_1(a^-) + \cdots + P_{\tau-1}(a^-).
\end{align}
On $\gE_{\infty}(a^-) \cap \gE_{\delta}$, $N_{t-1}(a^-) \to \infty$ as $t \to \infty$, we have $\theta_t(a^-) \to - \infty$ as $t \to \infty$.

Since $\PP{\left( \gE_{\infty}(a^-) \setminus \left( \gE_{\infty}(a^-) \cap \gE_{\delta} \right) \right)} \to 0$ as $\delta \to 0$, we have, almost surely on $\gE_{\infty}(a^-)$, 
\begin{align}
\label{eq:bad_action_parameter_upper_bounded_almost_surely_intermediate_6}
    \lim_{t \to \infty}{ \theta_t(a^-) } = - \infty,
\end{align}
which implies that on $\gE_{\infty}(a^-)$, we have $\sup_{t \ge 1}{ \theta_t(a^-) } < \infty$. 

On the other hand, on $\left( \gE_{\infty}(a^-) \right)^c$, we have $\sup_{t \ge 1}{ \theta_t(a^-) } < \infty$ by construction (finitely many updates of $a^-$ as $t \to \infty$, and each update is bounded according to \cref{eq:non_vanishing_nl_coefficient_stochastic_npg_value_baseline_special_claim_2_case_1_intermediate_7}).

Therefore, we have $\sup_{t \ge 1}{ \theta_t(a^-) } < \infty$ almost surely.
\end{proof}

\begin{lemma}[Lemma 6 in \citep{abbasi2011improved}]
\label{lem:improved_algorithm_confidence_intervals}
Let ${X}_t = \sum_{s=1}^{t}{I_s \cdot \eta_s}$, and $N_t = \sum_{s=1}^{t}{I_s}$. Assume $\eta_t$ is conditionally \ $\sigma$-sub-Gaussian, and $I_t$ is $\gF_t$-measurable. Then,
for all $\delta\in [0,1]$,
 with probability $1 - \delta$, for all $ t \ge 1$,
\begin{align}
    \left| {X}_t \right| \le \sigma \cdot \sqrt{ \left( 1+ N_t \right) \cdot \Bigg( 1 + 2 \log{\bigg( \frac{\left(1+N_t \right)^{\frac{1}{2}}}{\delta} \bigg)} \Bigg) }.
\end{align}
\end{lemma}

\begin{corollary}
\label{cor:z_t_confidence_intervals}
For all $a \in [K]$, $\forall \delta$, $\exists \ \gE_{\delta}$ with $\PP{\left( \gE_{\delta} \right)} \ge 1 - \delta$, such that on $\gE_{\delta}$, for all $ t \ge 1$,
\begin{align}
    \left| Z_t(a) \right| \le \frac{1}{8 R_{\max}} \cdot \sqrt{ \left( 1+ N_t(a) \right) \cdot \Bigg( 1 + 2 \log{\bigg( \frac{\left(1+N_t(a) \right)^{\frac{1}{2}}}{\delta} \bigg)} \Bigg) }.
\end{align}
\end{corollary}

\begin{lemma}
\label{lem:z_t_confidence_intervals_using_variance}
For all $a \in [K]$, $\forall \delta \in [0, 1]$, $\exists \ \gE_{\delta}$ with $\PP{\left( \gE_{\delta} \right)} \ge 1 - \delta$, such that on $\gE_{\delta}$, for all $ t \ge 1$,
\begin{align}
    \left| Z_t(a) \right| \le \frac{1}{8 R_{\max}} \cdot \sqrt{ \left( 1+ S_t^2(a) \right) \cdot \Bigg( 1 + 2 \log{\bigg( \frac{\left(1+ S_t^2(a) \right)^{\frac{1}{2}}}{\delta} \bigg)} \Bigg) },
\end{align}
where $S_t^2(a) \coloneqq \sum_{s=1}^{t} \left( r(a) - \pi_{\theta_s}^\top r \right)^2 \cdot I_s(a)$.
\end{lemma}
\begin{proof}
Follow the steps of the proof of Lemma 6 in \citep{abbasi2011improved}.
\end{proof}

\textbf{\cref{thm:almost_sure_convergence_rate_stochastic_npg_special_value_baseline} }(Almost sure global convergence rate)\textbf{.}
Using \cref{update_rule:softmax_natural_pg_special_on_policy_stochastic_gradient_value_baseline} with on-policy sampling $a_t \sim \pi_{\theta_t}(\cdot)$, the IS estimator in \cref{def:on_policy_importance_sampling}, $\eta$ in \cref{eq:non_uniform_lojasiewicz_stochastic_npg_value_baseline_special_stochastic_reward_result_1}, and any initialization $\theta_1 \in \sR^K$, we have,
\begin{align}
    \EE{ \left( \pi^* - \pi_{\theta_t} \right)^\top r } \le \frac{16 \cdot R_{\max}^2 }{ \Delta \cdot \EE{ c^2 } } \cdot \frac{K-1}{t},& \qquad \text{and} \\
    \limsup_{t \ge 1} \bigg\{ \frac{ \Delta \cdot c^2 }{16 \cdot R_{\max}^2 } \cdot \frac{t}{K-1} \cdot \left( \pi^* - \pi_{\theta_t} \right)^\top r \bigg\} < \infty,& \qquad \text{a.s.},
\end{align}
where $\EEt{\cdot}$ denotes $\EEt{\cdot | \gF_t}$, and $\gF_t$ is the $\sigma$-algebra generated by $a_1, x_1(a_1), \dots, a_{t-1}, x_{t-1}(a_{t-1})$, $\pi^* \coloneqq \argmax_{\pi \in \Delta(K)}{ \pi^\top r}$ is the optimal policy, $R_{\max}$ is the sampled reward range from \cref{assump:bounded_reward}, $\Delta \coloneqq r(a^*) - \max_{a \not= a^*}{ r(a) }$ is the reward gap of $r$, and $c > 0$ is from \cref{lem:non_vanishing_nl_coefficient_stochastic_npg_value_baseline_special}.
\begin{proof}
\textbf{First part.} According to \cref{lem:non_uniform_lojasiewicz_stochastic_npg_value_baseline_special}, we have,
\begin{align}
    \EEt{\pi_{\theta_{t+1}}^\top r} - \pi_{\theta_t}^\top r &\ge \frac{ 1 }{16 \cdot R_{\max}^2} \cdot \frac{\Delta}{K-1} \cdot \pi_{\theta_t}(a^*)^2 \cdot \left( r(a^*) - \pi_{\theta_t}^\top r \right)^2 \\
    &\ge \frac{ 1 }{16 \cdot R_{\max}^2} \cdot \frac{\Delta}{K-1} \cdot \inf_{t \ge 1}\pi_{\theta_t}(a^*)^2 \cdot \left( r(a^*) - \pi_{\theta_t}^\top r \right)^2 \\
    &= \frac{ 1 }{16 \cdot R_{\max}^2} \cdot \frac{\Delta}{K-1} \cdot c^2 \cdot \left( r(a^*) - \pi_{\theta_t}^\top r \right)^2,
\end{align}
where $c \coloneqq \inf_{t \ge 1}\pi_{\theta_t}(a^*) > 0$ is according to \cref{lem:non_vanishing_nl_coefficient_stochastic_npg_value_baseline_special}.
Let $\delta(\theta_t) \coloneqq \left( \pi^* - \pi_{\theta_t} \right)^\top r$ denote the sub-optimality gap. We have,
\begin{align}
    \delta(\theta_t) - \EEt{ \delta(\theta_{t+1})} &= \left( \pi^* - \pi_{\theta_t} \right)^\top r - \mathbb{E}_t{ \Big[ \left( \pi^* - \pi_{\theta_{t+1}} \right)^\top r \Big] } \\
    &= \left( \pi^* - \pi_{\theta_t} \right)^\top r - \left( \pi^* - \EEt{\pi_{\theta_{t+1}}} \right)^\top r \\
    &= \EEt{\pi_{\theta_{t+1}}^\top r} - \pi_{\theta_t}^\top r \\
    &\ge \frac{ 1 }{16 \cdot R_{\max}^2} \cdot \frac{\Delta}{K-1} \cdot c^2 \cdot \left( r(a^*) - \pi_{\theta_t}^\top r \right)^2 \\
    &= \frac{ 1 }{16 \cdot R_{\max}^2} \cdot \frac{\Delta}{K-1} \cdot c^2 \cdot \delta(\theta_t)^2.
\end{align}
Taking expectation, we have,
\begin{align}
    \expectation{ [ \delta(\theta_t) ]} - \expectation{ [ \delta(\theta_{t+1}) ]} &\ge \frac{ \Delta \cdot \EE{ c^2 } }{16 \cdot R_{\max}^2} \cdot \frac{1}{K-1} \cdot \expectation{ [ \delta(\theta_t)^2 ] } \\
    &\ge \frac{ \Delta \cdot \EE{ c^2 } }{16 \cdot R_{\max}^2} \cdot \frac{1}{K-1} \cdot \left( \expectation{ [ \delta(\theta_t) ] } \right)^2. \qquad \left( \text{by Jensen's inequality}\right)
\end{align}
Therefore, we have, for all $t \ge 1$,
\begin{align}
\label{eq:almost_sure_convergence_rate_stochastic_npg_special_value_baseline_intermediate_20}
    \frac{1}{ \expectation{ [ \delta(\theta_t) ]} } &= \frac{1}{\expectation{ [ \delta(\theta_{1}) ]}} + \sum_{s=1}^{t-1}{ \left[ \frac{1}{\expectation{ [ \delta(\theta_{s+1}) ]}} - \frac{1}{\expectation{ [ \delta(\theta_{s}) ]}} \right] } \\
    &= \frac{1}{\expectation{ [ \delta(\theta_{1}) ]}} + \sum_{s=1}^{t-1}{ \frac{1}{\expectation{ [ \delta(\theta_{s+1}) ]} \cdot \expectation{ [ \delta(\theta_{s}) ]} } \cdot \left( \expectation{ [ \delta(\theta_{s}) ]} - \expectation{ [ \delta(\theta_{s+1}) ]} \right) } \\
    &\ge \frac{1}{\expectation{ [ \delta(\theta_{1}) ]}} + \sum_{s=1}^{t-1}{ \frac{1}{\expectation{ [ \delta(\theta_{s+1}) ]} \cdot \expectation{ [ \delta(\theta_{s}) ]} } \cdot \frac{ \Delta \cdot \EE{ c^2 } }{16 \cdot R_{\max}^2} \cdot \frac{1}{K-1} \cdot \left( \expectation{ [ \delta(\theta_s) ] } \right)^2 } \\
    &\ge \frac{1}{\expectation{ [ \delta(\theta_{1}) ]}} + \sum_{s=1}^{t-1}{ \frac{ \Delta \cdot \EE{ c^2 } }{16 \cdot R_{\max}^2 } \cdot \frac{1}{K-1} }  \qquad \left(  \expectation{ [ \delta(\theta_{s}) ]} \ge  \expectation{ [ \delta(\theta_{s+1}) ]} > 0 \right) \\
    &= \frac{1}{\expectation{ [ \delta(\theta_{1}) ]}} + \frac{ \Delta \cdot \EE{ c^2 } }{16 \cdot R_{\max}^2 } \cdot \frac{1}{K-1} \cdot \left( t - 1 \right) \\
    &\ge \frac{ \Delta \cdot \EE{ c^2 } }{16 \cdot R_{\max}^2} \cdot \frac{t}{K-1}, \qquad \left( \expectation{ [ \delta(\theta_{1}) ]} \le 1 < \frac{16 \cdot R_{\max}^2 }{ \Delta \cdot \EE{ c^2 } } \cdot \left( K - 1 \right)  \right)
\end{align}
which implies that, for all $t \ge 1$,
\begin{align}
    \EE{ \left( \pi^* - \pi_{\theta_t} \right)^\top r } = \expectation{ [ \delta(\theta_t) ]} \le \frac{16 \cdot R_{\max}^2}{ \Delta \cdot \EE{ c^2 } } \cdot \frac{K-1}{t}.
\end{align}
\textbf{Second part.} 
The result follows from the following \cref{lem:problemma} by choosing $X_t = \left( \pi^* - \pi_{\theta_t} \right)^\top r$
and $f(t) =\frac{ \Delta \cdot \EE{ c^2 } }{16 \cdot R_{\max}^2 } \cdot \frac{t}{K-1}$.
\end{proof}

\begin{lemma}
\label{lem:problemma}
Let $(X_t)_{t\ge 1}$ be a sequence of random variables such that 
$X_t\in [0,1]$, $X_t\to 0$ almost surely and for $t\ge 1$, 
$\EE{X_t} \le \frac{1}{f(t)}$ with $f(t)\to\infty$ as $t\to\infty$.
Then $\limsup_{t\to\infty} f(t) X_t <\infty$ almost surely.
\end{lemma}
\begin{proof}[Proof of \cref{lem:problemma}]
Let $\gE$ be the event when $\limsup_{t \to\infty} \big\{ f(t) \cdot X_t \big\} = \infty$. 
It suffices to show that $\mathbb{P}(\gE)=0$.
Consider the event $\gE$. On this event,
there exists a strictly increasing sequence $\{ t_k \}_{k\ge 1}$, 
such that $ f(t_k) \cdot X_{t_k} \to \infty$ as $k \to \infty$. Since $X_t \ge 0$, we have,
\begin{align}
    \EE{X_{t_k}} \ge \EE{X_{t_k} \cdot \sI_{\gE}}.
\end{align}
Then we have,
\begin{align}
    1 &\ge \lim_{k \to \infty}{ \EE{ f(t_k) \cdot X_{t_k} } } \\
    &\ge \lim_{k \to \infty}{ \EE{ f(t_k) \cdot X_{t_k} \cdot \sI_{\gE} } } \\
    &= \liminf_{k \to \infty}{ \EE{ f(t_k) \cdot X_{t_k} \cdot \sI_{\gE} } } \\
    &\ge \EE{ (\liminf_{k \to \infty}{  f(t_k) \cdot X_{t_k}) \cdot \sI_{\gE} } }. \qquad \left( \text{Fatou's lemma} \right)
\end{align}
If $\mathbb{P}(\gE)>0$, the right-hand side above is $\infty$, which would imply that $\infty\le 1$. 
Hence, we must have $\mathbb{P}(\gE)=0$.
\end{proof}

\section{Proofs for General MDPs}

\textbf{\cref{lem:non_uniform_lojasiewicz_stochastic_npg_value_baseline_general} }(Stochastic N\L{})\textbf{.}
Using \cref{alg:softmax_natural_pg_general_on_policy_stochastic_gradient_deterministic_value} with constant $\eta > 0$, we have, for all $t \ge 1$,
\begin{align}
\MoveEqLeft
\label{eq:non_uniform_lojasiewicz_stochastic_npg_value_baseline_general_deterministic_value_result_1_appendix}
    V^{\pi_{\theta_{t+1}}}(s_0) - V^{\pi_{\theta_t}}(s_0) \ge 0,  \qquad \text{ a.s., } \qquad \forall s_0 \in \gS, \qquad \text{and} \\
\label{eq:non_uniform_lojasiewicz_stochastic_npg_value_baseline_general_deterministic_value_result_2_appendix}
    \EEt{ V^{\pi_{\theta_{t+1}}}(\mu) } -  V^{\pi_{\theta_t}}(\mu) &\ge \frac{\eta \cdot \left( 1-\gamma \right)^4 }{1 + \eta} \cdot \min_{s}{\mu(s) } \cdot \bigg\| \frac{d_{\mu}^{\pi^*}}{\mu} \bigg\|_\infty^{-1} \cdot \frac{ \min_{s}{ \pi_{\theta_t}(a^*(s) | s)^2 } }{S} \cdot \big( V^{\pi^*}(\mu) - V^{\pi_{\theta_t}}(\mu) \big)^2,
\end{align}
where $\EEt{\cdot}$ is on randomness from state sampling $s_t \sim d_\mu^{\pi_{\theta_t}}(\cdot)$ and on-policy sampling $a_t \sim \pi_{\theta_t}(\cdot | s_t)$, and $a^*(s)$ is the action selected by the optimal policy $\pi^*$ under state $s$.
\begin{proof}
For all $t \ge 1$, for any state action pair $(s ,i) \in \gS \times \gA$, denote
\begin{align}
\label{eq:non_uniform_lojasiewicz_stochastic_npg_value_baseline_general_deterministic_value_intermediate_1}
    \big[ V^{\pi_{\theta_{t+1}}}(s_0) \ | \ s_t = s, a_t = i \big]
\end{align}
as the the value of $V^{\pi_{\theta_{t+1}}}(s_0)$ given the sampled state action pair $(s_t, a_t) = (s, i)$.

Given $s_t = s$, for all $s^\prime \not= s$, we have, for all $a \in \gA$,
\begin{align}
\label{eq:non_uniform_lojasiewicz_stochastic_npg_value_baseline_general_deterministic_value_intermediate_2}
    \pi_{\theta_{t+1}}(a | s^\prime) &= \frac{ \exp\{ \theta_{t+1}(s^\prime, a) \} }{ \sum_{a^\prime \in \gA}{\exp\{ \theta_{t+1}(s^\prime, a^\prime) \}} } \\
    &= \frac{ \exp\{ \theta_{t}(s^\prime, a) \} }{ \sum_{a^\prime \in \gA}{\exp\{ \theta_{t}(s^\prime, a^\prime) \}} } \qquad \left( s^\prime \not= s_t, \text{ \cref{alg:softmax_natural_pg_general_on_policy_stochastic_gradient_deterministic_value}} \right) \\
    &= \pi_{\theta_t}(a | s^\prime).
\end{align}
According to the performance difference \cref{lem:performance_difference_general}, we have,
\begin{align}
\MoveEqLeft
\label{eq:non_uniform_lojasiewicz_stochastic_npg_value_baseline_general_deterministic_value_intermediate_3}
    \big[ V^{\pi_{\theta_{t+1}}}(s_0) \ | \ s_t = s, a_t = i \big] - V^{\pi_{\theta_t}}(s_0) \\
    &= \frac{1}{1 - \gamma} \cdot \sum_{s^\prime \in \gS}{ d_{s_0}^{\pi_{\theta_{t+1}}}(s^\prime) \cdot \sum_{a}{ \left( \pi_{\theta_{t+1}}(a | s^\prime) - \pi_{\theta_{t}}(a | s^\prime) \right) \cdot Q^{\pi_{\theta_t}}(s^\prime,a) } } \\
    &= \frac{1}{1 - \gamma} \cdot d_{s_0}^{\pi_{\theta_{t+1}}}(s) \cdot \sum_{a}{ \left( \pi_{\theta_{t+1}}(a | s) - \pi_{\theta_{t}}(a | s) \right) \cdot Q^{\pi_{\theta_t}}(s,a) }. \qquad \left( \text{by \cref{eq:non_uniform_lojasiewicz_stochastic_npg_value_baseline_general_deterministic_value_intermediate_2}} \right)
\end{align}
Note that, in the above equation $d_{s_0}^{\pi_{\theta_{t+1}}}(s) =  \big[ d_{s_0}^{\pi_{\theta_{t+1}}}(s) \ | \ s_t = s, a_t = i \big]$, which means that for each sampled state action pair $(s_t, a_t) = (s, i)$, we have a different $\pi_{\theta_{t+1}}$ and thus $d_{s_0}^{\pi_{\theta_{t+1}}}$.
According to the update in \cref{alg:softmax_natural_pg_general_on_policy_stochastic_gradient_deterministic_value}, we have,
\begin{align}
\label{eq:non_uniform_lojasiewicz_stochastic_npg_value_baseline_general_deterministic_value_intermediate_4}
\MoveEqLeft
    \bigg[ \sum_{a} \pi_{\theta_{t+1}}(a | s) \cdot Q^{\pi_{\theta_{t}}}(s, a) \ \Big| \ s_t = s, a_t = i \bigg] \\
    &= \frac{\exp\Big\{ \theta_t(s, i) + \eta \cdot \frac{ Q^{\pi_{\theta_t}}(s,i) - V^{\pi_{\theta_t}}(s) }{\pi_{\theta_t}(i | s)} \Big\} \cdot Q^{\pi_{\theta_t}}(s,i) + \sum_{j \not= i}{ \exp\{ \theta_t(s, j) \} \cdot Q^{\pi_{\theta_t}}(s,j) } }{ \exp\Big\{ \theta_t(s,i) + \eta \cdot \frac{ Q^{\pi_{\theta_t}}(s,i) - V^{\pi_{\theta_t}}(s) }{\pi_{\theta_t}(i | s)} \Big\} + \sum_{j \not= i}{ \exp\{ \theta_t(s,j) \}} }, 
\end{align}
which is similar to \cref{eq:non_uniform_lojasiewicz_stochastic_npg_value_baseline_special_deterministic_reward_intermediate_2}. Therefore, by algebra we have,
\begin{align}
\MoveEqLeft
\label{eq:non_uniform_lojasiewicz_stochastic_npg_value_baseline_general_deterministic_value_intermediate_5}
    \bigg[ \sum_{a} \left( \pi_{\theta_{t+1}}(a | s) - \pi_{\theta_{t}}(a | s) \right) \cdot Q^{\pi_{\theta_{t}}}(s, a) \ | \ s_t = s, a_t = i \bigg] \\
    &= \frac{ \left[ \exp\Big\{ \eta \cdot \frac{ Q^{\pi_{\theta_t}}(s,i) - V^{\pi_{\theta_t}}(s)}{\pi_{\theta_t}(i | s)} \Big\} - 1 \right] \cdot \left( Q^{\pi_{\theta_{t}}}(s, i) - V^{\pi_{\theta_t}}(s) \right) }{ \exp\Big\{ \eta \cdot \frac{Q^{\pi_{\theta_t}}(s,i) - V^{\pi_{\theta_t}}(s)}{\pi_{\theta_t}(i | s)} \Big\} + \frac{ 1 - \pi_{\theta_t}(i | s) }{ \pi_{\theta_t}(i | s) } } \ge 0,
\end{align}
where the last inequality is from $\left( e^{ c \cdot y} - 1 \right) \cdot y \ge 0$ for all $y \in \sR$ with $c \coloneqq \frac{\eta}{\pi_{\theta_t}(i | s)} > 0$.

Combining \cref{eq:non_uniform_lojasiewicz_stochastic_npg_value_baseline_general_deterministic_value_intermediate_3,eq:non_uniform_lojasiewicz_stochastic_npg_value_baseline_general_deterministic_value_intermediate_5}, we have,
\begin{align}
\MoveEqLeft
\label{eq:non_uniform_lojasiewicz_stochastic_npg_value_baseline_general_deterministic_value_intermediate_6}
    \big[ V^{\pi_{\theta_{t+1}}}(s_0) \ | \ s_t = s, a_t = i \big] - V^{\pi_{\theta_t}}(s_0) \\
    &= \frac{ d_{s_0}^{\pi_{\theta_{t+1}}}(s) }{1 - \gamma} \cdot \frac{ \left[ \exp\Big\{ \eta \cdot \frac{ Q^{\pi_{\theta_t}}(s,i) - V^{\pi_{\theta_t}}(s)}{\pi_{\theta_t}(i | s)} \Big\} - 1 \right] \cdot \left( Q^{\pi_{\theta_{t}}}(s, i) - V^{\pi_{\theta_t}}(s) \right) }{ \exp\Big\{ \eta \cdot \frac{Q^{\pi_{\theta_t}}(s,i) - V^{\pi_{\theta_t}}(s)}{\pi_{\theta_t}(i | s)} \Big\} + \frac{ 1 - \pi_{\theta_t}(i | s) }{ \pi_{\theta_t}(i | s) } } \ge 0,
\end{align}
which proves \cref{eq:non_uniform_lojasiewicz_stochastic_npg_value_baseline_general_deterministic_value_result_1_appendix} because of $(s, i) \in \gS \times \gA$ is arbitrary.

For all $t \ge 1$, given current policy $\pi_{\theta_t}$, the value function of next policy $V^{\pi_{\theta_{t+1}}}(\mu)$ is a random variable, and the randomness is from state sampling $s_t \sim d_\mu^{\pi_{\theta_t}}(\cdot)$ and on-policy sampling $a_t \sim \pi_{\theta_t}(\cdot | s_t)$. According to \cref{eq:non_uniform_lojasiewicz_stochastic_npg_value_baseline_general_deterministic_value_intermediate_6}, the expected progress after one update is,
\begin{align}
\MoveEqLeft
\label{eq:non_uniform_lojasiewicz_stochastic_npg_value_baseline_general_deterministic_value_intermediate_7}
    \EEt{ V^{\pi_{\theta_{t+1}}}(\mu) } - V^{\pi_{\theta_t}}(\mu) = \sum_{s} d_{\mu}^{\pi_{\theta_t}}(s) \sum_{i} \pi_{\theta_t}(i | s) \cdot \left( \big[ V^{\pi_{\theta_{t+1}}}(\mu) \ | \ s_t = s, a_t = i \big] - V^{\pi_{\theta_t}}(\mu) \right) \\
    &= \sum_{s} d_{\mu}^{\pi_{\theta_t}}(s) \sum_{i} \pi_{\theta_t}(i | s) \cdot \frac{d_{\mu}^{\pi_{\theta_{t+1}}}(s)}{1 - \gamma} \cdot \frac{ \left[ \exp\Big\{ \eta \cdot \frac{ Q^{\pi_{\theta_t}}(s,i) - V^{\pi_{\theta_t}}(s)}{\pi_{\theta_t}(i | s)} \Big\} - 1 \right] \cdot \left( Q^{\pi_{\theta_{t}}}(s, i) - V^{\pi_{\theta_t}}(s) \right) }{ \exp\Big\{ \eta \cdot \frac{Q^{\pi_{\theta_t}}(s,i) - V^{\pi_{\theta_t}}(s)}{\pi_{\theta_t}(i | s)} \Big\} + \frac{ 1 - \pi_{\theta_t}(i | s) }{ \pi_{\theta_t}(i | s) } } \\
    &\ge \sum_{s} \mu(s) \cdot d_{\mu}^{\pi_{\theta_t}}(s) \sum_{i} \pi_{\theta_t}(i | s) \cdot \frac{ \left[ \exp\Big\{ \eta \cdot \frac{ Q^{\pi_{\theta_t}}(s,i) - V^{\pi_{\theta_t}}(s)}{\pi_{\theta_t}(i | s)} \Big\} - 1 \right] \cdot \left( Q^{\pi_{\theta_{t}}}(s, i) - V^{\pi_{\theta_t}}(s) \right) }{ \exp\Big\{ \eta \cdot \frac{Q^{\pi_{\theta_t}}(s,i) - V^{\pi_{\theta_t}}(s)}{\pi_{\theta_t}(i | s)} \Big\} + \frac{ 1 - \pi_{\theta_t}(i | s) }{ \pi_{\theta_t}(i | s) } },
\end{align}
where the inequality is because of \cref{eq:non_uniform_lojasiewicz_stochastic_npg_value_baseline_general_deterministic_value_intermediate_5} and for any $\theta$ and $\mu$,
\begin{align}
\label{eq:non_uniform_lojasiewicz_stochastic_npg_value_baseline_general_deterministic_value_intermediate_8}
    d_{\mu}^{\pi_\theta}(s) &= \expectation_{s_0 \sim \mu}{ \left[ d_{\mu}^{\pi_\theta}(s) \right] } \\
    &= \expectation_{s_0 \sim \mu}{ \bigg[ (1 - \gamma) \cdot \sum_{t=0}^{\infty}{ \gamma^t \cdot \PP(s_t = s \ | \ s_0, \pi_\theta, \gP) } \bigg] } \\
    &\ge (1 - \gamma) \cdot \expectation_{s_0 \sim \mu}{ \left[  \PP(s_0 = s | s_0)  \right] } \\
    &= (1 - \gamma) \cdot \mu(s).
\end{align}
Partition the action set $\gA$ under state $s \in \gS$ into three parts using $V^{\pi_{\theta_t}}(s)$ as follows,
\begin{align}
\label{eq:non_uniform_lojasiewicz_stochastic_npg_value_baseline_general_deterministic_value_intermediate_9}
    \gA_t^0(s) &\coloneqq \left\{ a^0 \in \gA: Q^{\pi_{\theta_t}}(s,a^0) = V^{\pi_{\theta_t}}(s) \right\}, \\
    \gA_t^+(s) &\coloneqq \left\{ a^+ \in \gA: Q^{\pi_{\theta_t}}(s,a^+) > V^{\pi_{\theta_t}}(s) \right\}, \\
    \gA_t^-(s) &\coloneqq \left\{ a^- \in \gA: Q^{\pi_{\theta_t}}(s,a^-) < V^{\pi_{\theta_t}}(s) \right\}.
\end{align}
From \cref{eq:non_uniform_lojasiewicz_stochastic_npg_value_baseline_general_deterministic_value_intermediate_7}, we have,
\begin{align}
\MoveEqLeft
\label{eq:non_uniform_lojasiewicz_stochastic_npg_value_baseline_general_deterministic_value_intermediate_10}
    \EEt{ V^{\pi_{\theta_{t+1}}}(\mu) } - V^{\pi_{\theta_t}}(\mu) \\
    &\ge \sum_{s} \mu(s) \cdot d_{\mu}^{\pi_{\theta_t}}(s) \sum_{a^+ \in \gA_t^+(s)} \pi_{\theta_t}(a^+ | s)  \cdot \frac{ \left[ \exp\Big\{ \eta \cdot \frac{ Q^{\pi_{\theta_t}}(s,a^+) - V^{\pi_{\theta_t}}(s)}{\pi_{\theta_t}(a^+ | s)} \Big\} - 1 \right] \cdot \left( Q^{\pi_{\theta_{t}}}(s, a^+) - V^{\pi_{\theta_t}}(s) \right) }{ \exp\Big\{ \eta \cdot \frac{Q^{\pi_{\theta_t}}(s,a^+) - V^{\pi_{\theta_t}}(s)}{\pi_{\theta_t}(a^+ | s)} \Big\} + \frac{ 1 - \pi_{\theta_t}(a^+ | s) }{ \pi_{\theta_t}(a^+ | s) } } \\
    &\qquad + \sum_{s} \mu(s) \cdot d_{\mu}^{\pi_{\theta_t}}(s) \sum_{a^- \in \gA_t^+(s)} \pi_{\theta_t}(a^- | s)  \cdot \frac{ \left[ \exp\Big\{ \eta \cdot \frac{ Q^{\pi_{\theta_t}}(s,a^-) - V^{\pi_{\theta_t}}(s)}{\pi_{\theta_t}(a^- | s)} \Big\} - 1 \right] \cdot \left( Q^{\pi_{\theta_{t}}}(s, a^-) - V^{\pi_{\theta_t}}(s) \right) }{ \exp\Big\{ \eta \cdot \frac{Q^{\pi_{\theta_t}}(s,a^-) - V^{\pi_{\theta_t}}(s)}{\pi_{\theta_t}(a^- | s)} \Big\} + \frac{ 1 - \pi_{\theta_t}(a^- | s) }{ \pi_{\theta_t}(a^- | s) } }.
\end{align}
For any $a^+ \in \gA_t^+(t)$, using similar calculations in \cref{eq:non_uniform_lojasiewicz_stochastic_npg_value_baseline_special_deterministic_reward_intermediate_7}, we have,
\begin{align}
\MoveEqLeft
\label{eq:non_uniform_lojasiewicz_stochastic_npg_value_baseline_general_deterministic_value_intermediate_11}
    \frac{ \left[ \exp\Big\{ \eta \cdot \frac{ Q^{\pi_{\theta_t}}(s,a^+) - V^{\pi_{\theta_t}}(s)}{\pi_{\theta_t}(a^+ | s)} \Big\} - 1 \right] \cdot \left( Q^{\pi_{\theta_{t}}}(s, a^+) - V^{\pi_{\theta_t}}(s) \right) }{ \exp\Big\{ \eta \cdot \frac{Q^{\pi_{\theta_t}}(s,a^+) - V^{\pi_{\theta_t}}(s)}{\pi_{\theta_t}(a^+ | s)} \Big\} + \frac{ 1 - \pi_{\theta_t}(a^+ | s) }{ \pi_{\theta_t}(a^+ | s) } } \\
    &\ge \frac{\eta \cdot \left(  Q^{\pi_{\theta_{t}}}(s, a^+) - V^{\pi_{\theta_t}}(s)  \right)^2 }{ \eta \cdot \left(  Q^{\pi_{\theta_{t}}}(s, a^+) - V^{\pi_{\theta_t}}(s)  \right) + 1} \\
    &\ge \frac{\eta}{1 + \frac{\eta}{1-\gamma}} \cdot \left(  Q^{\pi_{\theta_{t}}}(s, a^+) - V^{\pi_{\theta_t}}(s)  \right)^2 \qquad \left( Q^{\pi_{\theta}}(s, a) \in [0,1/(1-\gamma)] \right) \\
    &\ge \frac{\eta}{1 + \frac{\eta}{1-\gamma}} \cdot \pi_{\theta_t}(a^+ | s) \cdot \left(  Q^{\pi_{\theta_{t}}}(s, a^+) - V^{\pi_{\theta_t}}(s)  \right)^2. \qquad \left( \pi_{\theta_t}(a^+ | s) \in (0,1) \right)
\end{align}
For any $a^- \in \gA_t^-(s)$, using similar calculations in \cref{eq:non_uniform_lojasiewicz_stochastic_npg_value_baseline_special_deterministic_reward_intermediate_8}, we have,
\begin{align}
\MoveEqLeft
\label{eq:non_uniform_lojasiewicz_stochastic_npg_value_baseline_general_deterministic_value_intermediate_12}
    \frac{ \left[ \exp\Big\{ \eta \cdot \frac{ Q^{\pi_{\theta_t}}(s,a^-) - V^{\pi_{\theta_t}}(s)}{\pi_{\theta_t}(a^- | s)} \Big\} - 1 \right] \cdot \left( Q^{\pi_{\theta_{t}}}(s, a^-) - V^{\pi_{\theta_t}}(s) \right) }{ \exp\Big\{ \eta \cdot \frac{Q^{\pi_{\theta_t}}(s,a^-) - V^{\pi_{\theta_t}}(s)}{\pi_{\theta_t}(a^- | s)} \Big\} + \frac{ 1 - \pi_{\theta_t}(a^- | s) }{ \pi_{\theta_t}(a^- | s) } } \\
    &\ge \frac{\eta \cdot \pi_{\theta_t}(a^- | s) \cdot \left( V^{\pi_{\theta_t}}(s) - Q^{\pi_{\theta_t}}(s,a^-)  \right)^2 }{ \eta \cdot \left( V^{\pi_{\theta_t}}(s) - Q^{\pi_{\theta_t}}(s,a^-)  \right)\cdot \big( 1 - \pi_{\theta_t}(a^- | s) \big) + \pi_{\theta_t}(a^- | s)} \\
    &\ge \frac{\eta \cdot \pi_{\theta_t}(a^- | s) \cdot \left( V^{\pi_{\theta_t}}(s) - Q^{\pi_{\theta_t}}(s,a^-)  \right)^2 }{ \eta \cdot \left( V^{\pi_{\theta_t}}(s) - Q^{\pi_{\theta_t}}(s,a^-)  \right) + 1} \qquad \left( \pi_{\theta_t}(a^- | s) \in (0, 1) \right) \\
    &\ge \frac{\eta}{1 + \frac{\eta}{1-\gamma}} \cdot \pi_{\theta_t}(a^- | s) \cdot \left( V^{\pi_{\theta_t}}(s) - Q^{\pi_{\theta_t}}(s,a^-)  \right)^2. \qquad \left(Q^{\pi_{\theta}}(s, a) \in [0,1/(1-\gamma)] \right)
\end{align}
Combining \cref{eq:non_uniform_lojasiewicz_stochastic_npg_value_baseline_general_deterministic_value_intermediate_10,eq:non_uniform_lojasiewicz_stochastic_npg_value_baseline_general_deterministic_value_intermediate_11,eq:non_uniform_lojasiewicz_stochastic_npg_value_baseline_general_deterministic_value_intermediate_12}, we have,
\begin{align}
\MoveEqLeft
\label{eq:non_uniform_lojasiewicz_stochastic_npg_value_baseline_general_deterministic_value_intermediate_13}
    \EEt{ V^{\pi_{\theta_{t+1}}}(\mu) } - V^{\pi_{\theta_t}}(\mu) \\
    &\ge \sum_{s} \mu(s) \cdot d_{\mu}^{\pi_{\theta_t}}(s) \sum_{a^+ \in \gA_t^+(s)} \pi_{\theta_t}(a^+ | s)  \cdot \frac{\eta}{1 + \frac{\eta}{1-\gamma}} \cdot \pi_{\theta_t}(a^+ | s) \cdot \left(  Q^{\pi_{\theta_{t}}}(s, a^+) - V^{\pi_{\theta_t}}(s)  \right)^2 \\
    &\qquad + \sum_{s} \mu(s) \cdot d_{\mu}^{\pi_{\theta_t}}(s) \sum_{a^- \in \gA_t^+(s)} \pi_{\theta_t}(a^- | s) \cdot \frac{\eta}{1 + \frac{\eta}{1-\gamma}} \cdot \pi_{\theta_t}(a^- | s) \cdot \left( V^{\pi_{\theta_t}}(s) - Q^{\pi_{\theta_t}}(s,a^-)  \right)^2 \\
    &= \frac{\eta}{1 + \frac{\eta}{1-\gamma}} \cdot \sum_{s} \mu(s) \cdot d_{\mu}^{\pi_{\theta_t}}(s) \cdot \sum_{a}{ \pi_{\theta_t}(a | s)^2 \cdot \left( Q^{\pi_{\theta_t}}(s,a) -  V^{\pi_{\theta_t}}(s) \right)^2 } \\
    &\ge \frac{\eta \cdot \left( 1-\gamma \right) }{1 + \eta} \cdot \sum_{s} \mu(s) \cdot d_{\mu}^{\pi_{\theta_t}}(s) \cdot \sum_{a}{ \pi_{\theta_t}(a | s)^2 \cdot \left( Q^{\pi_{\theta_t}}(s,a) -  V^{\pi_{\theta_t}}(s) \right)^2 }
\end{align}
Therefore, we have,
\begin{align}
\MoveEqLeft
\label{eq:non_uniform_lojasiewicz_stochastic_npg_value_baseline_general_deterministic_value_intermediate_14}
    \EEt{ V^{\pi_{\theta_{t+1}}}(\mu) } - V^{\pi_{\theta_t}}(\mu) \\
    &\ge \frac{\eta \cdot \left( 1-\gamma \right) }{1 + \eta} \cdot \sum_{s} \mu(s) \cdot d_{\mu}^{\pi_{\theta_t}}(s) \cdot \pi_{\theta_t}(a^*(s) | s)^2 \cdot \left( Q^{\pi_{\theta_t}}(s,a^*(s)) -  V^{\pi_{\theta_t}}(s) \right)^2 \qquad \left( \text{fewer terms} \right) \\
    &= \frac{\eta \cdot \left( 1-\gamma \right) }{1 + \eta} \cdot \sum_{s} \mu(s) \cdot \frac{ d_{\mu}^{\pi_{\theta_t}}(s) }{d_{\mu}^{\pi^*}(s)} \cdot d_{\mu}^{\pi^*}(s) \cdot \pi_{\theta_t}(a^*(s) | s)^2 \cdot \left( Q^{\pi_{\theta_t}}(s,a^*(s)) -  V^{\pi_{\theta_t}}(s) \right)^2 \\
    &\ge \frac{\eta \cdot \left( 1-\gamma \right) }{1 + \eta} \cdot \min_{s}{\mu(s) } \cdot \bigg\| \frac{d_{\mu}^{\pi^*}}{d_{\mu}^{\pi_{\theta_t}}} \bigg\|_\infty^{-1} \cdot \min_{s}{ \pi_{\theta_t}(a^*(s) | s)^2 } \cdot \sum_{s} d_{\mu}^{\pi^*}(s)  \cdot \left( Q^{\pi_{\theta_t}}(s,a^*(s)) -  V^{\pi_{\theta_t}}(s) \right)^2 \\
    &\ge \frac{\eta \cdot \left( 1-\gamma \right)^2 }{1 + \eta} \cdot \min_{s}{\mu(s) } \cdot \bigg\| \frac{d_{\mu}^{\pi^*}}{\mu} \bigg\|_\infty^{-1} \cdot \min_{s}{ \pi_{\theta_t}(a^*(s) | s)^2 }  \cdot \sum_{s} d_{\mu}^{\pi^*}(s) \cdot \left( Q^{\pi_{\theta_t}}(s,a^*(s)) -  V^{\pi_{\theta_t}}(s) \right)^2,
\end{align}
where $\min_{s}{\mu(s)} >0 $ is by \cref{assump:pos_init}, and the last inequality is according to \cref{eq:non_uniform_lojasiewicz_stochastic_npg_value_baseline_general_deterministic_value_intermediate_8},
\begin{align}
\MoveEqLeft
\label{eq:non_uniform_lojasiewicz_stochastic_npg_value_baseline_general_deterministic_value_intermediate_15}
    \bigg\| \frac{ d_{\mu}^{\pi^*} }{ d_{\mu}^{\pi_{\theta_t}} } \bigg\|_{\infty} \coloneqq \max_{s \in \gS} \frac{ d_{\mu}^{\pi^*}(s) }{ d_{\mu}^{\pi_{\theta_t}}(s) } \le \max_{s \in \gS} \frac{ d_{\mu}^{\pi^*}(s) }{ \left( 1 - \gamma \right) \cdot \mu(s) } = \frac{1}{ 1 - \gamma } \cdot \bigg\| \frac{ d_{\mu}^{\pi^*} }{ \mu } \bigg\|_{\infty}.
\end{align}
From \cref{eq:non_uniform_lojasiewicz_stochastic_npg_value_baseline_general_deterministic_value_intermediate_15}, since $d_{\mu}^{\pi^*}(s)^2 \in (0, 1)$, we have,
\begin{align}
\MoveEqLeft
\label{eq:non_uniform_lojasiewicz_stochastic_npg_value_baseline_general_deterministic_value_intermediate_16}
     \EEt{ V^{\pi_{\theta_{t+1}}}(\mu) } - V^{\pi_{\theta_t}}(\mu) \\
     &\ge \frac{\eta \cdot \left( 1-\gamma \right)^2 }{1 + \eta} \cdot \min_{s}{\mu(s) } \cdot \bigg\| \frac{d_{\mu}^{\pi^*}}{\mu} \bigg\|_\infty^{-1} \cdot \min_{s}{ \pi_{\theta_t}(a^*(s) | s)^2 }  \cdot \sum_{s} d_{\mu}^{\pi^*}(s)^2 \cdot \left( Q^{\pi_{\theta_t}}(s,a^*(s)) -  V^{\pi_{\theta_t}}(s) \right)^2 \\
     &\ge \frac{\eta \cdot \left( 1-\gamma \right)^2 }{1 + \eta} \cdot \min_{s}{\mu(s) } \cdot \bigg\| \frac{d_{\mu}^{\pi^*}}{\mu} \bigg\|_\infty^{-1} \cdot \frac{ \min_{s}{ \pi_{\theta_t}(a^*(s) | s)^2 } }{S} \cdot \bigg[ \sum_{s} d_{\mu}^{\pi^*}(s) \cdot \left| Q^{\pi_{\theta_t}}(s,a^*(s)) -  V^{\pi_{\theta_t}}(s) \right| \bigg]^2,
\end{align}
where the last inequality is by Cauchy–Schwarz. Note that,
\begin{align}
\MoveEqLeft
\label{eq:non_uniform_lojasiewicz_stochastic_npg_value_baseline_general_deterministic_value_intermediate_17}
     \sum_{s} d_{\mu}^{\pi^*}(s) \cdot \left| Q^{\pi_{\theta_t}}(s,a^*(s)) -  V^{\pi_{\theta_t}}(s) \right| \ge  \sum_{s} d_{\mu}^{\pi^*}(s) \cdot \left( Q^{\pi_{\theta_t}}(s,a^*(s)) -  V^{\pi_{\theta_t}}(s) \right) \\
     &= \sum_{s} d_{\mu}^{\pi^*}(s) \cdot \sum_{a} \left( \pi^*(a | s) - \pi_{\theta_t}(a | s) \right) \cdot Q^{\pi_{\theta_t}}(s,a) \\
     &= \left( 1 - \gamma \right) \cdot \big( V^{\pi^*}(\mu) - V^{\pi_{\theta_t}}(\mu) \big). \qquad \left( \text{by \cref{lem:performance_difference_general}} \right)
\end{align}
Combining \cref{eq:non_uniform_lojasiewicz_stochastic_npg_value_baseline_general_deterministic_value_intermediate_16,eq:non_uniform_lojasiewicz_stochastic_npg_value_baseline_general_deterministic_value_intermediate_17}, we have,
\begin{align}
\MoveEqLeft
     \EEt{ V^{\pi_{\theta_{t+1}}}(\mu) } - V^{\pi_{\theta_t}}(\mu) \\
     &\ge \frac{\eta \cdot \left( 1-\gamma \right)^4 }{1 + \eta} \cdot \min_{s}{\mu(s) } \cdot \bigg\| \frac{d_{\mu}^{\pi^*}}{\mu} \bigg\|_\infty^{-1} \cdot \frac{ \min_{s}{ \pi_{\theta_t}(a^*(s) | s)^2 } }{S} \cdot \big( V^{\pi^*}(\mu) - V^{\pi_{\theta_t}}(\mu) \big)^2,
\end{align}
thus finishing the proofs.
\end{proof}

\textbf{\cref{lem:non_vanishing_nl_coefficient_stochastic_npg_value_baseline_general}} (Non-vanishing stochastic N\L{} coefficient / ``automatic exploration'')\textbf{.}
Using \cref{alg:softmax_natural_pg_general_on_policy_stochastic_gradient_deterministic_value} with the same assumptions as \cref{lem:non_uniform_lojasiewicz_stochastic_npg_value_baseline_general}, with arbitrary initialization $\theta_1 \in \sR^{\gS \times \gA}$, we have,
\begin{align}
    c \coloneqq \inf_{t \ge 1, s \in \gS} \pi_{\theta_t}(a^*(s) | s) > 0, \qquad \text{a.s.}
\end{align}
\begin{proof}
Given any sampled state action pair $(s_t, a_t) = (s, i)$, we have,
\begin{align}
\MoveEqLeft
\label{eq:non_vanishing_nl_coefficient_stochastic_npg_value_baseline_general_deterministic_value_intermediate_1}
    \big[ V^{\pi_{\theta_{t+1}}}(\mu)  \ | \ s_t = s, a_t = i \big] - V^{\pi_{\theta_t}}(\mu) \\
    &= \frac{1}{1 - \gamma} \cdot \bigg[ \sum_{s^\prime} d_{\mu}^{\pi_{\theta_{t+1}}}(s^\prime) \cdot \sum_{a} \left( \pi_{\theta_{t+1}}(a | s^\prime) - \pi_{\theta_{t}}(a | s^\prime) \right) \cdot Q^{\pi_{\theta_{t}}}(s^\prime, a) \ \Big| \ s_t = s, a_t = i \bigg] \\
    &= \frac{1}{1 - \gamma} \cdot \bigg[ d_{\mu}^{\pi_{\theta_{t+1}}}(s) \cdot \sum_{a} \left( \pi_{\theta_{t+1}}(a | s) - \pi_{\theta_{t}}(a | s) \right) \cdot Q^{\pi_{\theta_{t}}}(s, a) \ \Big| \ a_t = i  \bigg] \\
    &= \frac{1}{1 - \gamma} \cdot d_{\mu}^{\pi_{\theta_{t+1}}}(s) \cdot \frac{ \left[ \exp\Big\{ \eta \cdot \frac{ Q^{\pi_{\theta_t}}(s,i) - V^{\pi_{\theta_t}}(s)}{\pi_{\theta_t}(i | s)} \Big\} - 1 \right] \cdot \left( Q^{\pi_{\theta_{t}}}(s, i) - V^{\pi_{\theta_t}}(s) \right) }{ \exp\Big\{ \eta \cdot \frac{Q^{\pi_{\theta_t}}(s,i) - V^{\pi_{\theta_t}}(s)}{\pi_{\theta_t}(i | s)} \Big\} + \frac{ 1 - \pi_{\theta_t}(i | s) }{ \pi_{\theta_t}(i | s) } } \\
    &\ge 0, \qquad \left( \text{by \cref{eq:non_uniform_lojasiewicz_stochastic_npg_value_baseline_general_deterministic_value_intermediate_5}} \right)
\end{align}
where the second equation is due to $\pi_{\theta_{t+1}}(a | s^\prime) = \pi_{\theta_t}(a | s^\prime)$ for all $s^\prime \not= s$ by \cref{alg:softmax_natural_pg_general_on_policy_stochastic_gradient_deterministic_value}. 

From \cref{eq:non_vanishing_nl_coefficient_stochastic_npg_value_baseline_general_deterministic_value_intermediate_1}, we have $V^{\pi_{\theta_{t+1}}}(\mu) \ge V^{\pi_{\theta_{t}}}(\mu)$ holds almost surely. According to the definition of $Q^{\pi}(s, a)$, we have,
\begin{align}
\label{eq:non_vanishing_nl_coefficient_stochastic_npg_value_baseline_general_deterministic_value_intermediate_2}
    Q^{\pi_{\theta_{t+1}}}(s, a) - Q^{\pi_{\theta_t}}(s, a) = \gamma \cdot \sum_{s^\prime}{ \gP( s^\prime | s, a) \cdot \left( V^{\pi_{\theta_{t+1}}}(s^\prime)  - V^{\pi_{\theta_{t}}}(s^\prime) \right) } \ge 0,
\end{align}
where the last inequality is by \cref{eq:non_vanishing_nl_coefficient_stochastic_npg_value_baseline_general_deterministic_value_intermediate_1}. Also note that $Q^{\pi}(s, a) \in [0, 1/(1-\gamma) ]$ since $r(s,a) \in [0,1]$ for all $(s, a) \in \gS \times \gA$. According to monotone convergence theorem, we have, for all $(s, a) \in \gS \times \gA$, the following exists,
\begin{align}
\label{eq:non_vanishing_nl_coefficient_stochastic_npg_value_baseline_general_deterministic_value_intermediate_3}
    Q^{\infty}(s, a) \coloneqq \lim_{t \to \infty} Q^{\pi_{\theta_t}}(s, a).
\end{align}
Also, define $V^{\infty}(s) \coloneqq \lim_{t \to \infty} V^{\pi_{\theta_t}}(s)$ for all $s \in \gS$.

For all state $s \in \gS$, given $i \in \gA$, define the following set $\gP(s, i)$ of ``generalized one-hot policy'' under state $s$,
\begin{align}
\label{eq:non_vanishing_nl_coefficient_stochastic_npg_value_baseline_general_deterministic_value_intermediate_4a}
    \gA(s, i) &\coloneqq \left\{ j \in \gA: Q^{\infty}(s, j) = Q^{\infty}(s, i) \right\}, \\
\label{eq:non_vanishing_nl_coefficient_stochastic_npg_value_baseline_general_deterministic_value_intermediate_4b}
    \gP(s, i) &\coloneqq \bigg\{ \pi(\cdot | s) \in \Delta(\gA): \sum_{j \in \gA(s, i)}{ \pi(j | s) } = 1 \bigg\}.
\end{align}
Similar to \cref{cl:approaching_generalized_one_hot_policy,cl:contradiction_approaching_sub_optimal_generalized_one_hot_policy} in the proofs for \cref{lem:non_vanishing_nl_coefficient_stochastic_npg_value_baseline_special}, we make the following two claims.
\begin{claim}
\label{cl:approaching_generalized_one_hot_policy_general_deterministic_value}
Almost surely, $\pi_{\theta_t}(\cdot | s)$ approaches one ``generalized one-hot policy'' under all state $s \in \gS$, i.e., 
there exists (a possibly random) $i \in \gA$, such that $\sum_{ j \in \gA(s, i) }{\pi_{\theta_t}(j | s) } \to 1$ as $t \to \infty$ almost surely as $t \to \infty$.
\end{claim}

\begin{claim}
\label{cl:contradiction_approaching_sub_optimal_generalized_one_hot_policy_general_deterministic_value}
Almost surely, $\pi_{\theta_t}(\cdot | s)$ cannot approach any ``sub-optimal generalized one-hot policies'' under all state $s \in \gS$, 
i.e., $i$ in the previous claim must be an optimal action. 
\end{claim}

From \cref{cl:contradiction_approaching_sub_optimal_generalized_one_hot_policy_general_deterministic_value}, it follows that $\sum_{ j \in \gA(a^*(s)) }{\pi_{\theta_t}(j | s) } \to 1$ almost surely under all state $s \in \gS$, as $t \to \infty$ and thus the policy sequence obtained
almost surely convergences to a globally optimal policy $\pi^*$.

\textbf{Proof of \cref{cl:approaching_generalized_one_hot_policy_general_deterministic_value}}.

Using similar arguments in \cref{eq:non_vanishing_nl_coefficient_stochastic_npg_value_baseline_special_claim_1_intermediate_3}, we have,
\begin{align}
\label{eq:non_vanishing_nl_coefficient_stochastic_npg_value_baseline_general_deterministic_value_claim_1_intermediate_1}
    \lim_{t \to \infty} \,\, \EEt{ V^{\pi_{\theta_{t+1}}}(\mu) } - V^{\pi_{\theta_t}}(\mu)  &= 0,\, 
    \qquad \text{a.s.}
\end{align}
According to \cref{eq:non_uniform_lojasiewicz_stochastic_npg_value_baseline_general_deterministic_value_intermediate_1,eq:non_uniform_lojasiewicz_stochastic_npg_value_baseline_general_deterministic_value_intermediate_13}, we have,
\begin{align}
\label{eq:non_vanishing_nl_coefficient_stochastic_npg_value_baseline_general_deterministic_value_claim_1_intermediate_2}
    \EEt{ V^{\pi_{\theta_{t+1}}}(\mu) } - V^{\pi_{\theta_t}}(\mu) &\ge \sum_{s} d_{\mu}^{\pi_{\theta_t}}(s) \cdot \mu(s) \cdot \frac{\eta \cdot \left( 1-\gamma \right) }{1 + \eta} \cdot \sum_{a} \pi_{\theta_t}(a | s)^2 \cdot \left( Q^{\pi_{\theta_t}}(s,a) -  V^{\pi_{\theta_t}}(s) \right)^2.
\end{align}
Since $d_{\mu}^{\pi_{\theta_t}}(s) \ge (1 - \gamma) \cdot \mu(s) > 0$ by \cref{eq:non_uniform_lojasiewicz_stochastic_npg_value_baseline_general_deterministic_value_intermediate_8,assump:pos_init}, we have, almost surely,
\begin{align}
\label{eq:non_vanishing_nl_coefficient_stochastic_npg_value_baseline_general_deterministic_value_claim_1_intermediate_3}
    \lim_{t \to \infty} \,\,  \sum_{s} \sum_{a} \pi_{\theta_t}(a | s)^2 \cdot \left( Q^{\pi_{\theta_t}}(s,a) -  V^{\pi_{\theta_t}}(s) \right)^2 = 0,
\end{align}
which implies that for all $s \in \gS$, almost surely,
\begin{align}
\label{eq:non_vanishing_nl_coefficient_stochastic_npg_value_baseline_general_deterministic_value_claim_1_intermediate_4}
    \lim_{t \to \infty} \,\, \sum_{a} \pi_{\theta_t}(a | s)^2 \cdot \left( Q^{\pi_{\theta_t}}(s,a) -  V^{\pi_{\theta_t}}(s) \right)^2 = 0.
\end{align}
Using similar arguments in \cref{eq:non_vanishing_nl_coefficient_stochastic_npg_value_baseline_special_claim_1_intermediate_5}, we have, for each state $s \in \gS$, there exists $i \in \gA$, such that,
\begin{align}
\label{eq:non_vanishing_nl_coefficient_stochastic_npg_value_baseline_general_deterministic_value_claim_1_intermediate_5}
    \lim_{t \to \infty}{ \sum_{j \in \gA(s, i)} \pi_{\theta_t}(j | s) } = 1, \qquad \text{a.s.},
\end{align}
which means $\pi_{\theta_t}(\cdot | s)$ a.s. approaches the ``generalized one-hot policy'' $\gP(s, i)$ in \cref{eq:non_vanishing_nl_coefficient_stochastic_npg_value_baseline_general_deterministic_value_intermediate_4b} as $t \to \infty$, finishing the proof of \cref{cl:approaching_generalized_one_hot_policy_general_deterministic_value}.

\textbf{Proof of \cref{cl:contradiction_approaching_sub_optimal_generalized_one_hot_policy_general_deterministic_value}}. The brief sketch of the proof is as follows:
By \cref{cl:approaching_generalized_one_hot_policy_general_deterministic_value}, for each state $s \in \gS$, there exists a (possibly random) $i\in \gA$ such that $\sum_{ j \in \gA(s, i) }{\pi_{\theta_t}(j | s) } \to 1$ almost surely, as $t \to \infty$.
If $i=a^*(s)$ almost surely, \cref{cl:contradiction_approaching_sub_optimal_generalized_one_hot_policy_general_deterministic_value} follows. 
Hence, it suffices to consider the event that $\{ i\not = a^*(s) \}$ for at least one state $s \in \gS$, and show that this event has zero probability mass. Hence, in the rest of the proof we assume that we are on the event when $i\not =a^*(s)$ for one state $s \in \gS$.

Since $i \not= a^*(s)$, there exists at least one ``good'' action $a^+ \in \gA$ such that $Q^{\infty}(s, a^+) > Q^{\infty}(s, i)$. The two cases are as follows.
\begin{description}[style=unboxed,leftmargin=0cm]
    \item[2a)] All ``good'' actions are sampled finitely many times as $t \to \infty$. \label{cl:contradiction_approaching_sub_optimal_generalized_one_hot_policy_general_deterministic_value:a}
    \item[2b)] At least one ``good'' action is sampled infinitely many times as $t \to \infty$.
\end{description}
In both cases, we show that $\sum_{ j \in \gA(s, i) }{ \exp\{ \theta_t(j | s) \} } < \infty$ as $t \to \infty$ (but for different reasons), \textcolor{red}{which is a contradiction with the assumption of $\sum_{ j \in \gA(s, i) }{\pi_{\theta_t}(j|s) } \to 1$ as $t \to \infty$}, given that a ``good'' action's parameter is almost surely lower bounded. Hence, $i\ne a^*(s)$ almost surely does not happen, which means that almost surely $i=a^*(s)$. Let
\begin{align}
\label{eq:non_vanishing_nl_coefficient_stochastic_npg_value_baseline_general_deterministic_value_claim_2_intermediate_1}
    I_t(s, a) = \begin{cases}
		1, & \text{if } (s_t, a_t) = (s,a)\,; \\
		0, & \text{otherwise}\,.
	\end{cases}
\end{align}
Define the following notations,
\begin{align}
\label{eq:non_vanishing_nl_coefficient_stochastic_npg_value_baseline_general_deterministic_value_claim_2_intermediate_2a}
    N_t(s,a) &\coloneqq \sum_{u=1}^{t}{ I_s(s,a) }, \\
\label{eq:non_vanishing_nl_coefficient_stochastic_npg_value_baseline_general_deterministic_value_claim_2_intermediate_2b}
    N_\infty(s,a) &\coloneqq \sum_{u=1}^{\infty}{ I_u(s,a) }.
\end{align}
Assume $\{ i\not = a^*(s) \}$ for at least one state $s \in \gS$, and $\sum_{ j \in \gA(s, i) }{\pi_{\theta_t}(j | s) } \to 1$ almost surely. Partition the action set $\gA$ under $s \in \gS$ into three parts using $V^{\infty}(s)$ as follows,
\begin{align}
\label{eq:non_vanishing_nl_coefficient_stochastic_npg_value_baseline_general_deterministic_value_claim_2_intermediate_3a}
    \gA(s, i) &\coloneqq \left\{ j \in \gA: Q^{\infty}(s, j) = Q^{\infty}(s, i) \right\}, \\
\label{eq:non_vanishing_nl_coefficient_stochastic_npg_value_baseline_general_deterministic_value_claim_2_intermediate_3b}
    \gA^+(s,i) &\coloneqq \left\{ a^+ \in \gA: Q^{\infty}(s, a^+) > Q^{\infty}(s, i) \right\}, \\
\label{eq:non_vanishing_nl_coefficient_stochastic_npg_value_baseline_general_deterministic_value_claim_2_intermediate_3c}
    \gA^-(s,i) &\coloneqq \left\{ a^- \in \gA: Q^{\infty}(s, a^-) < Q^{\infty}(s, i) \right\}.
\end{align}
Since $i \not= a^*(s)$, we have, $\gA^+(s,i)  \not= \emptyset$. Note that,
\begin{align}
\label{eq:non_vanishing_nl_coefficient_stochastic_npg_value_baseline_general_deterministic_value_claim_2_intermediate_4}
    \big| V^{\pi_{\theta_t}}(s) - Q^{\infty}(s, i) \big| &= \bigg| \sum_{ k \not\in \gA(s,i)}{ \pi_{\theta_t}(k | s) \cdot \left( Q^{\pi_{\theta_t}}(s,k) - Q^{\infty}(s, i) \right) } \\
    &\qquad + \sum_{\substack{j \not= i, \\ j \in \gA(s,i)}}{ \pi_{\theta_t}(j | s) \cdot \left( Q^{\pi_{\theta_t}}(s,j) - Q^{\infty}(s, i) \right) } \bigg| \\
    &\le \sum_{ k \not\in \gA(s,i)}{ \pi_{\theta_t}(k | s) \cdot \left| Q^{\pi_{\theta_t}}(s,k) - Q^{\infty}(s, i) \right| } \qquad \left( \text{triangle inequality} \right) \\
    &\qquad + \sum_{\substack{j \not= i, \\ j \in \gA(s,i)}}{ \pi_{\theta_t}(j | s) \cdot \left| Q^{\pi_{\theta_t}}(s,j) - Q^{\infty}(s, i) \right| }  \\
    &\le \frac{1}{1 - \gamma } \cdot \Big( \underbrace{  1 - \sum_{j \in \gA(s, i)   } \pi_{\theta_t}(j | s) }_{\to 0} \Big) + \sum_{\substack{j \not= i, \\ j \in \gA(s,i)}}{ \underbrace{ \left| Q^{\pi_{\theta_t}}(s,j) - Q^{\infty}(s, i) \right| }_{\to 0} },
\end{align}
which implies that $V^{\pi_{\theta_t}}(s) \to Q^{\infty}(s, i)$ as $t \to \infty$. Therefore, there exists $1\le \tau$, 
almost surely on $\{ i\ne a^* (s)\}$ 
 $\tau < \infty$ while we also have, for all $t \ge \tau$,
\begin{align}
\label{eq:non_vanishing_nl_coefficient_stochastic_npg_value_baseline_general_deterministic_value_claim_2_intermediate_5}
    Q^{\pi_{\theta_t}}(s,a^+) - c \ge V^{\pi_{\theta_t}}(s) \ge Q^{\pi_{\theta_t}}(s,a^-) + c, \\
\end{align}
for all $a^+ \in \gA^+(s,i)$, $a^- \in \gA^-(s,i)$, where $c > 0$. For all $t \ge \tau$, for any $a^+ \in \gA^+(s,a)$, we have, almost surely,
\begin{align}
\MoveEqLeft
\label{eq:non_vanishing_nl_coefficient_stochastic_npg_value_baseline_general_deterministic_value_claim_2_intermediate_6}
    \theta_{t+1}(s, a^+) = \theta_{t}( s, a^+) + \eta \cdot I_t(s, a^+) \cdot \frac{  Q^{\pi_{\theta_{t}}}(s, a^+) - V^{\pi_{\theta_t}}(s) }{ \pi_{\theta_t}(a^+ | s ) } \qquad \left( \text{by \cref{alg:softmax_natural_pg_general_on_policy_stochastic_gradient_deterministic_value}} \right) \\
    &\ge \theta_{t}( s, a^+) + \eta \cdot I_t(s, a^+) \cdot \frac{  c }{ \pi_{\theta_t}(a^+ | s ) } \qquad \left( \text{by \cref{eq:non_vanishing_nl_coefficient_stochastic_npg_value_baseline_general_deterministic_value_claim_2_intermediate_5}} \right) \\
    &\ge \theta_{t}( s, a^+) + \eta \cdot I_t(s, a^+) \cdot c \qquad \left( \pi_{\theta_t}(a^+ | s ) \in (0, 1) \right) \\
    &\ge \theta_{t}( s, a^+),
\end{align}
which implies that, almost surely,
\begin{align}
\label{eq:non_vanishing_nl_coefficient_stochastic_npg_value_baseline_general_deterministic_value_claim_2_intermediate_7}
    c_1 &\coloneqq \inf_{t \ge 1}{ \theta_t(s, a^+) } > - \infty.
\end{align}
On the other hand, for all $t \ge \tau$, for any $a^- \in \gA^-(s,a)$, we have, almost surely,
\begin{align}
\MoveEqLeft
\label{eq:non_vanishing_nl_coefficient_stochastic_npg_value_baseline_general_deterministic_value_claim_2_intermediate_8}
    \theta_{t+1}(s, a^-) = \theta_{t}( s, a^-) + \eta \cdot I_t(s, a^-) \cdot \frac{  Q^{\pi_{\theta_{t}}}(s, a^-) - V^{\pi_{\theta_t}}(s) }{ \pi_{\theta_t}(a^- | s ) } \qquad \left( \text{by \cref{alg:softmax_natural_pg_general_on_policy_stochastic_gradient_deterministic_value}} \right) \\
    &\le \theta_{t}( s, a^-) - \eta \cdot I_t(s, a^-) \cdot \frac{  c }{ \pi_{\theta_t}(a^- | s ) } \qquad \left( \text{by \cref{eq:non_vanishing_nl_coefficient_stochastic_npg_value_baseline_general_deterministic_value_claim_2_intermediate_5}} \right) \\
    &\le \theta_{t}( s, a^-) - \eta \cdot I_t(s, a^-) \cdot c \qquad \left( \pi_{\theta_t}(a^- | s ) \in (0, 1) \right) \\
    &\le \theta_{t}( s, a^-),
\end{align}
which implies that, almost surely,
\begin{align}
\label{eq:non_vanishing_nl_coefficient_stochastic_npg_value_baseline_general_deterministic_value_claim_2_intermediate_9}
    c_2 \coloneqq \sup_{t \ge 1}{ \theta_t(s, a^-)} < \infty.
\end{align}
\textbf{First case. 2a).} Consider the event,
\begin{align}
\label{eq:non_vanishing_nl_coefficient_stochastic_npg_value_baseline_general_deterministic_value_claim_2_case_1_intermediate_1}
    \gE_0 \coloneqq \bigcap\limits_{a^+ \in \gA^+(s, i)} 
    \underbrace{
    \left\{ N_\infty(s,a^+) < \infty \right\}}_{\mathcal{E}_0(s,a^+)},
\end{align}
i.e., any ``good'' action $a^+ \in \gA^+(s,i)$ has finitely many updates as $t \to \infty$. Using the extended Borel-Cantelli lemma (\cref{lem:ebc}), we have, almost surely,
\begin{align}
\label{eq:non_vanishing_nl_coefficient_stochastic_npg_value_baseline_general_deterministic_value_claim_2_case_1_intermediate_2}
    \Big\{ \sum_{t \ge 1} \pi_{\theta_t}(a^+|s)<\infty \Big\} = \left\{N_\infty(s, a^+)<\infty\right\}.
\end{align}
Next, we have, almost surely,
\begin{align}
\MoveEqLeft
    1 - \sum_{ j \in \gA(s,i) }{\pi_{\theta_t}(j | s) } = \frac{ \sum_{a^+ \in \gA^+(s,i)}{ e^{\theta_t(s,a^+)} } + \sum_{a^- \in \gA^-(s,i)}{ e^{\theta_t(s,a^-)} } }{ \sum_{a \in \gA}{ e^{\theta_t(s,a)} } } \\
    &\le \frac{ \sum_{a^+ \in \gA^+(s,i)}{ e^{\theta_t(s,a^+)} } + \sum_{a^- \in \gA^-(s,i)}{ e^{c_2} } }{ \sum_{a \in \gA}{ e^{\theta_t(s,a)} } } \qquad \left( \text{by \cref{eq:non_vanishing_nl_coefficient_stochastic_npg_value_baseline_general_deterministic_value_claim_2_intermediate_9}} \right) \\
    &= \frac{ \sum_{a^+ \in \gA^+(s,i)}{ e^{\theta_t(s,a^+)} } + e^{c_2 - c_1} \cdot \frac{ |\gA^-(s,i)| }{ | \gA^+(s,i) | } \cdot | \gA^+(s,i) | \cdot e^{c_1} }{ \sum_{a \in \gA}{ e^{\theta_t(s,a)} } } \\
    &\le \frac{ \sum_{a^+ \in \gA^+(s,i)}{ e^{\theta_t(s,a^+)} } + e^{c_2 - c_1} \cdot \frac{ |\gA^-(s,i)| }{ | \gA^+(s,i) | } \cdot \sum_{a^+ \in \gA^+(s,i)}{ e^{\theta_t(s,a^+)} } }{ \sum_{a \in \gA}{ e^{\theta_t(s,a)} } } \qquad \left( \text{by \cref{eq:non_vanishing_nl_coefficient_stochastic_npg_value_baseline_general_deterministic_value_claim_2_intermediate_7}} \right) \\
    &= \frac{ \sum_{a^+ \in \gA^+(s,i)}{ e^{\theta_t(s,a^+)} }  }{ \sum_{a \in \gA}{ e^{\theta_t(s,a)} } } \cdot \left( 1 + e^{c_2 - c_1} \cdot \frac{ |\gA^-(s,i)| }{ | \gA^+(s,i) | } \right) \\
    &= \left( 1 + e^{c_2 - c_1} \cdot \frac{ |\gA^-(s,i)| }{ | \gA^+(s,i) | } \right) \cdot \sum_{a^+ \in \gA^+(s,i)}{ \pi_{\theta_t}(a^+|s) } \,.
\end{align}
Define
\begin{align}
    q_t \coloneqq \sum_{a^+ \in \gA^+(s,i)}{ \pi_{\theta_t}(a^+|s) }.
\end{align}
According to \cref{eq:non_vanishing_nl_coefficient_stochastic_npg_value_baseline_general_deterministic_value_claim_2_case_1_intermediate_2}, we have, on $\gE_0$, almost surely,
\begin{align}
    \sum_{t=1}^{\infty}{q_t} < \infty.
\end{align}
On the other hand, according to the assumption of $\sum_{ j \in \gA(s, i) }{\pi_{\theta_t}(j | s) } \to 1$, there exists at least one $j \in \gA(s, i)$, such that almost surely, for all $t \ge \tau$, $\pi_{\theta_t}(j | s) > c^\prime$ for some $c^\prime >0$. We have, 
\begin{align}
\MoveEqLeft
    \theta_{t+1}(s, j) = \theta_{t}( s, j) + \eta \cdot I_t(s, j) \cdot \frac{  Q^{\pi_{\theta_{t}}}(s, j) - V^{\pi_{\theta_t}}(s) }{ \pi_{\theta_t}(j | s ) } \qquad \left( \text{by \cref{alg:softmax_natural_pg_general_on_policy_stochastic_gradient_deterministic_value}} \right) \\
    &\le \theta_{t}( s, j) + \eta \cdot I_t(s, j) \cdot \frac{ 1 - \sum_{ j \in \gA(s,i) }{\pi_{\theta_t}(j | s) } }{ \pi_{\theta_t}(j | s ) } \cdot \frac{1}{1 - \gamma} \\
    &\le \theta_{t}( s, j) + \eta \cdot I_t(s, j) \cdot \frac{ 1 - \sum_{ j \in \gA(s,i) }{\pi_{\theta_t}(j | s) } }{ c^\prime } \cdot \frac{1}{1 - \gamma}, \qquad \left( \pi_{\theta_t}(j | s) > c^\prime \right)
\end{align}
which implies that, for $C \coloneqq \max_{t \in [1, \tau]}{ \theta_t(s, j)}$, we have
\begin{align}
    \sup_{t \ge 1}{\theta_t(s, j)} \le C + \frac{\eta \cdot \left( 1 + e^{c_2 - c_1} \cdot \frac{ |\gA^-(s,i)| }{ | \gA^+(s,i) | } \right) }{ \left( 1 - \gamma \right) \cdot c^\prime} \cdot \sum_{t=\tau}^{\infty} \sum_{a^+ \in \gA^+(s,i)}{ \pi_{\theta_t}(a^+|s) } < \infty.
\end{align}
Following calculations in \cref{eq:non_vanishing_nl_coefficient_stochastic_npg_value_baseline_special_claim_2_case_1_intermediate_22}, almost surely on $\gE^\prime \coloneqq \gE_0\cap \{ i\ne a^*(s) \}$, we have, $\sum_{j \in \gA(s, i)} \pi_{\theta_t}(j | s) \not\to 1$, which is a contradiction with the assumption,
showing that $\mathbb{P}(\gE^\prime)=0$.

\textbf{Second case. 2b).} Consider the complement $\gE_0^c$ of $\gE_0$, where $\gE_0$ is by \cref{eq:non_vanishing_nl_coefficient_stochastic_npg_value_baseline_general_deterministic_value_claim_2_case_1_intermediate_1}.
We now show that also $\PP(\gE^{\prime\prime})=0$ where $\gE^{\prime\prime}=\gE_0^c \cap \{ i\ne a^*(s) \}$.

Pick $a^+ \in \gA^+(s, i)$, such that $\mathbb{P}{\left( N_\infty(s, a^+) = \infty \right) } > 0$. 
On event $\gE_{\infty}(s, a^+) \coloneqq \{ N_\infty(s, a^+) = \infty \}$, accoding to \cref{eq:non_vanishing_nl_coefficient_stochastic_npg_value_baseline_general_deterministic_value_claim_2_intermediate_6}, we have, almost surely,
\begin{align}
\label{eq:non_vanishing_nl_coefficient_stochastic_npg_value_baseline_general_deterministic_value_claim_2_case_2_intermediate_1}
    c_3 \coloneqq \lim_{t \to \infty}{ \theta_t(s, a^+) } = \infty.
\end{align}
Therefore, we have, for all $t \ge \tau$,
\begin{align}
\MoveEqLeft
    V^{\pi_{\theta_t}}(s) = Q^{\pi_{\theta_t}}(s,i) + \sum_{\substack{j \not= i, \\ j \in \gA(s,i)}}{ \pi_{\theta_t}(j|s) \cdot \underbrace{ \left( Q^{\pi_{\theta_t}}(s,j) - Q^{\pi_{\theta_t}}(s,i)\right) }_{\to 0} } \\
    &\qquad + \sum_{a^- \in \gA^-(s,i)} \pi_{\theta_t}(a^-|s) \cdot \underbrace{  \left( Q^{\pi_{\theta_t}}(s,a^-) - Q^{\pi_{\theta_t}}(s,i)\right)}_{ < 0} \\
    &\qquad + \sum_{\tilde{a}^+ \in \gA^+(s,i)} \pi_{\theta_t}(\tilde{a}^+|s) \cdot \underbrace{ \left( Q^{\pi_{\theta_t}}(s,\tilde{a}^+) - Q^{\pi_{\theta_t}}(s,i)\right)}_{ > 0} \\
    &\ge Q^{\pi_{\theta_t}}(s,i) + \sum_{\substack{j \not= i, \\ j \in \gA(s,i)}}{ \pi_{\theta_t}(j|s) \cdot \left( Q^{\pi_{\theta_t}}(s,j) - Q^{\pi_{\theta_t}}(s,i)\right)  } \\
    &\qquad + \pi_{\theta_t}(a^+|s) \cdot \bigg[ \left( Q^{\pi_{\theta_t}}(s,a^+) - Q^{\pi_{\theta_t}}(s,i)\right) -  \sum_{a^- \in \gA^-(s,i)} \frac{Q^{\pi_{\theta_t}}(s,i)-Q^{\pi_{\theta_t}}(s,a^-) }{ \exp\{ \theta_t(s, a^+) - \theta_t(s, a^-) \}} \bigg].
\end{align}
According to \cref{eq:non_vanishing_nl_coefficient_stochastic_npg_value_baseline_general_deterministic_value_claim_2_intermediate_9,eq:non_vanishing_nl_coefficient_stochastic_npg_value_baseline_general_deterministic_value_claim_2_case_2_intermediate_1}, $\theta_t(s, a^+) - \theta_t(s, a^-) \to \infty$, which implies that, on event $\gE_\infty(s, a^+)$, almost surely, for all $t \ge \tau$,
\begin{align}
    V^{\pi_{\theta_t}}(s) > Q^{\pi_{\theta_t}}(s,i) +  \sum_{\substack{j \not= i, \\ j \in \gA(s,i)}}{ \pi_{\theta_t}(j|s) \cdot \left( Q^{\pi_{\theta_t}}(s,j) - Q^{\pi_{\theta_t}}(s,i)\right)  },
\end{align}
which implies that,
\begin{align}
    \sum_{k \in \gA(s,i)}  \pi_{\theta_t}(k|s) \cdot V^{\pi_{\theta_t}}(s) &> \sum_{k \in \gA(s,i)}  \pi_{\theta_t}(k|s) \cdot Q^{\pi_{\theta_t}}(s,k) \\
    &\qquad +  \sum_{k \in \gA(s,i)}  \pi_{\theta_t}(k|s) \cdot \sum_{\substack{j \not= k, \\ j \in \gA(s,i)}}{ \pi_{\theta_t}(j|s) \cdot \left( Q^{\pi_{\theta_t}}(s,j) - Q^{\pi_{\theta_t}}(s,k)\right)  }\\
    &= \sum_{k \in \gA(s,i)}  \pi_{\theta_t}(k|s) \cdot Q^{\pi_{\theta_t}}(s,k).
\end{align}
For all $t \ge \tau$, we have,
\begin{align}
    \theta_{t+1}(s, i) &= \theta_{t}( s, i) + \eta \cdot I_t(s, i) \cdot \frac{  Q^{\pi_{\theta_{t}}}(s, i) - V^{\pi_{\theta_t}}(s) }{ \pi_{\theta_t}(i | s ) } \qquad \left( \text{by \cref{alg:softmax_natural_pg_general_on_policy_stochastic_gradient_deterministic_value}} \right) \\
    &\le \theta_{t}( s, i),
\end{align}
which implies that,
\begin{align}
    \sup_{t \ge 1}{\theta_{t}(s, i)} < \infty.
\end{align}
Following calculations in \cref{eq:non_vanishing_nl_coefficient_stochastic_npg_value_baseline_special_claim_2_case_2_intermediate_10}, almost surely on $\gE^{\prime\prime}=\gE_0^c \cap \{ i\ne a^*(s) \}$, we have, $\sum_{j \in \gA(s, i)} \pi_{\theta_t}(j | s) \not\to 1$, which is a contradiction with the assumption,
showing that $\mathbb{P}(\gE^{\prime\prime})=0$.
\end{proof}

\textbf{\cref{thm:almost_sure_convergence_rate_stochastic_npg_general_value_baseline}} (Almost sure global convergence rate) \textbf{.}
Using \cref{alg:softmax_natural_pg_general_on_policy_stochastic_gradient_deterministic_value} with any initialization $\theta_1 \in \sR^K$,  under the same assumptions as Lemmas \ref{lem:non_uniform_lojasiewicz_stochastic_npg_value_baseline_general}, we have, for all $t \ge 1$,
\begin{align}
    \EE{ V^*(\mu) - V^{\pi_{\theta_t}}(\mu) } \le \frac{1 + \eta}{\eta \cdot \left( 1 - \gamma \right)^4 \cdot \min_{s}{\mu(s) } }  \cdot \bigg\| \frac{d_{\mu}^{\pi^*}}{\mu} \bigg\|_\infty \cdot \frac{S}{\EE{ c^2 }} \cdot \frac{1}{t},& \qquad \text{and} \\
    \limsup_{t \ge 1} \bigg\{ \frac{\eta \cdot \left( 1 - \gamma \right)^4 \cdot \min_{s}{\mu(s) } }{1 + \eta}  \cdot \bigg\| \frac{d_{\mu}^{\pi^*}}{\mu} \bigg\|_\infty^{-1} \cdot \frac{c^2 \cdot t}{S} \cdot \left( V^*(\mu) - V^{\pi_{\theta_t}}(\mu) \right) \bigg\} < \infty,& \qquad \text{a.s.},
\end{align}
where we use $\EEt{\cdot}$ to denote $\EEt{\cdot | \gF_t}$ for brevity, and $\gF_t$ is the $\sigma$-algebra generated by $(s_1, a_1), (s_2, a_2), \dots, (s_{t-1}, a_{t-1})$, $\pi^*$ is the global optimal policy, $S$ is the state number, $\min_{s}{\mu(s)} > 0$ by \cref{assump:pos_init}, and $c \coloneqq \inf_{t \ge 1, s \in \gS} \pi_{\theta_t}(a^*(s) | s) > 0$ is from \cref{lem:non_vanishing_nl_coefficient_stochastic_npg_value_baseline_general}.
\begin{proof}
\textbf{First part. } According to \cref{lem:non_uniform_lojasiewicz_stochastic_npg_value_baseline_general}, we have,
\begin{align}
\MoveEqLeft
    \EEt{ V^{\pi_{\theta_{t+1}}}(\mu) } - V^{\pi_{\theta_t}}(\mu) \\
    &\ge \frac{\eta \cdot \left( 1 - \gamma \right)^4 \cdot \min_{s}{\mu(s) } }{1 + \eta}  \cdot \bigg\| \frac{d_{\mu}^{\pi^*}}{\mu} \bigg\|_\infty^{-1} \cdot \frac{ \min_{s}{ \pi_{\theta_t}(a^*(s) | s)^2 } }{S} \cdot \big( V^{\pi^*}(\mu) - V^{\pi_{\theta_t}}(\mu) \big)^2 \\
    &\ge \frac{\eta \cdot \left( 1 - \gamma \right)^4 \cdot \min_{s}{\mu(s) } }{1 + \eta}  \cdot \bigg\| \frac{d_{\mu}^{\pi^*}}{\mu} \bigg\|_\infty^{-1} \cdot \frac{ \inf_{t \ge 1, s \in \gS}{ \pi_{\theta_t}(a^*(s) | s)^2 } }{S} \cdot \big( V^{\pi^*}(\mu) - V^{\pi_{\theta_t}}(\mu) \big)^2 \\
    &= \frac{\eta \cdot \left( 1 - \gamma \right)^4 \cdot \min_{s}{\mu(s) } }{1 + \eta}  \cdot \bigg\| \frac{d_{\mu}^{\pi^*}}{\mu} \bigg\|_\infty^{-1} \cdot \frac{ c^2 }{S} \cdot \big( V^{\pi^*}(\mu) - V^{\pi_{\theta_t}}(\mu) \big)^2,
\end{align}
where $c \coloneqq \inf_{t \ge 1, s \in \gS} \pi_{\theta_t}(a^*(s) | s) > 0$ according to \cref{lem:non_vanishing_nl_coefficient_stochastic_npg_value_baseline_general}. Let $\delta(\theta_t) \coloneqq V^*(\mu) - V^{\pi_{\theta_t}}(\mu)$ denote the sub-optimality gap. Using similar calculations in \cref{thm:almost_sure_convergence_rate_stochastic_npg_special_value_baseline}, we have, for all $t \ge 1$,
\begin{align}
    \EE{ V^*(\mu) - V^{\pi_{\theta_t}}(\mu) } = \expectation{ [ \delta(\theta_t) ]} \le \frac{1 + \eta}{\eta \cdot \left( 1 - \gamma \right)^4 \cdot \min_{s}{\mu(s) } }  \cdot \bigg\| \frac{d_{\mu}^{\pi^*}}{\mu} \bigg\|_\infty \cdot \frac{S}{\EE{ c^2 }} \cdot \frac{1}{t}.
\end{align}
\textbf{Second part. } The result follows from \cref{lem:problemma} by choosing $X_t = V^*(\mu) - V^{\pi_{\theta_t}}(\mu)$
and $f(t) = \frac{\eta \cdot \left( 1 - \gamma \right)^4 \cdot \min_{s}{\mu(s) } }{1 + \eta}  \cdot \Big\| \frac{d_{\mu}^{\pi^*}}{\mu} \Big\|_\infty^{-1} \cdot \frac{\EE{ c^2 }}{S}  \cdot t$.
\end{proof}

\section{Proofs for Understanding Baselines}

\textbf{\cref{prop:softmax_natural_pg_unbiased}} (Unbiasedness of NPG)\textbf{.}
For NPG with and without a state value baseline, corresponding to \cref{update_rule:softmax_natural_pg_special_on_policy_stochastic_gradient,update_rule:softmax_natural_pg_special_on_policy_stochastic_gradient_value_baseline} respectively,
we have 
$\expectation_{a_t \sim \pi_{\theta_t}(\cdot)}{ \left[ \hat{r}_t \right] } = \expectation_{a_t \sim \pi_{\theta_t}(\cdot)}{ [ \hat{r}_t - \hat{b}_t ] } = r$. 
\begin{proof}
\textbf{First part.} $\expectation_{a_t \sim \pi_{\theta_t}(\cdot)}{ \left[ \hat{r}_t \right] } = r$.

According to \cref{def:simplified_on_policy_importance_sampling}, we have, for all $i \in [K]$,
\begin{align}
    \expectation_{a_t \sim \pi_{\theta_t}(\cdot)}{ \left[ \hat{r}_t(i) \right] } &= \sum_{a \in [K]}{ \PP{(a_t = a )} \cdot \hat{r}_t(i) } \\
    &= \sum_{a \in [K]}{ \pi_{\theta_t}(a) \cdot \frac{ \sI\left\{ a = i \right\} }{ \pi_{\theta_t}(i) } \cdot r(i) } = r(i). \qquad \left( a_t \sim \pi_{\theta_t}(\cdot) \right)
\end{align}

\textbf{Second part.} $\expectation_{a_t \sim \pi_{\theta_t}(\cdot)}{ [ \hat{r}_t - \hat{b}_t ] } = r$.
According to \cref{def:simplified_on_policy_importance_sampling}, we have, for all $i \in [K]$,
\begin{align}
    \expectation_{a_t \sim \pi_{\theta_t}(\cdot)}{ [ \hat{r}_t(i) - \hat{b}_t(i) ] } &= \sum_{a \in [K]}{ \pi_{\theta_t}(a) \cdot \left[ \frac{ \sI\left\{ a = i \right\} }{ \pi_{\theta_t}(i) } \cdot \left( r(i) - \pi_{\theta_t}^\top r \right) + \pi_{\theta_t}^\top r \right] } \qquad \left( \text{by \cref{update_rule:softmax_natural_pg_special_on_policy_stochastic_gradient_value_baseline}} \right) \\
    &= r(i) - \pi_{\theta_t}^\top r  + \pi_{\theta_t}^\top r \\
    &= r(i). \qedhere
\end{align}
\end{proof}

\textbf{\cref{prop:softmax_natural_pg_variances}} (Unboundedness of NPG)\textbf{.}
For NPG without a baseline, \cref{update_rule:softmax_natural_pg_special_on_policy_stochastic_gradient},
we have $\expectation_{a_t \sim \pi_{\theta_t}(\cdot)}{ \left\| \hat{r}_t \right\|_2^2 } = \sum_{a \in [K]}{ \frac{ r(a)^2 }{ \pi_{\theta_t}(a) }  }$. 
For NPG with a state value baseline, \cref{update_rule:softmax_natural_pg_special_on_policy_stochastic_gradient_value_baseline},
we have $\expectation_{a_t \sim \pi_{\theta_t}(\cdot)}{ \| \hat{r}_t - \hat{b}_t \|_2^2 } = \sum_{a \in [K]}{ \frac{ ( r(a) - \pi_{\theta_t}^\top r )^2 }{ \pi_{\theta_t}(a) } } - K \cdot ( \pi_{\theta_t}^\top r)^2 + 2 \cdot ( \pi_{\theta_t}^\top r) \cdot ( r^\top \rvone )$.
\begin{proof}
\textbf{First part.} $\expectation_{a_t \sim \pi_{\theta_t}(\cdot)}{ \left\| \hat{r}_t \right\|_2^2 } = \sum_{a \in [K]}{ \frac{ r(a)^2 }{ \pi_{\theta_t}(a) }  }$. 

According to \cref{def:simplified_on_policy_importance_sampling}, we have,
\begin{align}
    \left\| \hat{r}_t \right\|_2^2 = \sum_{i}{ \hat{r}_t(i)^2 } = \sum_{i}{ \frac{ \left( \sI\left\{ a_t = i \right\} \right)^2 }{ \pi_{\theta_t}(i)^2 } \cdot r(i)^2 } = \sum_{i}{ \frac{ \sI\left\{ a_t = i \right\} }{ \pi_{\theta_t}(i)^2 } \cdot r(i)^2 }.
\end{align}
Taking expectation, we have,
\begin{align}
    \expectation_{a_t \sim \pi_{\theta_t}(\cdot)}{ \left\| \hat{r}_t \right\|_2^2 } &= \sum_{a \in [K]}{ \pi_{\theta_t}(a) \cdot \sum_{i}{ \frac{ \sI\left\{ a = i \right\} }{ \pi_{\theta_t}(i)^2 } \cdot r(i)^2 } } \\
    &= \sum_{a \in [K]}{ \pi_{\theta_t}(a) \cdot \frac{1}{ \pi_{\theta_t}(a)^2 } \cdot r(a)^2 } \\
    &= \sum_{a \in [K]}{ \frac{ r(a)^2 }{ \pi_{\theta_t}(a) }  }.
\end{align}

\textbf{Second part.} $\expectation_{a_t \sim \pi_{\theta_t}(\cdot)}{ \| \hat{r}_t - \hat{b}_t \|_2^2 } = \sum_{a \in [K]}{ \frac{ ( r(a) - \pi_{\theta_t}^\top r )^2 }{ \pi_{\theta_t}(a) } } - K \cdot ( \pi_{\theta_t}^\top r)^2 + 2 \cdot ( \pi_{\theta_t}^\top r) \cdot ( r^\top \rvone )$.

According to \cref{def:simplified_on_policy_importance_sampling}, we have,
\begin{align}
\MoveEqLeft
    \big\| \hat{r}_t - \hat{b}_t \big\|_2^2 = \sum_{i}{ \left( \hat{r}_t(i) - \hat{b}_t(i) \right)^2 } \\
    &= \sum_{i}{ \left[ \frac{ \sI\left\{ a_t = i \right\} }{ \pi_{\theta_t}(i) } \cdot \left( r(i) - \pi_{\theta_t}^\top r \right) + \pi_{\theta_t}^\top r \right]^2 } \\
    &= \sum_{i}{ \frac{ \left( \sI\left\{ a_t = i \right\} \right)^2 }{ \pi_{\theta_t}(i)^2 } \cdot \left( r(i) - \pi_{\theta_t}^\top r \right)^2 } + \sum_{i} \left( \pi_{\theta_t}^\top r \right)^2 + 2 \cdot \sum_{i}{ \frac{ \sI\left\{ a_t = i \right\} }{ \pi_{\theta_t}(i) } \cdot \left( r(i) - \pi_{\theta_t}^\top r \right) \cdot \left( \pi_{\theta_t}^\top r \right) } \\
    &= \sum_{i}{ \frac{ \sI\left\{ a_t = i \right\} }{ \pi_{\theta_t}(i)^2 } \cdot \left( r(i) - \pi_{\theta_t}^\top r \right)^2 } + K \cdot \left( \pi_{\theta_t}^\top r \right)^2 + 2 \cdot \sum_{i}{ \frac{ \sI\left\{ a_t = i \right\} }{ \pi_{\theta_t}(i) } \cdot \left( r(i) - \pi_{\theta_t}^\top r \right) \cdot \left( \pi_{\theta_t}^\top r \right) }.
\end{align}
Taking expectation, we have,
\begin{align}
\MoveEqLeft
    \expectation_{a_t \sim \pi_{\theta_t}(\cdot)}{ \big\| \hat{r}_t - \hat{b}_t \big\|_2^2 } = \sum_{a \in [K]}{ \pi_{\theta_t}(a) \cdot \sum_{i}{ \frac{ \sI\left\{ a = i \right\} }{ \pi_{\theta_t}(i)^2 } \cdot \left( r(i) - \pi_{\theta_t}^\top r \right)^2 } } \\
    &\qquad + \sum_{a \in [K]}{ \pi_{\theta_t}(a) \cdot K \cdot \left( \pi_{\theta_t}^\top r \right)^2 } + 2 \cdot \left( \pi_{\theta_t}^\top r \right) \cdot \sum_{a \in [K]}{ \pi_{\theta_t}(a) \cdot \sum_{i}{ \frac{ \sI\left\{ a_t = i \right\} }{ \pi_{\theta_t}(i) } \cdot \left( r(i) - \pi_{\theta_t}^\top r \right)  } } \\
    &= \sum_{a \in [K]}{ \pi_{\theta_t}(a) \cdot \frac{ 1 }{ \pi_{\theta_t}(a)^2 } \cdot \left( r(a) - \pi_{\theta_t}^\top r \right)^2 } \\
    &\qquad + K \cdot \left( \pi_{\theta_t}^\top r \right)^2 + 2 \cdot \left( \pi_{\theta_t}^\top r \right) \cdot \sum_{a \in [K]}{ \pi_{\theta_t}(a) \cdot \frac{ 1 }{ \pi_{\theta_t}(a) } \cdot \left( r(a) - \pi_{\theta_t}^\top r \right) } \\
    &= \sum_{a \in [K]}{ \frac{ ( r(a) - \pi_{\theta_t}^\top r )^2 }{ \pi_{\theta_t}(a) } } - K \cdot ( \pi_{\theta_t}^\top r)^2 + 2 \cdot ( \pi_{\theta_t}^\top r) \cdot ( r^\top \rvone ). \qedhere
\end{align}
\end{proof}

\textbf{\cref{lem:positive_infinite_product}} (Bad sampling)\textbf{.}
Let $\pi_{\theta_t}(a) \in (0, 1)$ be the probability of sampling action $a$ using online sampling $a_t \sim \pi_{\theta_t}(\cdot)$, for all $t \ge 1$. If $1 - \pi_{\theta_t}(a) \in O(1/t^{1+ \epsilon})$, where $\epsilon > 0$, then $\prod_{t=1}^{\infty}{ \pi_{\theta_t}(a) } > 0$.
\begin{proof}
According to \cref{lem:infinite_product_infinite_sum}, we have, for a sequence $u_t \in (0, 1)$ for all $t \ge 1$, if $\sum_{t=1}^{\infty}{ u_t } < \infty$, then $\prod_{t=1}^{\infty}{\left( 1 - u_t \right) } > 0$.

Let $u_t = 1 - \pi_{\theta_t}(a) \in (0, 1)$ according to the softmax parameterization. If $1 - \pi_{\theta_t}(a) \in O(1/t^{1+\epsilon})$, such as $1 - \pi_{\theta_t}(a) \in \Theta(1/t^\alpha)$ where $a \in (1, \infty)$, then we have, for all $C > 0$,
\begin{align}
\MoveEqLeft
    \sum_{t=1}^{\infty}{ u_t } = \sum_{t=1}^{\infty}{ \left( 1 - \pi_{\theta_t}(a) \right) } \\
    &= \sum_{t=1}^{\infty}{ \frac{C}{t^\alpha} } \\
    &\le C \cdot \left( 1 + \int_{t = 1}^{\infty}{ \frac{1}{t^\alpha} dt } \right) \\
    &= \frac{C \cdot \alpha}{ \alpha - 1},
\end{align}
or if $1 - \pi_{\theta_t}(a) \in \Theta(e^{- c \cdot t})$ where $c > 0$, then we have, for all $C > 0$ and $C^\prime > 0$,
\begin{align}
\MoveEqLeft
    \sum_{t=1}^{\infty}{ u_t } = \sum_{t=1}^{\infty}{ \left( 1 - \pi_{\theta_t}(a) \right) } \\
    &= \sum_{t=1}^{\infty}{ \frac{C}{\exp\{ C^\prime \cdot t\}} } \\
    &\le \int_{t=0}^{\infty}{ \frac{C}{\exp\{ C^\prime \cdot t\}} } \\
    &= \frac{C}{C^\prime}.
\end{align}
Therefore, using \cref{lem:infinite_product_infinite_sum}, we have,
\begin{align}
    \prod_{t=1}^{\infty}{\left( 1 - u_t \right) } = \prod_{t=1}^{\infty}{ \pi_{\theta_t}(a) } > 0,
\end{align}
finishing the proofs.
\end{proof}

\textbf{\cref{lem:npg_aggressiveness}} (NPG aggressiveness)\textbf{.}
Fix sampling $a_t = a$ for all $t \ge 1$, using \cref{update_rule:softmax_natural_pg_special_on_policy_stochastic_gradient} with constant learning rate $\eta > 0$, where $\hat{r}_t$ is from \cref{def:simplified_on_policy_importance_sampling}, we have $1 - \pi_{\theta_t}(a) \in O(e^{-c \cdot t})$ for all $t \ge 1$, where $c > 0$.
\begin{proof}
See \citep[Theorem 3]{mei2021understanding}. We include a proof for completeness.

Suppose $a_1 = a, a_2 = a, \cdots, a_{t-1} = a$. We have,
\begin{align}
\label{eq:npg_aggressiveness_intermediate_1}
    \theta_{t}(a) &= \theta_1(a) + \eta \cdot \sum_{s=1}^{t-1}{ \hat{r}_s(a) } \qquad \left( \text{by \cref{update_rule:softmax_natural_pg_special_on_policy_stochastic_gradient}} \right) \\
    &= \theta_1(a) + \eta \cdot \sum_{s=1}^{t-1}{ \frac{ \sI\left\{ a_s = a \right\} }{ \pi_{\theta_s}(a) } \cdot r(a) } \qquad \left( \text{by \cref{def:simplified_on_policy_importance_sampling}} \right) \\
    &= \theta_1(a) + \eta \cdot \sum_{s=1}^{t-1}{ \frac{ r(a) }{ \pi_{\theta_s}(a) }  } \qquad \left( a_s = a \text{ for all } s \in \left\{ 1, 2, \dots, t-1 \right\} \right) \\
    &\ge \theta_1(a) + \eta \cdot \sum_{s=1}^{t-1}{ r(a) } \qquad \left( \pi_{\theta_s}(a) \in (0, 1) \right) \\
    &= \theta_1(a) +  \eta \cdot r(a) \cdot \left( t - 1 \right).
\end{align}
On the other hand, we have, for any other action $a^\prime \not= a$,
\begin{align}
\label{eq:npg_aggressiveness_intermediate_2}
    \theta_{t}(a^\prime) &= \theta_1(a^\prime) + \eta \cdot \sum_{s=1}^{t-1}{ \frac{ \sI\left\{ a_s = a^\prime \right\} }{ \pi_{\theta_s}(a^\prime) } \cdot r(a^\prime) } \qquad \left( \text{by \cref{update_rule:softmax_natural_pg_special_on_policy_stochastic_gradient,def:simplified_on_policy_importance_sampling}} \right) \\
    &= \theta_1(a^\prime). \qquad \left( a_s \not= a^\prime \text{ for all } s \in \left\{ 1, 2, \dots, t-1 \right\} \right)
\end{align}
Therefore, we have,
\begin{align}
\label{eq:npg_aggressiveness_intermediate_3}
\MoveEqLeft
    \pi_{\theta_t}(a) = 1 - \sum_{a^\prime \not= a}{ \pi_{\theta_t}(a^\prime) } \\
    &= 1 - \frac{ \sum_{a^\prime \not= a}{ \exp\{ \theta_t(a^\prime) \} } }{ \exp\{ \theta_t(a) \} + \sum_{a^\prime \not= a}{ \exp\{ \theta_t(a^\prime) \} } } \\
    &\ge 1 - \frac{ \sum_{a^\prime \not= a}{ \exp\{ \theta_1(a^\prime) \} } }{ \exp\{ \theta_1(a) + \eta \cdot r(a) \cdot \left( t - 1 \right) \} + \sum_{a^\prime \not= a}{ \exp\{ \theta_1(a^\prime) \} } }, \qquad \left( \text{by \cref{eq:npg_aggressiveness_intermediate_1,eq:npg_aggressiveness_intermediate_2}} \right)
\end{align}
which implies that,
\begin{align}
    1 - \pi_{\theta_t}(a) &\le \frac{ \sum_{a^\prime \not= a}{ \exp\{ \theta_1(a^\prime) \} } }{ \exp\{ \theta_1(a) + \eta \cdot r(a) \cdot \left( t - 1 \right) \} + \sum_{a^\prime \not= a}{ \exp\{ \theta_1(a^\prime) \} } } \\
    &\in O(e^{-c \cdot t}),
\end{align}
where $c \coloneqq \eta \cdot r(a) > 0$.
\end{proof}

\textbf{\cref{lem:zero_infinite_product}} (Good sampling)\textbf{.}
Let $\pi_{\theta_t}(a) \in (0, 1)$ and $a_t \sim \pi_{\theta_t}(\cdot)$, for all $t \ge 1$. If $\sum_{t=1}^{\infty}{ \left( 1 - \pi_{\theta_t}(a) \right) } = \infty$ (e.g., $1 - \pi_{\theta_t}(a) \in \Omega(1/t)$), then $\prod_{t=1}^{\infty}{ \pi_{\theta_t}(a) } = 0$.
\begin{proof}
According to \cref{lem:infinite_product_infinite_sum_2}, we have, for a sequence $u_t \in (0, 1)$ for all $t \ge 1$, if $\sum_{t=1}^{\infty}{ u_t } = \infty$, then $\prod_{t=1}^{\infty}{\left( 1 - u_t \right) } = 0$.

Let $u_t = 1 - \pi_{\theta_t}(a) \in (0, 1)$ according to the softmax parameterization, the result follows.
\end{proof}

\textbf{\cref{lem:npg_aggressiveness_value_baseline}} (Value baselines reduce NPG aggressiveness)\textbf{.}
Fix sampling $a_t = a$ for all $t \ge 1$. 
Then using \cref{update_rule:softmax_natural_pg_special_on_policy_stochastic_gradient_value_baseline} with a constant learning rate $\eta > 0$ and $\hat{r}_t$ from \cref{def:simplified_on_policy_importance_sampling} obtains $1 - \pi_{\theta_t}(a) \in \Omega(1/t)$ for all $t \ge 1$.
\begin{proof}
Since the claim is concerned with the policies underlying the parameter vectors and not the parameter vectors themselves, as noted after 
\cref{update_rule:softmax_natural_pg_special_on_policy_stochastic_gradient_value_baseline},
we used the equivalent \cref{update_rule:equivalent_update_softmax_natural_pg_special_on_policy_stochastic_gradient_value_baseline} with the change of $\hat{r}_t$ is from \cref{def:simplified_on_policy_importance_sampling} as follows,
\begin{align}
\theta_{t+1}(a) \gets \theta_t(a) + \eta \cdot \frac{ \sI\left\{ a_t = a \right\} }{ \pi_{\theta_t}(a) } \cdot \left( r(a) - \pi_{\theta_t}^\top r \right)\,.
\end{align}
Since $a_t = a$ for all $t \ge 1$ by assumption, we have,
\begin{align}
\label{eq:npg_aggressiveness_value_baseline_intermediate_1}
    \theta_{t+1}(a) \gets \theta_t(a) + \eta \cdot \frac{ r(a) - \pi_{\theta_t}^\top r }{ \pi_{\theta_t}(a) },
\end{align}
while for all $a^\prime \not= a$,
\begin{align}
\label{eq:npg_aggressiveness_value_baseline_intermediate_2}
    \theta_{t+1}(a^\prime) \gets \theta_t(a^\prime).
\end{align}
If $\pi_{\theta_t}^\top r < r(a)$, then we have,
\begin{align}
\label{eq:npg_aggressiveness_value_baseline_intermediate_3}
    \theta_{t+1}(a) &= \theta_t(a) + \eta \cdot \frac{ r(a) - \pi_{\theta_t}^\top r }{ \pi_{\theta_t}(a) } \qquad \left( \text{by \cref{eq:npg_aggressiveness_value_baseline_intermediate_1}} \right) \\
    &\ge 0, \qquad \left( \pi_{\theta_t}^\top r < r(a) \right)
\end{align}
which implies that,
\begin{align}
\label{eq:npg_aggressiveness_value_baseline_intermediate_4}
    \pi_{\theta_{t+1}}(a) &= \frac{ \exp\{ \theta_{t+1}(a) \}}{ \exp\{ \theta_{t+1}(a) \} + \sum_{a^\prime \not= a}{ \exp\{ \theta_{t+1}(a^\prime) \} } } \\
    &= \frac{ \exp\{ \theta_{t+1}(a) \}}{ \exp\{ \theta_{t+1}(a) \} + \sum_{a^\prime \not= a}{ \exp\{ \theta_{t}(a^\prime) \} } } \qquad \left( \text{by \cref{eq:npg_aggressiveness_value_baseline_intermediate_2}} \right) \\
    &\ge \frac{ \exp\{ \theta_{t}(a) \}}{ \exp\{ \theta_{t}(a) \} + \sum_{a^\prime \not= a}{ \exp\{ \theta_{t}(a^\prime) \} } } \qquad \left( \text{by \cref{eq:npg_aggressiveness_value_baseline_intermediate_3}} \right) \\
    &= \pi_{\theta_t}(a),
\end{align}
which means $1 - \pi_{\theta_t}(a)$ is decreasing. Otherwise, if $\pi_{\theta_t}^\top r \ge r(a)$, then using similar calculations, we have $\pi_{\theta_{t+1}}(a) \le \pi_{\theta_{t}}(a)$, i.e., $1 - \pi_{\theta_t}(a)$ is increasing and will not approach $0$. Since we prove $1 - \pi_{\theta_t}(a) \in \Omega(1/t)$, we assume the non-trivial case where  $\pi_{\theta_t}^\top r < r(a)$ for all $t \ge 1$.

According to \cref{lem:smoothness_softmax_special}, we have, 
\begin{align}
\label{eq:npg_aggressiveness_value_baseline_intermediate_5}
    \left| \pi_{\theta_{t+1}}(a) - \pi_{\theta_t}(a) - \Big\langle \frac{d \pi_{\theta_t}(a)}{d \theta_t}, \theta_{t+1} - \theta_{t} \Big\rangle \right| \le \frac{3}{4} \cdot \| \theta_{t+1} - \theta_{t} \|_2^2.
\end{align}
Therefore, we have,
\begin{align}
\label{eq:npg_aggressiveness_value_baseline_intermediate_6}
\MoveEqLeft
    \left( 1 - \pi_{\theta_t}(a) \right) - \left( 1 - \pi_{\theta_{t+1}}(a) \right) =  \pi_{\theta_{t+1}}(a) - \pi_{\theta_t}(a) - \Big\langle \frac{d \pi_{\theta_t}(a)}{d \theta_t}, \theta_{t+1} - \theta_{t} \Big\rangle + \Big\langle \frac{d \pi_{\theta_t}(a)}{d \theta_t}, \theta_{t+1} - \theta_{t} \Big\rangle \\
    &\le \frac{3}{4} \cdot \| \theta_{t+1} - \theta_{t} \|_2^2 + \Big\langle \frac{d \pi_{\theta_t}(a)}{d \theta_t}, \theta_{t+1} - \theta_{t} \Big\rangle \qquad \left( \text{by \cref{eq:npg_aggressiveness_value_baseline_intermediate_1}} \right) \\
    &= \frac{3 \cdot \eta^2}{4} \cdot \frac{ ( r(a) - \pi_{\theta_t}^\top r )^2 }{ \pi_{\theta_t}(a)^2 } + \eta \cdot \frac{d \pi_{\theta_t}(a)}{d \theta_t(a)} \cdot \frac{ r(a) - \pi_{\theta_t}^\top r }{ \pi_{\theta_t}(a) }, \qquad \left( \text{using the update} \right) \\
    &= \frac{3 \cdot \eta^2}{4} \cdot \frac{ ( r(a) - \pi_{\theta_t}^\top r )^2 }{ \pi_{\theta_t}(a)^2 } + \eta \cdot \left( 1 - \pi_{\theta_t}(a)  \right) \cdot \left( r(a) - \pi_{\theta_t}^\top r \right) \qquad \left( \frac{d \pi_{\theta_t}(a)}{d \theta_t(a)} = \pi_{\theta_t}(a) \cdot \left( 1 - \pi_{\theta_t}(a)  \right) \right) \\
    &\le \frac{3 \cdot \eta^2}{4} \cdot \frac{ ( r(a) - \pi_{\theta_t}^\top r )^2 }{ \pi_{\theta_1}(a)^2 } + \eta \cdot \left( 1 - \pi_{\theta_t}(a)  \right) \cdot \left( r(a) - \pi_{\theta_t}^\top r \right) \qquad \left( \text{by \cref{eq:npg_aggressiveness_value_baseline_intermediate_4}} \right) \\
    &\le \frac{3 \cdot \eta^2}{4} \cdot \frac{ \left( 1 - \pi_{\theta_t}(a)  \right)^2 }{ \pi_{\theta_1}(a)^2 } + \eta \cdot \left( 1 - \pi_{\theta_t}(a)  \right)^2 \\
    &= C \cdot \left( 1 - \pi_{\theta_t}(a)  \right)^2 \qquad \left( C \coloneqq \frac{3 \cdot \eta^2}{4 \cdot \pi_{\theta_1}(a)^2} + \eta \right) 
\end{align}
where the last inequality is because of,
\begin{align}
\label{eq:npg_aggressiveness_value_baseline_intermediate_7}
    r(a) - \pi_{\theta_t}^\top r &= \sum_{a^\prime \not= a}{ \pi_{\theta_t}(a^\prime) \cdot \left( r(a) - r(a^\prime) \right) } \\
    &\le 1 - \pi_{\theta_t}(a). \qquad \left( r \in (0, 1]^K \right)
\end{align}
Next, we have,
\begin{align}
\label{eq:npg_aggressiveness_value_baseline_intermediate_8}
\MoveEqLeft
    \frac{1}{ 1 - \pi_{\theta_t}(a) } = \frac{1}{ 1 - \pi_{\theta_1}(a)} + \sum_{s=1}^{t-1}{ \left[ \frac{1}{1 - \pi_{\theta_{s+1}}(a)} - \frac{1}{1 - \pi_{\theta_s}(a)} \right] } \\
    &= \frac{1}{ 1 - \pi_{\theta_1}(a)} + \sum_{s=1}^{t-1}{ \frac{1}{ \left( 1 - \pi_{\theta_{s+1}}(a) \right) \cdot \left( 1 - \pi_{\theta_{s}}(a) \right) } \cdot \left[ \left( 1 - \pi_{\theta_{s}}(a) \right) - \left( 1 - \pi_{\theta_{s+1}}(a) \right) \right] } \\
    &\le \frac{1}{ 1 - \pi_{\theta_1}(a)} + \sum_{s=1}^{t-1}{ \frac{1}{ \left( 1 - \pi_{\theta_{s+1}}(a) \right) \cdot \left( 1 - \pi_{\theta_{s}}(a) \right) } \cdot C \cdot \left( 1 - \pi_{\theta_s}(a) \right)^2 } \qquad \left( \text{by \cref{eq:npg_aggressiveness_value_baseline_intermediate_6}} \right) \\
    &\le \frac{1}{ 1 - \pi_{\theta_1}(a)} + \frac{ C }{2} \cdot (t-1),
\end{align}
which implies that, for all large enough $t \ge 1$, 
\begin{equation*}
    1 - \pi_{\theta_t}(a) \ge \frac{1}{ \frac{1}{ 1 - \pi_{\theta_1}(a)} + \frac{ C }{2} \cdot (t-1) } \\
    \in \Omega(1/t). \qedhere
\end{equation*}
\end{proof}

\section{Simulation Settings}

\subsection{One-state MDPs}

The detailed settings for simulations in \cref{fig:adversarial_initialization_uniform_initialization} are as follows. The total number of actions is $K = 20$, and after sorting rewards the true mean reward vector $r \in (0, 1)^K$ is,
\begin{align*}
    r = (&0.96990985, \ 0.95071431, \ 0.86617615, \ 0.83244264, \\
    &0.73199394, \ 0.70807258, \ 0.60111501, \ 0.59865848, \\
    &0.52475643, \ 0.43194502, \ 0.37454012, \ 0.30424224, \\
    &0.29122914, \ 0.21233911, \ 0.18340451, \ 0.18182497, \\
    &0.15601864, \ 0.15599452, \ 0.05808361, \ 0.02058449)^\top.
\end{align*}
For each $a \in [K]$, the sampled reward distribution is $\text{Bernoulli}(0.5)$, such that with probability $0.5$, one of the following two sampled reward values is observed,
\begin{align*}
    R_1 &= (-2.03009015, \ 3.96990985), \quad R_2 = (-2.04928569, \ 3.95071431), \\
    R_3 &= (-2.13382385, \ 3.86617615), \quad R_4 = (-2.16755736, \ 3.83244264), \\
    R_5 &= (-2.26800606, \ 3.73199394), \quad R_6 = (-2.29192742, \  3.70807258), \\
    R_7 &= (-2.39888499, \ 3.60111501), \quad R_8 = (-2.40134152, \ 3.59865848), \\
    R_9 &= (-2.47524357, \ 3.52475643), \quad R_{10} = (-2.56805498, \ 3.43194502), \\
    R_{11} &= (-2.62545988, \ 3.37454012), \quad R_{12} = (-2.69575776, \ 3.30424224), \\
    R_{13} &= (-2.70877086, \ 3.29122914), \quad R_{14} = (-2.78766089, \ 3.21233911), \\
    R_{15} &= (-2.81659549, \ 3.18340451), \quad R_{16} = (-2.81817503, \ 3.18182497), \\
    R_{17} &= (-2.84398136, \ 3.15601864), \quad R_{18} = (-2.84400548, \ 3.15599452), \\
    R_{19} &= (-2.94191639, \ 3.05808361), \quad R_{20} = (-2.97941551, \ 3.02058449).
\end{align*}
The initial parameter $\theta_1 \in \sR^K$ is,
\begin{align}
\label{eq:initial_theta}
    \theta(i) = \begin{cases}
		5, & \text{if } i = 2, \\
		0, & \text{otherwise},
	\end{cases}
\end{align}
such that the initial probability of best sub-optimal action is,
\begin{align}
    \pi_{\theta_1}(2) = \frac{ e^5}{ e^5 + 19 \cdot e^0 } \approx 0.8865,
\end{align}
and all the other action's probability, including the optimal action, is
\begin{align}
    \pi_{\theta_1}(1) = \frac{ e^0 }{ e^5 + 19 \cdot e^0 } \approx 0.0060.
\end{align}
We run \cref{update_rule:softmax_natural_pg_special_on_policy_stochastic_gradient_value_baseline} with learning rate,
\begin{align}
    \eta = \frac{1}{2} \cdot \frac{\pi_{\theta_t}(a_t) \cdot \left| r(a_t) - \pi_{\theta_t}^\top r \right|}{9},
\end{align}
and the results are shown in \cref{fig:adversarial_initialization_expected_reward,fig:adversarial_initialization_optimal_action_probability}.

For the results in \cref{fig:uniform_initialization_sub_optimality_gap}, \cref{def:simplified_on_policy_importance_sampling} is used, i.e., the true mean reward value $r(a_t)$ is observed for sampled action $a_t$, and we run the same update \cref{update_rule:softmax_natural_pg_special_on_policy_stochastic_gradient_value_baseline} using the same true mean reward vector $r \in (0,1)^K$ with learning rate $\eta = 0.1$ and uniform initial policy $\pi_{\theta_1}(a) = 1 / K$ for all $a \in [K]$.

\subsection{Tree MDPs}

We conduct experiments using a synthetic tree MDP with depth $d = 4$ and branch factor (number of actions) $k = 4$. The total number of states is
\begin{align}
    S = \sum_{i= 0}^{d-1}{ k^i } = \sum_{i= 0}^{3}{ 4^i } = 85.
\end{align}
The discount factor $\gamma = 0.9$. For each state $s \in \gS$, the immediate reward vector is,
\begin{align}
    r(s, \cdot) \coloneqq \left( 1.0, 0.9, 0.8, 0.2 \right)^\top.
\end{align}
The state distribution $\rho$ we used to measure the sub-optimality gap $V^*(\rho) - V^{\pi_{\theta_t}}(\rho)$ is $\rho(s_0) = 1$ for the root state $s_0$. The initial state distribution $\mu$ we used in the algorithm is set to satisfy \cref{assump:pos_init} as follows,
\begin{align}
    \mu = 0.2 \cdot \rho + \frac{0.8}{S - 1} \cdot \left( 1 - \rho \right),
\end{align}
i.e., $\mu(s_0) = 0.2$ and $\mu(s^\prime) = \frac{0.8}{84}$ for any other state $s^\prime \not= s_0$. We use an adversarial initialization, such that optimal actions have smallest initial probabilities, i.e., for all $s \in \gS$,
\begin{align}
    \pi_{\theta_1}(a^*(s) | s) = 0.07,
\end{align}
and $\pi_{\theta_1}(a^\prime | s) = 0.31$ for any sub-optimal action $a^\prime \not= a^*(s)$, where the optimal action $a^*(s)$ and policy $\pi^*$ are calculated using dynamic programming.
\begin{figure*}[ht]
\centering
\begin{subfigure}[b]{.325\linewidth}
\includegraphics[width=\linewidth]{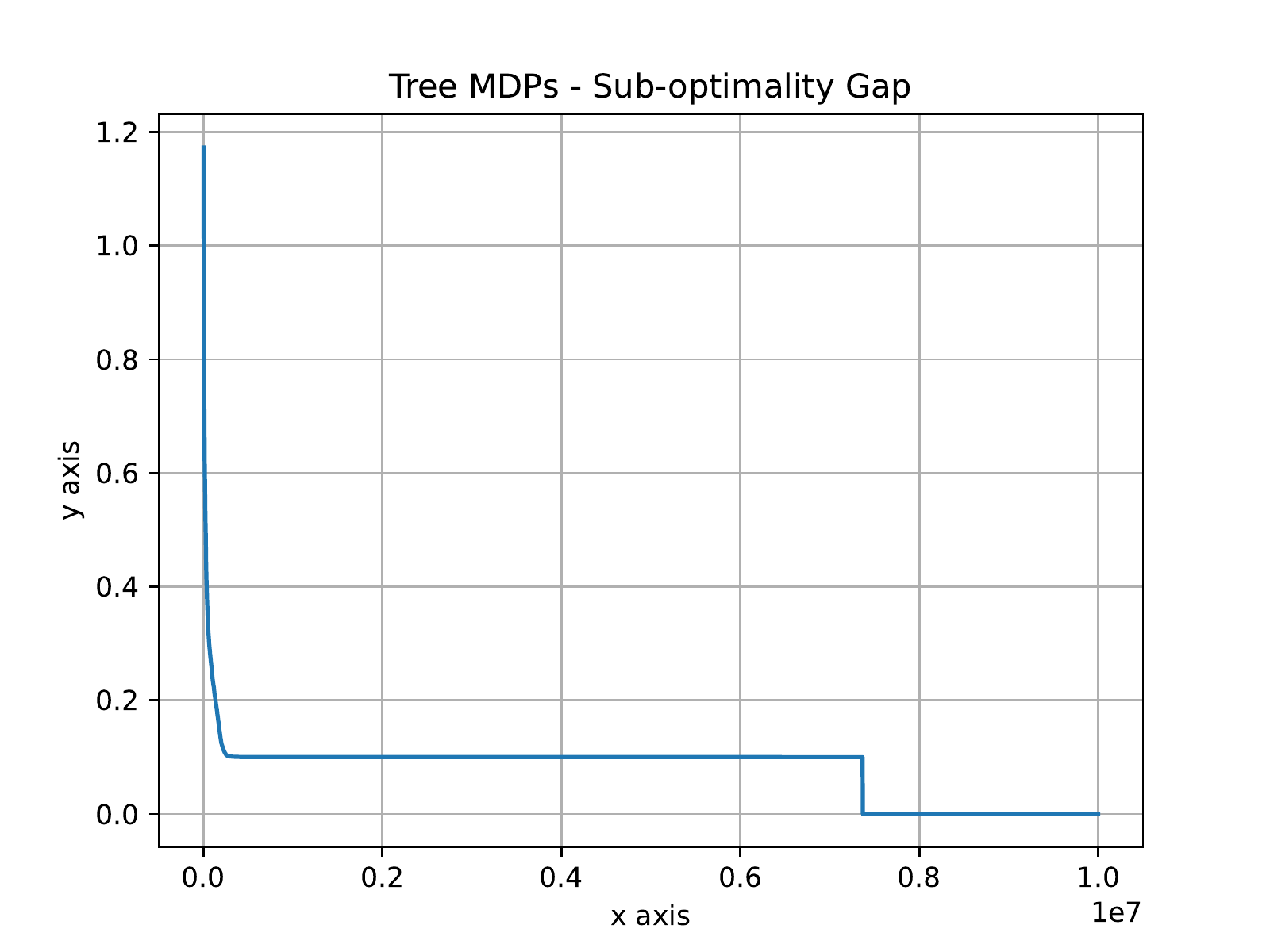}
\caption{$V^*(\rho) - V^{\pi_{\theta_t}}(\rho)$.}\label{fig:adversarial_initialization_tree_mdp_suboptimality_gap}
\end{subfigure}
\begin{subfigure}[b]{.325\linewidth}
\includegraphics[width=\linewidth]{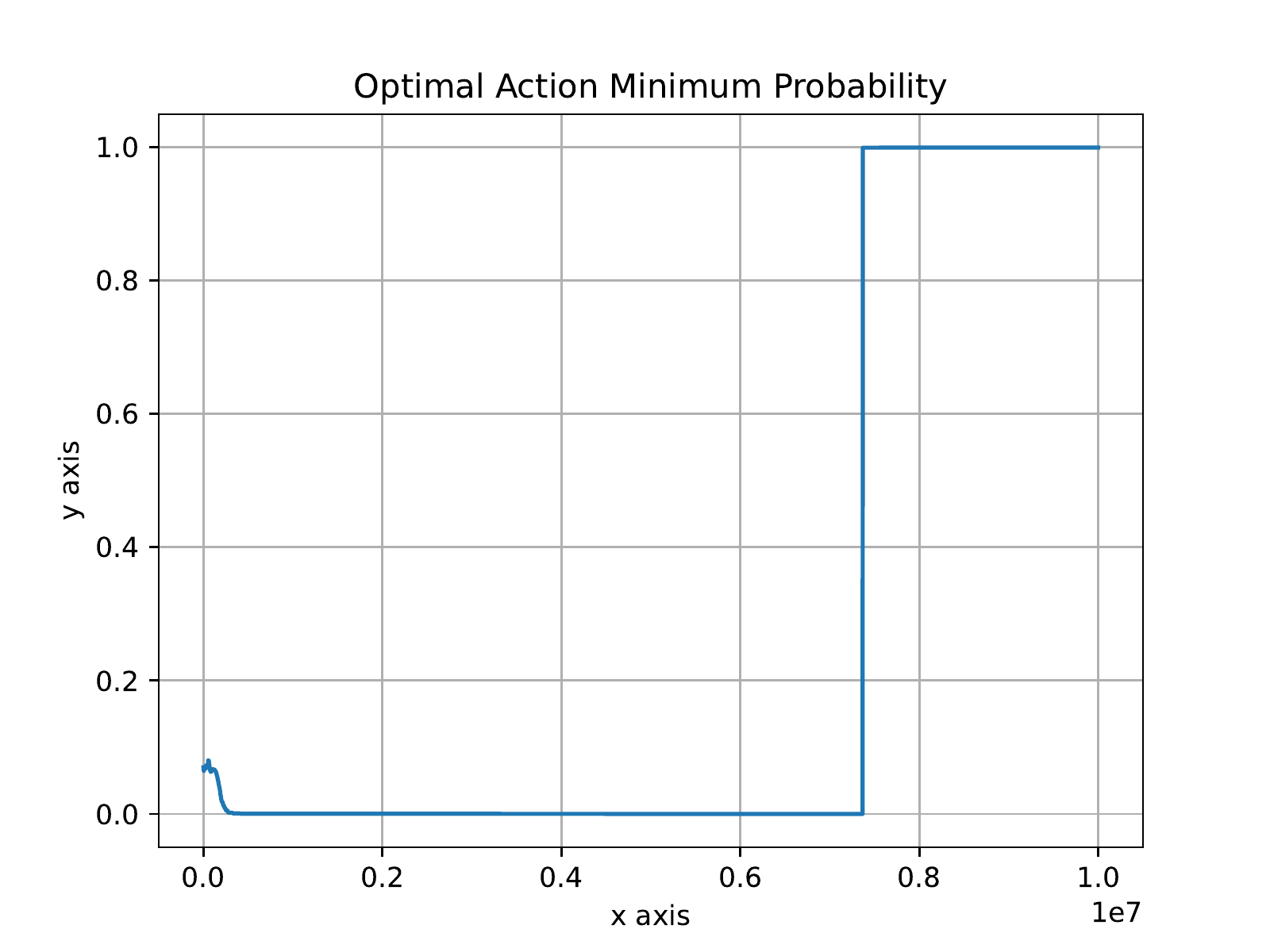}
\caption{$\min_{s \in \gS} \pi_{\theta_t}(a^*(s) | s)$.}\label{fig:adversarial_initialization_tree_mdp_min_optimal_action_probability}
\end{subfigure}
\begin{subfigure}[b]{.325\linewidth}
\includegraphics[width=\linewidth]{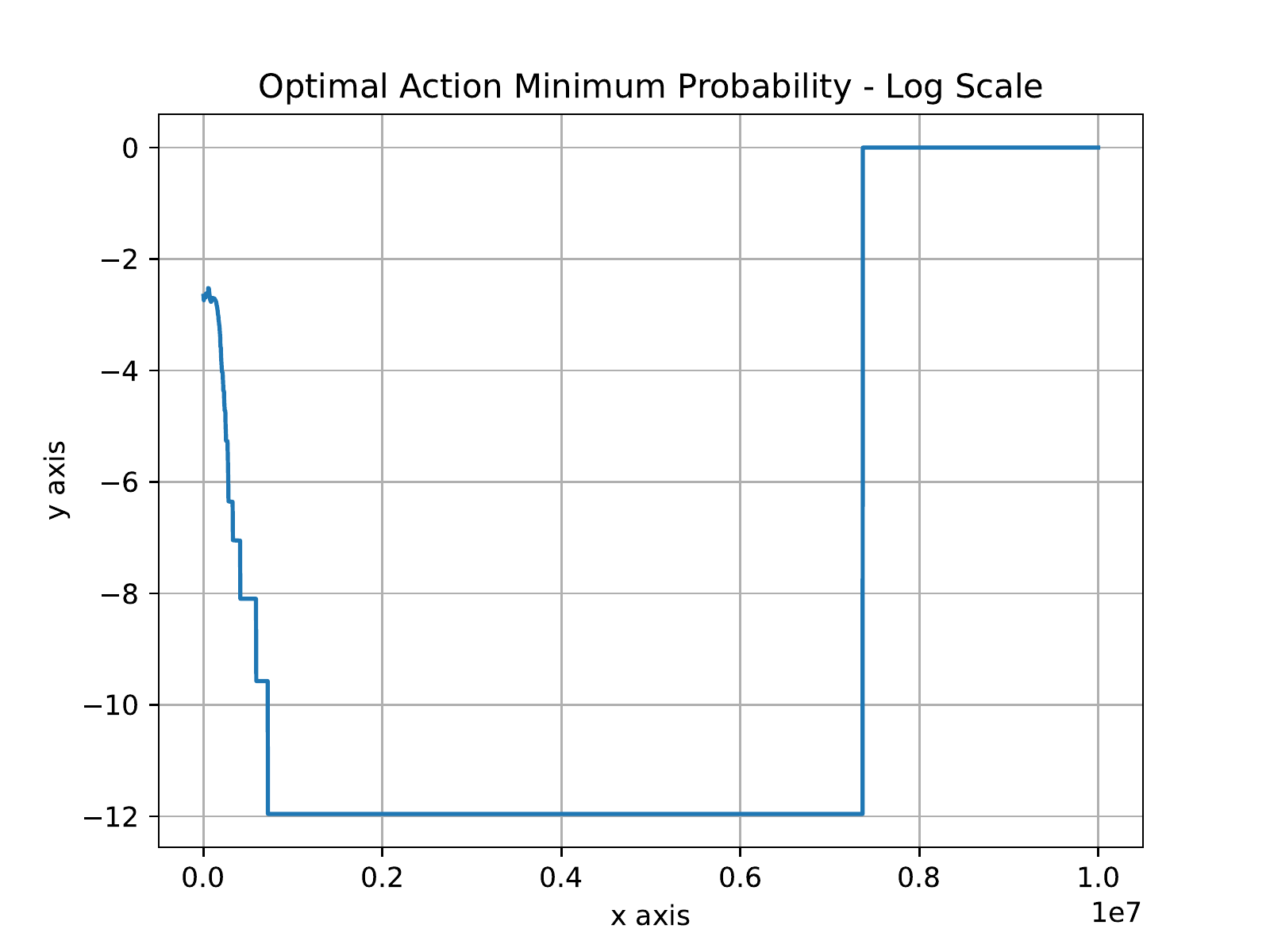}
\caption{$\log{ \min_{s \in \gS} \pi_{\theta_t}(a^*(s) | s)}$.}\label{fig:adversarial_initialization_tree_mdp_min_optimal_action_probability_log_scale}
\end{subfigure}
\caption{Results on a tree MDP, adversarial initialization.}
\label{fig:adversarial_initialization_tree_mdp}
\vspace{-10pt}
\end{figure*}

As shown in \cref{fig:adversarial_initialization_tree_mdp}, the sub-optimality gap $V^*(\rho) - V^{\pi_{\theta_t}}(\rho)$ quickly approached about $0.1$ value, while the optimal action's minimum probability $\min_{s \in \gS} \pi_{\theta_t}(a^*(s) | s)$ approaching very close to $0$. The algorithm got stuck on the sub-optimality plateau and finally escaped and approached the global optimal policy $\pi^*$ after about $7 \times 10^6$ iterations.
\begin{figure*}[ht]
\centering
\begin{subfigure}[b]{.245\linewidth}
\includegraphics[width=\linewidth]{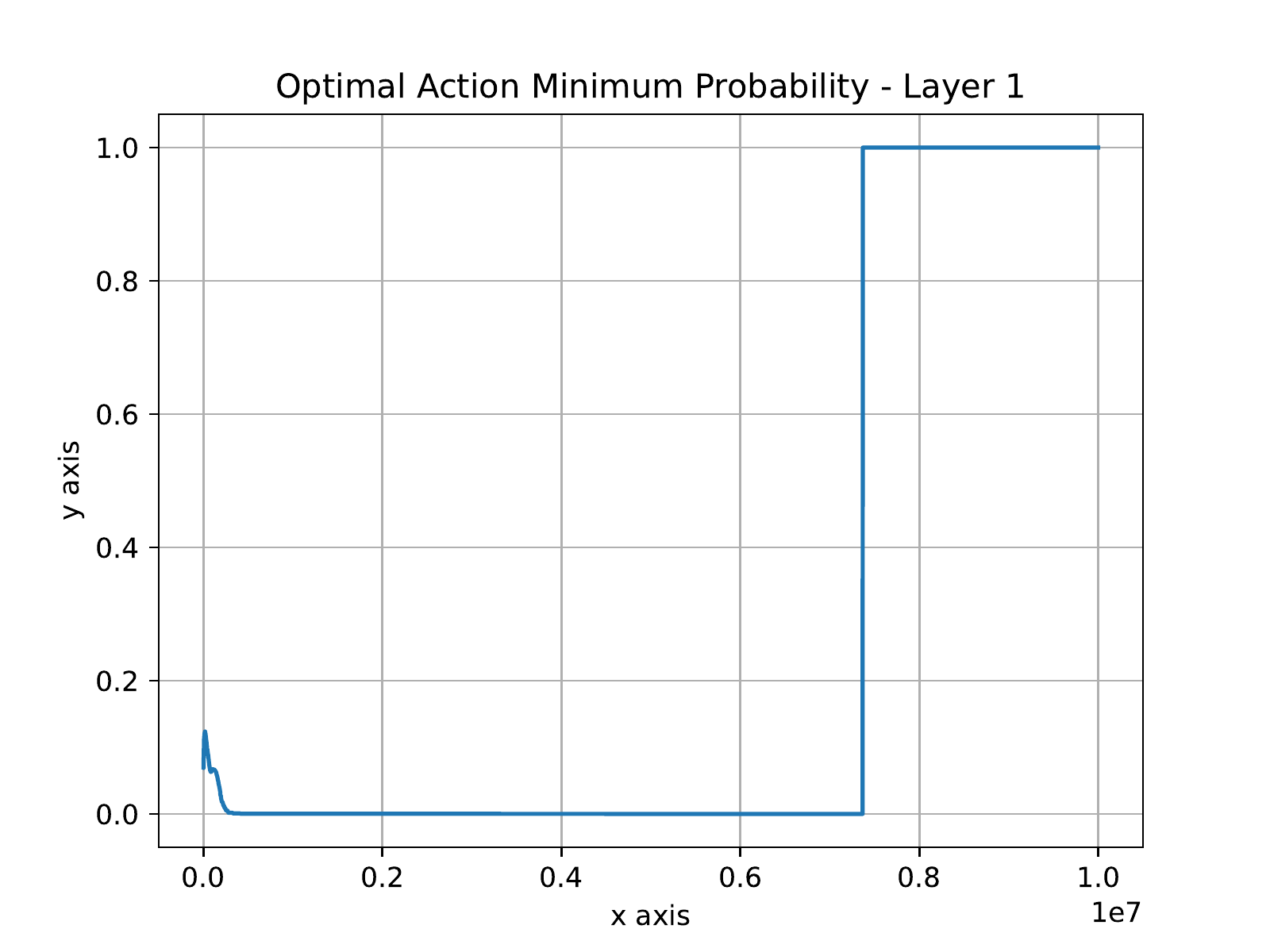}
\caption{$\pi_{\theta_t}(a^*(s_0) | s_0)$.}\label{fig:adversarial_initialization_tree_mdp_min_optimal_action_probability_layer_1}
\end{subfigure}
\begin{subfigure}[b]{.245\linewidth}
\includegraphics[width=\linewidth]{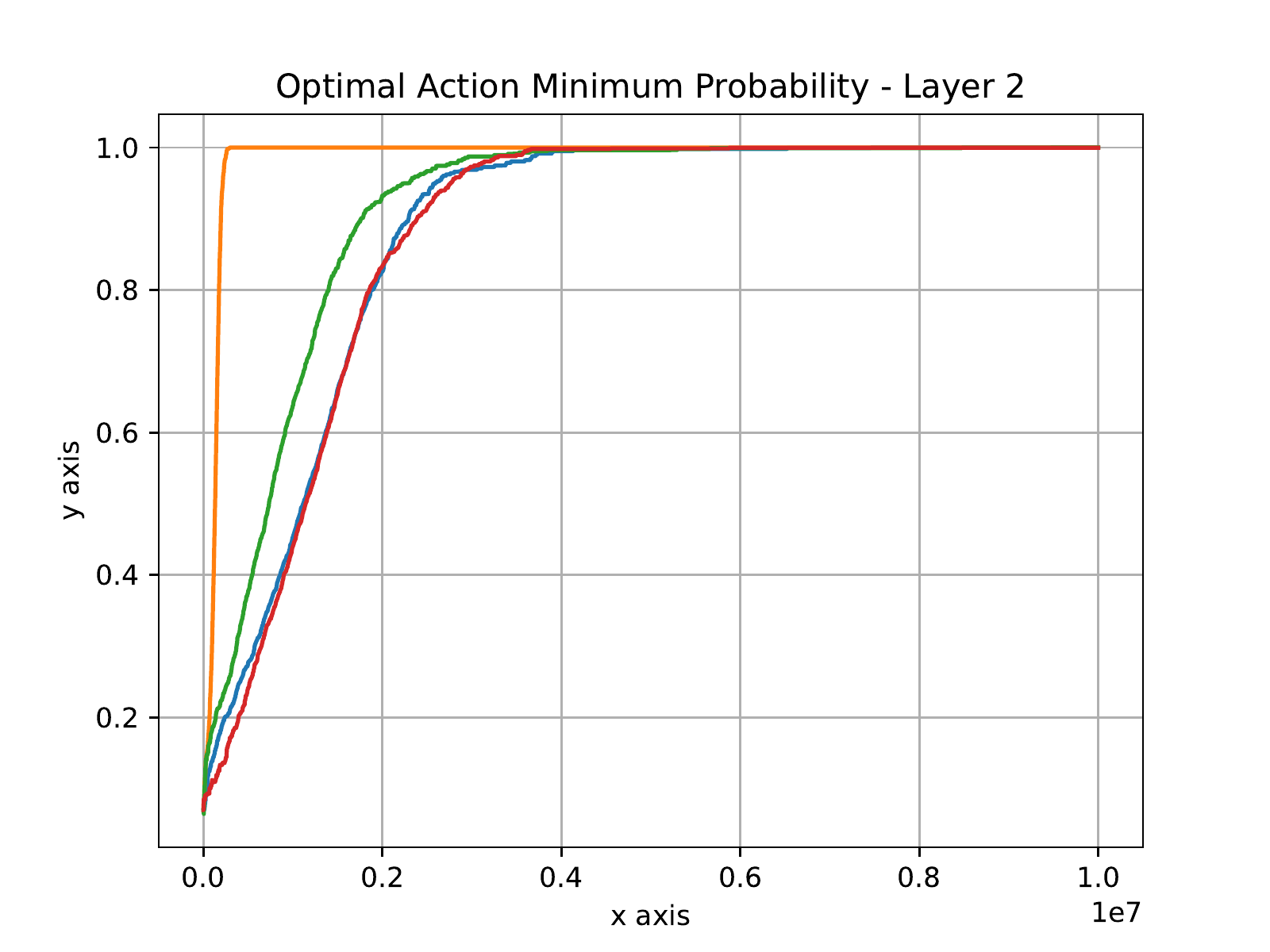}
\caption{Layer 2.}\label{fig:adversarial_initialization_tree_mdp_min_optimal_action_probability_layer_2}
\end{subfigure}
\begin{subfigure}[b]{.245\linewidth}
\includegraphics[width=\linewidth]{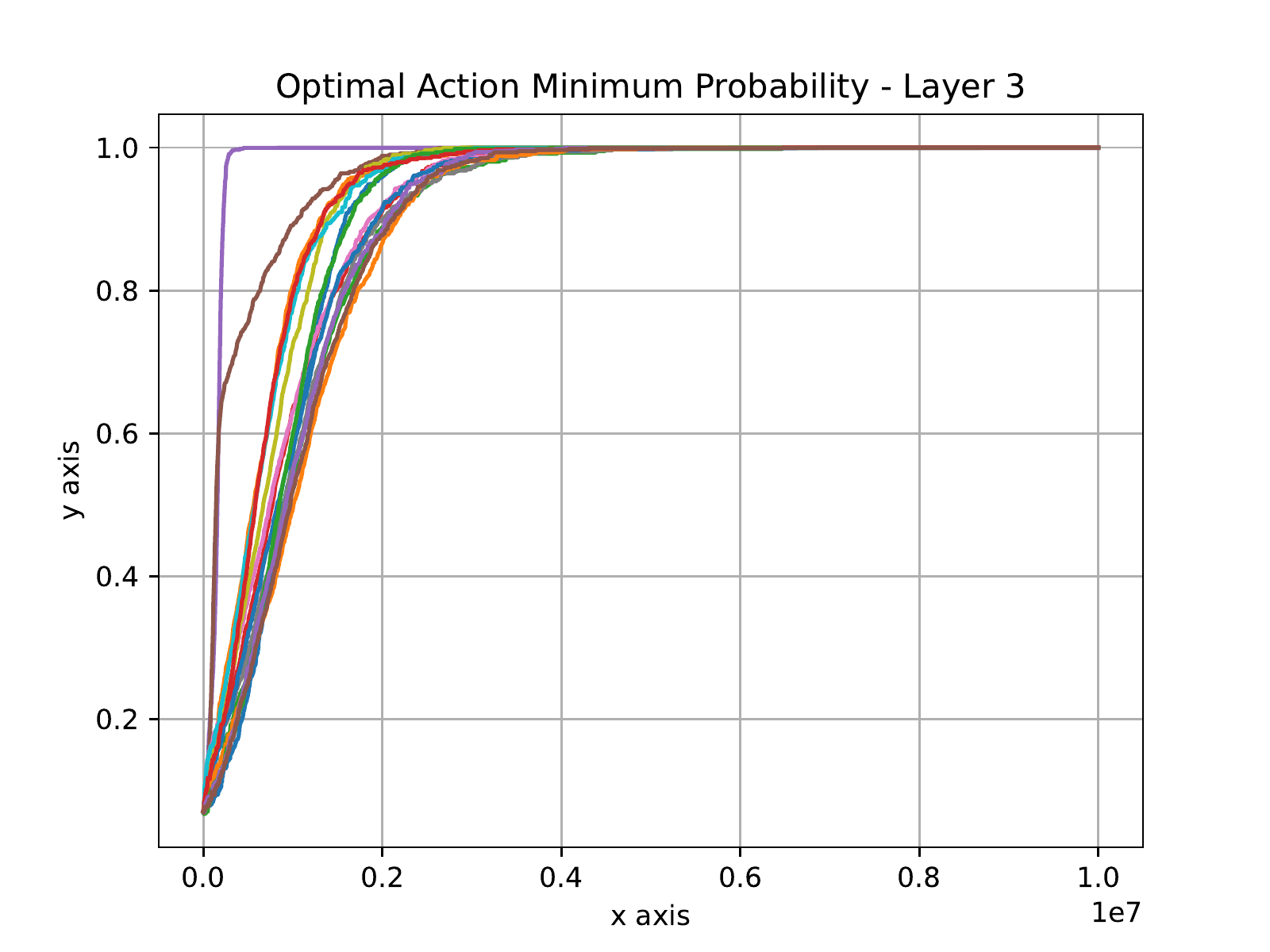}
\caption{Layer 3.}\label{fig:adversarial_initialization_tree_mdp_min_optimal_action_probability_layer_3}
\end{subfigure}
\begin{subfigure}[b]{.245\linewidth}
\includegraphics[width=\linewidth]{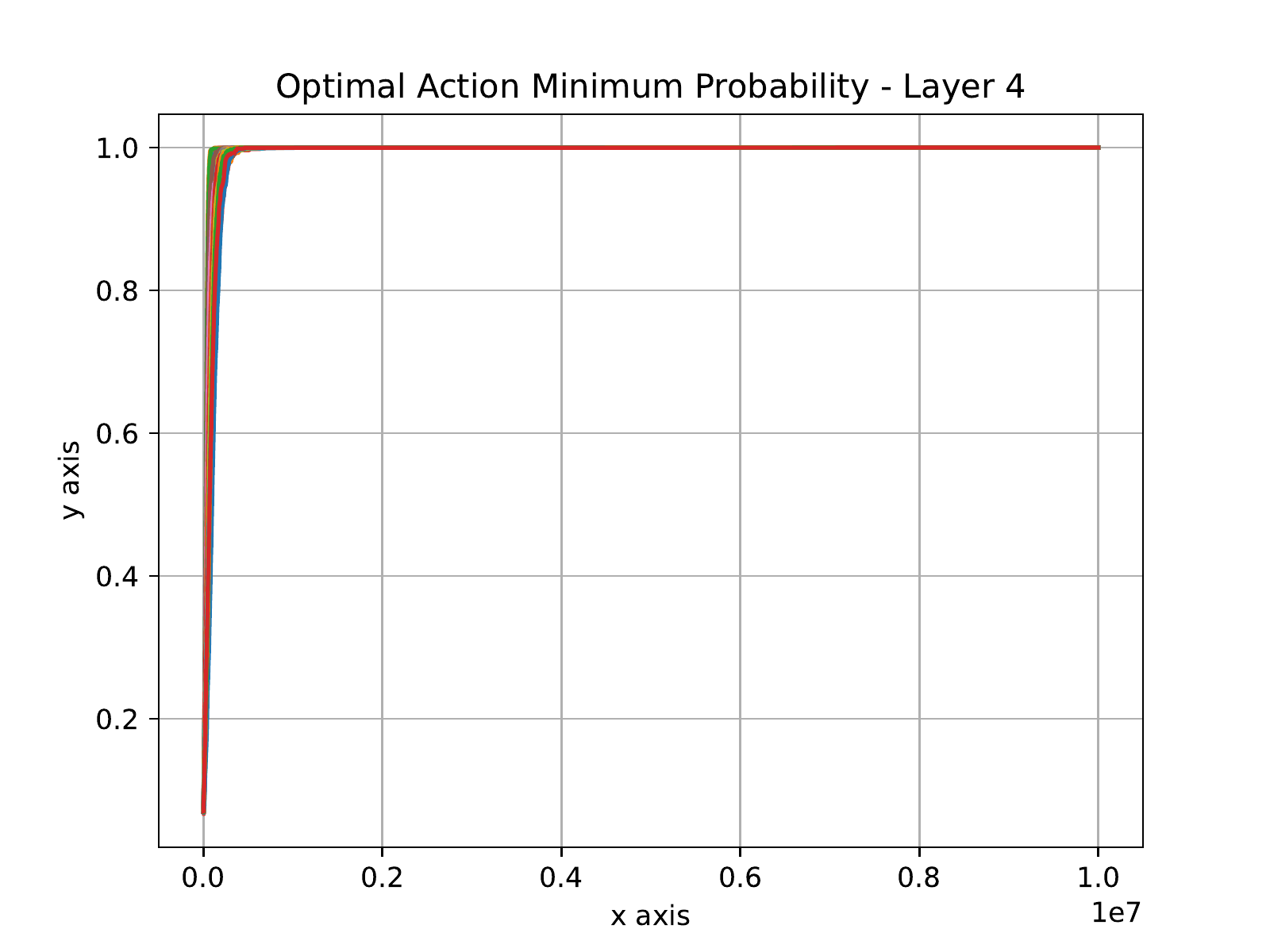}
\caption{Layer 4.}\label{fig:adversarial_initialization_tree_mdp_min_optimal_action_probability_layer_4}
\end{subfigure}
\caption{Optimal actions' probabilities for different layers of states.}
\label{fig:adversarial_initialization_tree_mdp_min_prob_optimal_action_layers}
\vspace{-10pt}
\end{figure*}

\cref{fig:adversarial_initialization_tree_mdp_min_prob_optimal_action_layers} demonstrates a more detailed process of the optimization. Note that the tree MDP has four layers of states, with state numbers $S_1 = 1$ (root state), $S_2 = k = 4$, $S_3 = k^2 = 16$, and $S_4 = k^3 = 64$, respectively. We calculated the optimal actions' probabilities for each layers of states. For example, \cref{fig:adversarial_initialization_tree_mdp_min_prob_optimal_action_layers}(b) shows $ \pi_{\theta_t}(a^*(s) | s)$ for all state $s$ in Layer 2.

As shown in \cref{fig:adversarial_initialization_tree_mdp_min_prob_optimal_action_layers}, $ \pi_{\theta_t}(a^*(s) | s)$ for states in Layer 4 approaches to $1$ most quickly comparing to other layers of states. However, it took $ \pi_{\theta_t}(a^*(s) | s)$ for Layers 2 and 3 several millions of iterations to approach $1$, and in the meanwhile $ \pi_{\theta_t}(a^*(s_0) | s_0)$ decreased to near zero values. Therefore, $a^*(s_0)$ would have very small chance to be sampled and learned using on-policy sampling, which created the sub-optimality plateau for about $7 \times 10^6$ iterations.

\section{Miscellaneous Extra Supporting Results}
\label{sec:supporting_results}
Recall that 
$(X_t,\cF_t)_{t\ge 1}$ is a \emph{sub-martingale} (super-martingale, martingale) if $(X_t)_{t\ge 1}$ is adapted to the filtration $(\cF_t)_{t\ge 1}$ 
and $\EE{X_{t+1}|\cF_t} \ge X_t$
($\EE{X_{t+1}|\cF_t} \le X_t$, $\EE{X_{t+1}|\cF_t} = X_t$, respectively)
 holds almost surely for any $t\ge 1$.
For brevity, let $\EEt{\cdot}$ denote $\EE{\cdot|\cF_t}$ where the filtration should be clear from the context and we also extend this notation to $t=0$ such that $\chE_0{U} = \EE{U}$.

\begin{theorem}[Theorem 13.3.2 of \citep{AthreyaLahiri2006}]
\label{thm:doob2}
Let $(X_t,\cF_t)_{t\ge 1}$ be a sub-martingale such that $\sup_{n\ge 1} \EE{ X_n^+ }<\infty$.
Then $(X_t)_{t\ge 1}$ converges to a finite limit $X_\infty$ a.s. and $\EE{|X_\infty|}<\infty$.
\end{theorem}
\cref{thm:doob2} implies the following \cref{thm:smc}.
\begin{theorem}[Doob's supermartingale convergence theorem \citep{doob2012measure}]
\label{thm:smc}
If $(Y_t)_{t\ge 1}$ is an $\{ \cF_t \}_{t\ge 1}$-adapted sequence such that $\EE{Y_{t+1}|\cF_t}\le Y_t$ and
$\sup_t \EE{ |Y_t| }<\infty$ then $\{Y_t\}_{t\ge 1}$ almost surely converges (a.s.) and, in particular, $Y_t \to Y$ a.s. as $t\to\infty$ where $Y=\limsup_{t\to\infty}Y_t$ is such that $\EE{|Y|}<\infty$.
\end{theorem}

\begin{lemma}
\label{lem:submnoiseconv}
Let $(X_t,\cF_t)_{t\ge 1}$ be a sub-martingale such that $\sup_{n\ge 1} \EE{X_n^+}<\infty$.
Let 
$Z_n= \sum_{t=0}^{n-1} X_{t+1} - \EEt{X_{t+1}}$ 
and
assume that for any $n$, $\EE{|Z_n|}<\infty$.
Then, 
 $X_{t+1}-\EEt{X_{t+1}} \to 0$ almost surely as $t\to\infty$.
\end{lemma}
\begin{proof}
By construction, and the assumption that $\EE{|Z_n|}<\infty$, $(Z_n, \cF_n)_{n \ge 1}$ is a martingale
and as such, it is also a sub-martingale.
Further, for any $n\ge 1$, 
\begin{align*}
Z_n 
& = (X_n - \chE_{n-1}[X_n]) + (X_{n-1} - \chE_{n-2}[X_{n-1}]) + \dots + (X_1 - \chE_0[X_1]) \\
& = X_n  + (X_{n-1} - \chE_{n-1}[X_n]) + (X_{n-2} - \chE_{n-2}[X_{n-1}]) + \dots + (X_1 - \chE_1[X_2])- \chE_0[X_1] \\
& \le X_n - \chE_0[X_1]\,.
\end{align*}
Hence, $Z_n^+ \le (X_n - \chE_0[X_1])^+
 \le 
 (X_n + |\chE_0[X_1]|)^+
 \le
 X_n^+ + \chE[|X_1|]$, and hence 
$\sup_{n\ge 1} \EE{Z_n^+}
\le \sup_{n\ge 1} \EE{X_n^+} + \EE{|X_1|}<\infty$.
Applying \cref{thm:doob2} to $(Z_n,\cF_n)_{n\ge 1}$, we get that 
there exist a random variable $Z_\infty$ such that $\EE{|Z_\infty|}<\infty$ and $Z_n \to Z_\infty$ almost surely as $n\to \infty$.
On the set where $(Z_n)_{n\ge 1}$ converges to $Z_\infty$, $(Z_n)_{n\ge 1}$ is a Cauchy sequence,
and it follows that  $|X_{n+1}-\chE_n{X_{n+1}}|=|Z_{n+1}-Z_n|\to 0$, finishing the proof.
\end{proof}

\begin{corollary}
\label{cor:submnoiseconv}
Let $(X_t,\cF_t)_{t\ge 1}$ be a sub-martingale such that $X_n\in [a,b]$ almost surely for some reals $a<b$.
Let $Z_n= \sum_{t=0}^{n-1} X_{t+1} - \EEt{X_{t+1}}$ and
assume that for any $n$, $\EE{|Z_n|}<\infty$.
Then, 
 $X_{t+1}-\EEt{X_{t+1}} \to 0$ almost surely as $t\to\infty$.
\end{corollary}
\begin{proof}
We use \cref{lem:submnoiseconv}, hence we need to verify that the conditions of this result hold.
Clearly, $\sup_{n\ge 1} \EE{X_n^+} \le b^+<\infty$.
Next, we have for any $n\ge 1$ that
$|Z_n|\le \sum_{t=0}^{n-1} | X_{t+1} - \EEt{X_{t+1}} |
\le n (b-a)<\infty$ since $\EEt{X_{t+1}}\in [a,b]$ also holds when $X_{t+1}\in [a,b]$.
\end{proof}

\begin{lemma}[Extended Borel-Cantelli Lemma, Corollary 5.29 of \citep{breiman1992probability}]
\label{lem:ebc}
Let $( \cF_n)_{n \ge 1}$ be a filtration, $A_n \in \cF_n$.
Then, almost surely, 
\begin{align*}
\{ \omega \,: \, \omega \in A_n \text{ infinitely often } \} = \left\{ \omega \, : \, 
\sum_{n=1}^\infty \PP(A_n|\cF_n) \right\}\,.
\end{align*}
\end{lemma}

\begin{lemma}[Piecewise linear domination for sigmoid-like functions]
\label{lem:piecewise_linear_domination}
Given $p \in (0,1]$, define the following function,
\begin{align}
\label{eq:piecewise_linear_domination_result_0}
	f_p(y) \coloneqq \frac{e^y - 1}{e^y + \frac{1 - p}{ p } }.
\end{align}
For any fixed $p \in (0, 1]$, and any fixed $\epsilon \in [0, 1]$, we have,
\begin{align}
\label{eq:piecewise_linear_domination_result_1}
	\left( 1 - \epsilon \right) \cdot p \cdot y &\le f_p(y) \le \left( 1 + \epsilon \right) \cdot p \cdot y, \quad \text{for all } y \in  [0, \epsilon], \\
\label{eq:piecewise_linear_domination_result_2}
	\left( 1 + \epsilon \right) \cdot p \cdot y &\le f_p(y) \le \left( 1 - \epsilon \right) \cdot p \cdot y, \quad \text{for all } y \in  [- \epsilon, 0].
\end{align}
\end{lemma}
\begin{proof}
\textbf{First part.} For $y = 0$ or $\epsilon = 0$, \cref{eq:piecewise_linear_domination_result_1,eq:piecewise_linear_domination_result_2} hold trivially.

First, if $y = 0$, then we have $f_p(y) = p \cdot y = 0$, which means \cref{eq:piecewise_linear_domination_result_1,eq:piecewise_linear_domination_result_2} hold. Next, if $\epsilon = 0$, then $y = 0$ (since we prove for $\left| y \right| \le \epsilon$) and \cref{eq:piecewise_linear_domination_result_1,eq:piecewise_linear_domination_result_2} again hold trivially.

We then prove for $\epsilon \in (0, 1]$ and for $y \not= 0$. Define the following function, for $p \in [0, 1]$, 
\begin{align}
\label{eq:piecewise_linear_domination_intermediate_1}
	g_p(y) &\coloneqq \frac{ e^y - 1 }{ p \cdot y \cdot \left( e^y - 1 \right) + y }, \text{ for all } y \not= 0.
\end{align}

\textbf{Second part.} \cref{eq:piecewise_linear_domination_result_1}. We prove for any fixed $p \in (0, 1]$, and any fixed $\epsilon \in (0, 1]$,
\begin{align}
\label{eq:piecewise_linear_domination_intermediate_2}
	1 - \epsilon \le g_p(y) \le 1 + \epsilon, \text{ for all } y \in  (0, \epsilon].
\end{align}
First, for $p = 1$, and any fixed $\epsilon \in (0, 1]$, we have, for all $y \in  (0, \epsilon]$,
\begin{align}
\label{eq:piecewise_linear_domination_intermediate_3}
	g_1(y) &= \frac{ e^y - 1 }{ y \cdot e^y } \qquad \left( \text{by \cref{eq:piecewise_linear_domination_intermediate_1}} \right) \\
	&= \frac{1 - e^{-y}}{y} \\
	&\ge \frac{ y - y^2 }{y} \qquad \left( e^{- y} \le 1 - y + y^2, \text{ for all } y > 0 \right) \\
	&= 1 - y \qquad \left( y > 0 \right) \\
	&\ge 1 - \epsilon. \qquad \left( y \in  (0, \epsilon] \right)
\end{align}
Second, for $p = 0$, and any fixed $\epsilon \in (0, 1]$, we have, for all $y \in  (0, \epsilon]$,
\begin{align}
\label{eq:piecewise_linear_domination_intermediate_4}
	g_0(y) &= \frac{ e^y - 1 }{ y } \qquad \left( \text{by \cref{eq:piecewise_linear_domination_intermediate_1}} \right) \\
	&\le \frac{ y + y^2 }{ y } \qquad \left( e^y \le 1 + y + y^2, \text{ for all } y \le 1 \right) \\
	&= 1 + y \qquad \left( y > 0 \right) \\
	&\le 1 + \epsilon. \qquad \left( y \in  (0, \epsilon] \right)
\end{align}
Note that, for any $y > 0$, we have, $g_p(y)$ is monotonically decreasing over $p$, since
\begin{align}
\label{eq:piecewise_linear_domination_intermediate_5}
    g_p(y)^{-1} = p \cdot y + \frac{y}{ e^y - 1}
\end{align}
is monotonically increasing over $p$. 

Therefore, we have, any fixed $p \in (0, 1]$, and any fixed $\epsilon \in (0, 1]$, for all $y \in  (0, \epsilon]$,
\begin{align}
\label{eq:piecewise_linear_domination_intermediate_6}
    1 - \epsilon &\le g_1(y) \qquad \left( \text{by \cref{eq:piecewise_linear_domination_intermediate_3}} \right) \\
    &\le g_p(y) \qquad \left( g_p(y) \text{ is monotonically decreasing over } p \right) \\
    &\le g_0(y) \\
    &\le 1 + \epsilon, \qquad \left( \text{by \cref{eq:piecewise_linear_domination_intermediate_4}} \right)
\end{align}
Note that,
\begin{align}
\label{eq:piecewise_linear_domination_intermediate_7}
	f_p(y) &= \frac{e^y - 1}{e^y + \frac{1 - p}{ p } } \qquad \left( \text{by \cref{eq:piecewise_linear_domination_result_0}} \right) \\
	&= \frac{ e^y - 1 }{ p \cdot y \cdot \left( e^y - 1 \right) + y } \cdot p \cdot y \qquad \left( p \in (0,1], \epsilon \in (0, 1], \text{ and } y \in  (0, \epsilon] \right) \\
	&= g_p(y) \cdot p \cdot y. \qquad \left( \text{by \cref{eq:piecewise_linear_domination_intermediate_1}} \right)
\end{align}
Therefore, according to \cref{eq:piecewise_linear_domination_intermediate_6,eq:piecewise_linear_domination_intermediate_7}, we have,
\begin{align}
\label{eq:piecewise_linear_domination_intermediate_8}
    \left( 1 - \epsilon \right) \cdot p \cdot y &\le f_p(y) \le \left( 1 + \epsilon \right) \cdot p \cdot y, \qquad \left( p \cdot y > 0 \right) 
\end{align}
which means any fixed $p \in (0, 1]$, and any fixed $\epsilon \in (0, 1]$, \cref{eq:piecewise_linear_domination_result_1} holds for all $y \in  (0, \epsilon]$.

\textbf{Second part.} \cref{eq:piecewise_linear_domination_result_2}. We prove for any fixed $p \in (0, 1]$, and any fixed $\epsilon \in (0, 1]$,
\begin{align}
\label{eq:piecewise_linear_domination_intermediate_9}
	1 - \epsilon \le g_p(y) \le 1 + \epsilon, \text{ for all } y \in  [ - \epsilon, 0).
\end{align}
First, for $p = 1$, and any fixed $\epsilon \in (0, 1]$, we have, for all $y \in  [ - \epsilon, 0)$,
\begin{align}
\label{eq:piecewise_linear_domination_intermediate_10}
\MoveEqLeft
	g_1(y) = \frac{ e^y - 1 }{ y \cdot e^y } \qquad \left( \text{by \cref{eq:piecewise_linear_domination_intermediate_1}} \right) \\
	&= \frac{ 1 - e^{-y} }{y} \\
	&\le \frac{ y -  y^2 }{y } \qquad \left( e^{- y} \le 1 - y + y^2, \text{ for all } y \ge -1 \right) \\
	&= 1 - y \qquad \left( y < 0 \right) \\
	&\le 1 + \epsilon. \qquad \left( y \in  [ - \epsilon, 0) \right)
\end{align}
Second, for $p = 0$, and any fixed $\epsilon \in (0, 1]$, we have, for all $y \in  [ - \epsilon, 0)$,
\begin{align}
\label{eq:piecewise_linear_domination_intermediate_11}
	g_0(y) &= \frac{ e^y - 1 }{ y } \qquad \left( \text{by \cref{eq:piecewise_linear_domination_intermediate_1}} \right) \\
	&\ge \frac{ y + y^2 }{ y } \qquad \left( e^y \le 1 + y + y^2, \text{ for all } y \le 1 \right) \\
	&= 1 + y \qquad \left( y < 0 \right) \\
	&\ge 1 - \epsilon, \qquad \left( y \in  [ - \epsilon, 0) \right)
\end{align}
Note that, for any $y < 0$, we have, $g_p(y)$ is monotonically increasing over $p$, since
\begin{align}
\label{eq:piecewise_linear_domination_intermediate_12}
    g_p(y)^{-1} = p \cdot y + \frac{y}{ e^y - 1}
\end{align}
is monotonically decreasing over $p$. 

Therefore, we have, any fixed $p \in (0, 1]$, and any fixed $\epsilon \in (0, 1]$, for all $y \in [ -\epsilon, 0)$,
\begin{align}
\label{eq:piecewise_linear_domination_intermediate_13}
    1 - \epsilon &\le g_0(y) \qquad \left( \text{by \cref{eq:piecewise_linear_domination_intermediate_11}} \right) \\
    &\le g_p(y) \qquad \left( g_p(y) \text{ is monotonically increasing over } p \right) \\
    &\le g_1(y) \\
    &\le 1 + \epsilon, \qquad \left( \text{by \cref{eq:piecewise_linear_domination_intermediate_10}} \right)
\end{align}
Note that,
\begin{align}
\label{eq:piecewise_linear_domination_intermediate_14}
	f_p(y) &= \frac{e^y - 1}{e^y + \frac{1 - p}{ p } } \qquad \left( \text{by \cref{eq:piecewise_linear_domination_result_0}} \right) \\
	&= \frac{ e^y - 1 }{ p \cdot y \cdot \left( e^y - 1 \right) + y } \cdot p \cdot y \qquad \left( p \in (0,1], \epsilon \in (0, 1], \text{ and } y \in  [ -\epsilon, 0) \right) \\
	&= g_p(y) \cdot p \cdot y. \qquad \left( \text{by \cref{eq:piecewise_linear_domination_intermediate_1}} \right)
\end{align}
Therefore, according to \cref{eq:piecewise_linear_domination_intermediate_13,eq:piecewise_linear_domination_intermediate_14}, we have,
\begin{align}
\label{eq:piecewise_linear_domination_intermediate_15}
    \left( 1 + \epsilon \right) \cdot p \cdot y &\le f_p(y) \le \left( 1 - \epsilon \right) \cdot p \cdot y, \qquad \left( p \cdot y < 0 \right) 
\end{align}
which means any fixed $p \in (0, 1]$, and any fixed $\epsilon \in (0, 1]$, \cref{eq:piecewise_linear_domination_result_2} holds for all $y \in [ - \epsilon, 0)$.
\end{proof}

\begin{lemma}
\label{lem:stochastic_natural_lojasiewicz_continuous_special}
Let $r \in [0,1]^K$ and $a^* \coloneqq \argmax_{a \in [K]}{ r(a) }$ be the optimal action. Denote $\Delta \coloneqq r(a^*) - \max_{a \not= a^*}{ r(a) }$ as the reward gap of $r$. We have, for any policy $\pi$,
\begin{align}
    \sum_{i = 1}^{K} \pi(i)^2 \cdot \left| r(i) - \pi^\top r \right|^3 &\ge \frac{\Delta}{K-1} \cdot \pi(a^*)^2 \cdot \left( r(a^*) - \pi^\top r \right)^2.
\end{align}
\end{lemma}
\begin{proof}
\textbf{First case.} If $\pi^\top r \le \max_{a \not= a^*}{ r(a) }$, then we have,
\begin{align}
\label{eq:stochastic_natural_lojasiewicz_continuous_special_intermediate_1}
    r(a^*) - \pi^\top r \ge r(a^*) - \max_{a \not= a^*}{ r(a) } = \Delta.
\end{align}
Therefore, we have,
\begin{align}
\label{eq:stochastic_natural_lojasiewicz_continuous_special_intermediate_2}
    \sum_{i = 1}^{K} \pi(i)^2 \cdot \left| r(i) - \pi^\top r \right|^3 &\ge  \pi(a^*)^2 \cdot \left| r(a^*) - \pi^\top r \right|^3 \qquad \left( \text{fewer terms} \right) \\
    &\ge \pi(a^*)^2 \cdot \left( r(a^*) - \pi^\top r \right)^2 \cdot \Delta \qquad \left( \text{by \cref{eq:stochastic_natural_lojasiewicz_continuous_special_intermediate_1}} \right) \\
    &\ge \frac{\Delta}{K-1} \cdot \pi(a^*)^2 \cdot \left( r(a^*) - \pi^\top r \right)^2. \qquad \left( K \ge 2 \right)
\end{align}
\textbf{Second case.} If $\pi^\top r > \max_{a \not= a^*}{ r(a) }$, then we have, for all $a \not= a^*$,
\begin{align}
\label{eq:stochastic_natural_lojasiewicz_continuous_special_intermediate_3}
    \pi^\top r - r(a) \ge  \pi^\top r - \max_{a \not= a^*}{ r(a) } > 0.
\end{align}
Therefore, we have,
\begin{align}
\label{eq:stochastic_natural_lojasiewicz_continuous_special_intermediate_4}
    \sum_{i = 1}^{K} \pi(i)^2 \cdot \left| r(i) - \pi^\top r \right|^3 = \pi(a^*)^2 \cdot \left( r(a^*) - \pi^\top r \right)^3 + \sum_{a \not= a^*}{ \pi(a)^2 \cdot \left( \pi^\top r - r(a) \right)^3 }.
\end{align}
Note that,
\begin{align}
\label{eq:stochastic_natural_lojasiewicz_continuous_special_intermediate_5}
    \pi(a^*) \cdot \left( r(a^*) - \pi^\top r \right) &= \underbrace{ \sum_{i = 1 }^{K}{ \pi(i) \cdot \left( r(i) - \pi^\top r \right) } }_{=0} - \sum_{a \not= a^*}{ \pi(a) \cdot \left( r(a) - \pi^\top r \right) } \\
    &= \sum_{a \not= a^*}{ \pi(a) \cdot \left( \pi^\top r - r(a) \right) }. 
\end{align}
Next, we have,
\begin{align}
\MoveEqLeft
\label{eq:stochastic_natural_lojasiewicz_continuous_special_intermediate_6}
    \sum_{a \not= a^*}{ \pi(a)^2 \cdot \left( \pi^\top r - r(a) \right)^3 } \ge \left( \pi^\top r - \max_{a \not= a^*}{ r(a) } \right) \cdot \sum_{a \not= a^*}{ \pi(a)^2 \cdot \left( \pi^\top r - r(a) \right)^2 } \qquad \left( \text{by \cref{eq:stochastic_natural_lojasiewicz_continuous_special_intermediate_3}} \right) \\
    &\ge \frac{ \pi^\top r - \max_{a \not= a^*}{ r(a) } }{ K - 1 } \cdot \left[\sum_{a \not= a^*}{ \pi(a) \cdot \left( \pi^\top r - r(a) \right) } \right]^2 \qquad \left( \text{by Cauchy–Schwarz} \right) \\
    &= \frac{ \pi^\top r - \max_{a \not= a^*}{ r(a) } }{ K-1 } \cdot \pi(a^*)^2 \cdot \left( r(a^*) - \pi^\top r \right)^2. \qquad \left( \text{by \cref{eq:stochastic_natural_lojasiewicz_continuous_special_intermediate_5}} \right)
\end{align}
Combining \cref{eq:stochastic_natural_lojasiewicz_continuous_special_intermediate_4,eq:stochastic_natural_lojasiewicz_continuous_special_intermediate_6}, we have,
\begin{align}
\MoveEqLeft
\label{eq:stochastic_natural_lojasiewicz_continuous_special_intermediate_7}
    \sum_{i = 1}^{K} \pi(i)^2 \cdot \left| r(i) - \pi^\top r \right|^3 \ge \pi(a^*)^2 \cdot \left( r(a^*) - \pi^\top r \right)^3 + \frac{ \pi^\top r - \max_{a \not= a^*}{ r(a) } }{ K-1 } \cdot \pi(a^*)^2 \cdot \left( r(a^*) - \pi^\top r \right)^2 \\
    &\ge \left[ \frac{r(a^*) - \pi^\top r }{K-1} + \frac{ \pi^\top r - \max_{a \not= a^*}{ r(a) } }{ K-1 } \right] \cdot \pi(a^*)^2 \cdot \left( r(a^*) - \pi^\top r \right)^2 \qquad \left( K \ge 2 \right) \\
    &= \frac{ r(a^*) - \max_{a \not= a^*}{ r(a) } }{ K-1 } \cdot \pi(a^*)^2 \cdot \left( r(a^*) - \pi^\top r \right)^2 \\
    &= \frac{ \Delta }{ K-1 } \cdot \pi(a^*)^2 \cdot \left( r(a^*) - \pi^\top r \right)^2.
\end{align}
Combining \cref{eq:stochastic_natural_lojasiewicz_continuous_special_intermediate_2,eq:stochastic_natural_lojasiewicz_continuous_special_intermediate_7} we finish the proofs.
\end{proof}

\begin{lemma}[Performance difference lemma \citep{kakade2002approximately}]
\label{lem:performance_difference_general}
For any policies $\pi$ and $\pi^\prime$,
\begin{align}
    V^{\pi^\prime}(\rho) - V^{\pi}(\rho) &= \frac{1}{1 - \gamma} \cdot \sum_{s}{ d_\rho^{\pi^\prime}(s) \cdot \sum_{a}{ \left( \pi^\prime(a | s) - \pi(a | s) \right) \cdot Q^{\pi}(s,a) } }\\
    &= \frac{1}{1 - \gamma} \cdot \sum_{s}{ d_{\rho}^{\pi^\prime}(s) \cdot \sum_{a}{ \pi^\prime(a | s) \cdot A^{\pi}(s, a) } }.
\end{align}
\end{lemma}
\begin{proof}
According to the definition of value function,
\begin{align}
\MoveEqLeft
    V^{\pi^\prime}(s) - V^{\pi}(s) = \sum_{a}{ \pi^\prime(a | s) \cdot Q^{\pi^\prime}(s,a) } - \sum_{a}{ \pi(a | s) \cdot Q^{\pi}(s,a) } \\
    &= \sum_{a}{ \pi^\prime(a | s) \cdot \left( Q^{\pi^\prime}(s,a) - Q^{\pi}(s,a) \right) } + \sum_{a}{ \left( \pi^\prime(a | s) - \pi(a | s) \right) \cdot Q^{\pi}(s,a) } \\
    &= \sum_{a}{ \left( \pi^\prime(a | s) - \pi(a | s) \right) \cdot Q^{\pi}(s,a) } + \gamma \cdot \sum_{a}{ \pi^\prime(a | s) \cdot \sum_{s^\prime}{  \gP( s^\prime | s, a) \cdot \left[ V^{\pi^\prime}(s^\prime) -  V^{\pi}(s^\prime)  \right] } } \\
    &= \frac{1}{1 - \gamma} \cdot \sum_{s^\prime}{ d_{s}^{\pi^\prime}(s^\prime) \cdot \sum_{a^\prime}{ \left( \pi^\prime(a^\prime | s^\prime) - \pi(a^\prime | s^\prime) \right) \cdot Q^{\pi}(s^\prime, a^\prime) }  } \\
    &= \frac{1}{1 - \gamma} \cdot \sum_{s^\prime}{ d_{s}^{\pi^\prime}(s^\prime) \cdot \sum_{a^\prime}{ \pi^\prime(a^\prime | s^\prime) \cdot \left( Q^{\pi}(s^\prime, a^\prime) - V^{\pi}(s^\prime) \right) }  } \\
    &= \frac{1}{1 - \gamma} \cdot \sum_{s^\prime}{ d_{s}^{\pi^\prime}(s^\prime) \cdot \sum_{a^\prime}{ \pi^\prime(a^\prime | s^\prime) \cdot A^{\pi}(s^\prime, a^\prime) } }. \qedhere
\end{align}
\end{proof}

\begin{lemma}
\label{lem:infinite_product_infinite_sum}
Let $u_t \in (0, 1)$ for all $t \ge 1$. The infinite product $\prod_{t=1}^{\infty}{\left( 1 - u_t \right) }$ converges to a positive value if and only if the series $\sum_{t=1}^{\infty}{ u_t }$ converges to a finite value.
\end{lemma}
\begin{proof}
See \citep[Lemma 16]{mei2021understanding}. We include a proof for completeness.

Define the following partial products and partial sums,
\begin{align}
\label{eq:positive_infinite_product_2_intermediate_1a}
    p_T &\coloneqq \prod_{t=1}^{T}{\left( 1 - u_t \right) }, \\
\label{eq:positive_infinite_product_2_intermediate_1b}
    s_T &\coloneqq \sum_{t=1}^{T}{ u_t }.
\end{align}
Since $p_T$ is monotonically decreasing and non-negative, the infinite product converges to positive values, i.e.,
\begin{align}
    \prod_{t=1}^{\infty}{\left( 1 - u_t \right) } = \lim_{T \to \infty}{ \prod_{t=1}^{T}{\left( 1 - u_t \right) } } = \lim_{T \to \infty}{p_T} > 0,
\end{align}
if and only if $p_T$ is lower bounded away from zero (boundedness convergence criterion for monotone sequence) \citep[p. 80]{knopp1947theory}.

Similarly, since $s_T$ is monotonically increasing, the series converges to finite values, i.e.,
\begin{align}
    \sum_{t=1}^{\infty}{ u_t } = \lim_{T \to \infty}{ \sum_{t=1}^{T}{ u_t } } = \lim_{T \to \infty}{s_T} < \infty,
\end{align} if and only if $s_T$ is upper bounded.

\textbf{First part.} $\prod_{t=1}^{\infty}{\left( 1 - u_t \right) }$ converges to a positive value only if $\sum_{t=1}^{\infty}{ u_t }$ converges to a finite value.

Suppose $\prod_{t=1}^{\infty}{\left( 1 - u_t \right) }$ converges to a positive value. We have, for all $T \ge 1$,
\begin{align}
    q_T \ge q > 0.
\end{align}
Then we have,
\begin{align}
    q &\le q_T \\
    &= \exp\bigg\{ \log{ \bigg( \prod_{t=1}^{T}{\left( 1 - u_t \right) } \bigg) } \bigg\} \\
    &= \exp\bigg\{ \sum_{t=1}^{T}{ \log{ \left( 1 - u_t \right) } } \bigg\} \\
    &\le \exp\bigg\{ - \sum_{t=1}^{T}{ u_t } \bigg\} \qquad \left( \log{\left( 1 - x \right) } < -x \right) \\
    &= \exp\{ - s_T \},
\end{align}
which implies that,
\begin{align}
    s_T \le - \log{q} < \infty.
\end{align}
Therefore, we have $\sum_{t=1}^{\infty}{ u_t }$ converges to a finite value.

\textbf{Second part.} $\prod_{t=1}^{\infty}{\left( 1 - u_t \right) }$ converges to a positive value if $\sum_{t=1}^{\infty}{ u_t }$ converges to a finite value.

Suppose $\sum_{t=1}^{\infty}{ u_t }$ converges to a finite value. Then we have, $u_t \to 0$ as $t \to \infty$. There exists a finite number $t_0 \ge 1$, such that for all $t \ge t_0$, we have $u_t \le 1/2$. Also, we have, for all $T \ge 1$,
\begin{align}
    s_T \le s < \infty.
\end{align}
Then we have,
\begin{align}
    \prod_{t=t_0}^{T}{\left( 1 - u_t \right) } &= \exp\bigg\{ \sum_{t=t_0}^{T}{ \log{ \left( 1 - u_t \right) } } \bigg\} \\
    &\ge \exp\bigg\{ - \sum_{t=t_0}^{T}{ 2 \cdot u_t } \bigg\} \qquad \left( - 2 \cdot x \le \log{\left( 1 - x \right) } \text{ for all } x \in [0, 1/2] \right) \\
    &= \exp\{ - 2 \cdot s_T \},
\end{align}
which implies that, for all large enough $T \ge 1$,
\begin{align}
    q_T &= \left( \prod_{t=1}^{t_0-1}{\left( 1 - u_t \right) } \right) \cdot \left( \prod_{t=t_0}^{T}{\left( 1 - u_t \right) } \right) \\
    &\ge \left( \prod_{t=1}^{t_0-1}{\left( 1 - u_t \right) } \right) \cdot \exp\{ - 2 \cdot s_T \} \\
    &\ge \left( \prod_{t=1}^{t_0-1}{\left( 1 - u_t \right) } \right) \cdot \exp\{ - 2 \cdot s \} \\
    &> 0.
\end{align}
Therefore, we have $\prod_{t=1}^{\infty}{\left( 1 - u_t \right) }$ converges to a positive value.
\end{proof}

\begin{lemma}
\label{lem:infinite_product_infinite_sum_2}
Let $u_t \in (0, 1)$ for all $t \ge 1$. We have  $\prod_{t=1}^{\infty}{\left( 1 - u_t \right) } = \lim_{T \to \infty}{ \prod_{t=1}^{T}{\left( 1 - u_t \right) } } = 0 $
if and only if the series $\sum_{t=1}^{\infty}{ u_t }$ diverges to positive infinity.
\end{lemma}
\begin{proof}
See \citep[Lemma 17]{mei2021understanding}. We include a proof for completeness.

\textbf{First part.} $\prod_{t=1}^{\infty}{\left( 1 - u_t \right) }$ diverges to $0$ only if $\sum_{t=1}^{\infty}{ u_t }$ diverges to positive infinity.

Suppose $\prod_{t=1}^{\infty}{\left( 1 - u_t \right) }$ diverges to $0$. According to \cref{lem:infinite_product_infinite_sum}, $\sum_{t=1}^{\infty}{ u_t }$ diverges. And since the partial sum $s_T \coloneqq \sum_{t=1}^{T}{ u_t }$ is monotonically increasing, we have $\sum_{t=1}^{\infty}{ u_t }$ diverges to positive infinity. 

\textbf{Second part.} $\prod_{t=1}^{\infty}{\left( 1 - u_t \right) }$ diverges to $0$ if $\sum_{t=1}^{\infty}{ u_t }$ diverges to a positive infinity.

Suppose $\sum_{t=1}^{\infty}{ u_t }$ diverges to positive infinity. According to \cref{lem:infinite_product_infinite_sum}, $\prod_{t=1}^{\infty}{\left( 1 - u_t \right) }$ diverges. And since the partial product $q_T \coloneqq \prod_{t=1}^{T}{\left( 1 - u_t \right) }$ is non-negative and monotonically decreasing, we have $\prod_{t=1}^{\infty}{\left( 1 - u_t \right) }$ diverges to $0$. 
\end{proof}

\begin{lemma}[Smoothness]
\label{lem:smoothness_softmax_special}
Let $\pi_\theta = \softmax(\theta)$ and $\pi_{\theta^\prime} = \softmax(\theta^\prime)$. For any $r \in \left( 0, 1\right]^K$, for any $\pi_{\theta}(a)$, we have $\theta \mapsto \pi_\theta(a)$ is $3/2$-smooth, i.e.,
\begin{align}
    \left| \pi_{\theta^\prime}(a) - \pi_\theta(a) - \Big\langle \frac{d \pi_\theta(a)}{d \theta}, \theta^\prime - \theta \Big\rangle \right| \le \frac{3}{4} \cdot \| \theta^\prime - \theta \|_2^2.
\end{align}
\end{lemma}
\begin{proof}
The proof is based on and improves \citep[Lemma 2]{mei2020global}.

Let $S \coloneqq S(r,\theta)\in \R^{K\times K}$ be 
the second derivative of the value map $\theta \mapsto \pi_\theta(a) = \pi_\theta^\top \rvone_{a}$, where
\begin{align}
\label{eq:rvone_a_notation}
    \rvone_{a}(i) = \begin{cases}
		1, & \text{if } i = a, \\
		0, & \text{otherwise}.
	\end{cases}
\end{align}

By Taylor's theorem, it suffices to show that the spectral radius of $S$ (regardless of $r$ and $\theta$) is bounded by $3/2$.
Now, by its definition we have
\begin{align}
    S &= \frac{d }{d \theta } \left\{ \frac{d \pi_\theta^\top \rvone_{a}}{d \theta} \right\} \\
    &= \frac{d }{d \theta } \left\{ ( \diagonalmatrix(\pi_\theta) - \pi_\theta \pi_\theta^\top) \rvone_{a} \right\}.
\end{align}
Continuing with our calculation fix $i, j \in [K]$. Then, 
\begin{align}
\MoveEqLeft
    S_{i, j} = \frac{d \{ \pi_\theta(i) \cdot  ( \rvone_{a}(i) - \pi_\theta^\top \rvone_{a} ) \} }{d \theta(j)} \\
    &= \frac{d \pi_\theta(i) }{d \theta(j)} \cdot ( \rvone_{a}(i) - \pi_\theta^\top \rvone_{a} ) + \pi_\theta(i) \cdot \frac{d \{ \rvone_{a}(i) - \pi_\theta^\top \rvone_{a} \} }{d \theta(j)} \\
    &= (\delta_{ij} \pi_\theta(j) -  \pi_\theta(i) \pi_\theta(j) ) \cdot ( \rvone_{a}(i) - \pi_\theta^\top \rvone_{a} ) - \pi_\theta(i) \cdot ( \pi_\theta(j) \rvone_{a}(j) - \pi_\theta(j) \pi_\theta^\top \rvone_{a} ) \\
    &= \delta_{ij} \pi_\theta(j) \cdot ( \rvone_{a}(i) - \pi_\theta^\top \rvone_{a} ) -  \pi_\theta(i) \pi_\theta(j) \cdot ( \rvone_{a}(i) - \pi_\theta^\top \rvone_{a} ) - \pi_\theta(i) \pi_\theta(j) \cdot ( \rvone_{a}(j) -  \pi_\theta^\top \rvone_{a} ),
\end{align}
where
\begin{align}
\label{eq:delta_ij_notation}
    \delta_{ij} = \begin{cases}
		1, & \text{if } i = j, \\
		0, & \text{otherwise},
	\end{cases}
\end{align}
is Kronecker's $\delta$-function.
To show the bound on 
the spectral radius of $S$, pick $y \in \sR^K$. Then,
\begin{align}
\MoveEqLeft
    \left| y^\top S y \right| = \left| \sum\limits_{i=1}^{K}{ \sum\limits_{j=1}^{K}{ S_{i,j} y(i) y(j)} } \right| \\
    &= \left| \sum_{i}{ \pi_\theta(i) ( \rvone_{a}(i) - \pi_\theta^\top \rvone_{a}) y(i)^2 } - 2 \sum_{i} \pi_\theta(i) ( \rvone_{a}(i) - \pi_\theta^\top \rvone_{a} ) y(i) \sum_{j} \pi_\theta(j) y(j) \right| \\
    &= \left| \left( ( \diagonalmatrix(\pi_\theta) - \pi_\theta \pi_\theta^\top) \rvone_{a} \right)^\top \left( y \odot y \right) - 2 \cdot \left( ( \diagonalmatrix(\pi_\theta) - \pi_\theta \pi_\theta^\top) \rvone_{a} \right)^\top y \cdot \left( \pi_\theta^\top y \right) \right| \\
    &\le \left\| ( \diagonalmatrix(\pi_\theta) - \pi_\theta \pi_\theta^\top) \rvone_{a} \right\|_\infty \cdot \left\| y \odot y \right\|_1 + 2 \cdot \left\| ( \diagonalmatrix(\pi_\theta) - \pi_\theta \pi_\theta^\top) \rvone_{a} \right\|_1 \cdot \left\| y \right\|_\infty \cdot \left\| \pi_\theta \right\|_1 \cdot \left\| y \right\|_\infty \\
    &\le \left\| ( \diagonalmatrix(\pi_\theta) - \pi_\theta \pi_\theta^\top) \rvone_{a} \right\|_\infty \cdot \| y \|_2^2 + 2 \cdot \left\| ( \diagonalmatrix(\pi_\theta) - \pi_\theta \pi_\theta^\top) \rvone_{a} \right\|_1 \cdot \| y \|_2^2 \\
    &\le 3 \cdot \left\| ( \diagonalmatrix(\pi_\theta) - \pi_\theta \pi_\theta^\top) \rvone_{a} \right\|_1 \cdot \| y \|_2^2,
\end{align}
where $\odot$ is Hadamard (component-wise) product, and the third last inequality uses H{\" o}lder's inequality together with the triangle inequality, and the second inequality uses $\| y \odot y \|_1 = \| y \|_2^2$, $\| \pi_\theta \|_1 = 1$, and $\| y \|_\infty \le \| y \|_2$. Next, we have,
\begin{align}
\label{eq:H_matrix_r_1_norm_upper_bound_special}
\MoveEqLeft
    \left\| ( \diagonalmatrix(\pi_\theta) - \pi_\theta \pi_\theta^\top) \rvone_{a} \right\|_1 = \sum_{i}{ \pi_\theta(i) \cdot \left| \rvone_{a}(i) - \pi_\theta^\top \rvone_{a} \right| } 
    \\
    &= \pi_\theta(a) \cdot \left( 1 - \pi_\theta(a) \right) + \pi_\theta(a) \cdot \sum_{i \not= a}{ \pi_\theta(i) } \\
    &= 2 \cdot \pi_\theta(a) \cdot \left( 1 - \pi_\theta(a) \right) \\
    &\le 1/2. \qquad\left( x \cdot (1 - x) \le 1/4 \text{ for all } x \in [0, 1] \right)
\end{align}
Therefore we have,
\begin{align}
\label{eq:H_matrix_maximum_eigenvalue}
    \left| y^\top S(r, \theta) y \right| &\le 3 \cdot \left\| ( \diagonalmatrix(\pi_\theta) - \pi_\theta \pi_\theta^\top) \rvone_{a} \right\|_1 \cdot \| y \|_2^2 \\
    &\le 3/2 \cdot \left\| y \right\|_2^2,
\end{align}
finishing the proof.
\end{proof}

\end{document}